\documentclass[twocolumn]{IEEEtran}
\IEEEoverridecommandlockouts                              % This command is only needed if
\bstctlcite{IEEEexample:BSTcontrol}
\usepackage[utf8]{inputenc} % allow utf-8 input
\usepackage[T1]{fontenc}    % use 8-bit T1 fonts
\usepackage{hyperref}       % hyperlinks
\usepackage{url}            % simple URL typesetting
\usepackage{booktabs}       % professional-quality tables
\usepackage{amsfonts}       % blackboard math symbols
\usepackage{nicefrac}       % compact symbols for 1/2, etc.
\usepackage{comment}
\usepackage{lmodern}
\usepackage{scalefnt}
\usepackage{setspace}
\usepackage{romannum}
\usepackage[usenames,dvipsnames]{xcolor}
\usepackage{tkz-berge}
\usetikzlibrary{fit,shapes}
\usepackage{calrsfs}
\DeclareMathAlphabet{\pazocal}{OMS}{zplm}{m}{n}
\usepackage{graphicx}
\usepackage{graphics}
\usepackage{epsfig}
\usepackage{mathptmx}
\usepackage{caption}
\usepackage{subcaption}
\usepackage{tikz}
\usetikzlibrary{positioning}
\usepackage{amsmath, amssymb}
\usetikzlibrary{arrows}
\usepackage{dsfont}

\usepackage{amsthm}
\usepackage{graphicx}
\DeclareMathOperator*{\argmax}{arg\,max}

\usepackage{kantlipsum}
\usepackage{float}
\usepackage{amsfonts}
\usepackage{setspace}
\usepackage{multirow}
\usepackage[english]{babel}
\usepackage{float}
\usepackage{algorithmicx}
\usepackage[section]{algorithm}
\usepackage{algpascal}
\usepackage{algc}
\usepackage{algcompatible}
\usepackage{multicol}
\usepackage{algpseudocode}
\let\oldReturn\Return
\renewcommand{\Return}{\State\oldReturn}
\usepackage{mathtools}
\usepackage{mathptmx}
\usepackage{mathtools, cuted}
\usepackage{textcomp}
\usepackage[english]{babel}
\usepackage{rotating}
\usepackage{pgfplots}
\pgfplotsset{compat=1.5}
\usepackage{siunitx}
\usepackage{pgfplotstable}
\usepackage{bbold}
\usepackage{float}

\newtheorem{theorem}{Theorem}

\newtheorem{lemma}{Lemma}
\newtheorem{prop}{Proposition}

\newtheorem{assume}{Assumption}
\newtheorem{constraint}{Constraint}
\newtheorem*{remark}{Remark}

\renewcommand{\footnotesize}{\scriptsize}

\DeclareMathAlphabet{\pazocal}{OMS}{zplm}{m}{n}

\renewcommand{\bar}[1]{\mskip.5\thinmuskip\overline{\mskip-.5\thinmuskip {#1} \mskip-.5\thinmuskip}\mskip.5\thinmuskip} % overline short
\newcommand{\norm}[1]{\left\lVert#1\right\rVert}
\newcommand{\I}{\pazocal{I}}

\newcommand{\X}{\pazocal{X}}
\newcommand{\Z}{\pazocal{Z}}
\newcommand{\F}{\pazocal{F}}
\newcommand{\bR}{\mathbb{R}}
\newcommand{\bE}{\mathbb{E}}
\newcommand{\bN}{\mathbb{N}}
\newcommand{\ti}{{t, i }}
\newcommand{\bti}{{b_t, i }}
\newcommand{\taui}{{\tau, i }}

\newcommand{\R}{\pazocal{R}}
\newcommand{\xs}{x_{\ti}^{\star}}
\newcommand{\xsb}{\bar{x}_{\ti}^{\star}}

\newcommand{\A}{{\pazocal{A}}}
\newcommand{\C}{{\pazocal{C}}}

\newcommand{\B}{{\pazocal{B}}}
\newcommand{\mP}{{\mathbb{ P}}}
\renewcommand{\P}{{\pazocal{ P}}}

\newcommand{\pixc}{\phi_i(x_{\ti}, c_t)}
\newcommand{\piixc}{\phi_i(x, c_t)}

\newcommand{\pixsc}{\phi_i(\xs, c_t)}
\newcommand{\pixsbc}{\phi_i(\xsb, c_t)}
\newcommand{\pixbc}{\phi_i(\xb, c_t)}
\newcommand{\pixbic}{\phi_i(x_{\bti}, c_t)}
\newcommand{\psixc}{\psi_i(x_{\ti}, \mu_t)}
\newcommand{\psiixc}{\psi_i(x, \mu_t)}
\newcommand{\psixsc}{\psi_i(\xs, \mu_t)}
\newcommand{\psixsbc}{\psi_i(\xsb, \mu_t)}
\newcommand{\psixbc}{\psi_i(\xb, \mu_t)}
\newcommand{\psixbic}{\psi_i(x_{\bti}, \mu_t)}

\newcommand{\hth}{\hat{\theta}}
\newcommand{\ths}{\theta^\star}

\renewcommand{\det}{{\mathrm{det}}}

\newcommand{\rbti}{{r_{\bti}}}
\newcommand{\rbtaui}{{r_{l}}}
\newcommand{\isyn}{{\mathrm{syn}}}
\newcommand{\inew}{{t, i}}

\newcommand{\xb}{\bar{x}}
\newcommand{\set}{\Xi}

\newcommand{\trace}{{\mathrm{trace}}}

\newcommand\scalemath[2]{\scalebox{#1}{\mbox{\ensuremath{\displaystyle #2}}}}

\title{\LARGE \bf Distributed Multi-Task Learning for Stochastic Bandits with Context Distribution and Stage-wise Constraints}
\author{Jiabin Lin and Shana Moothedath
\thanks{J. Lin and S. Moothedath are with the Department of Electrical and Computer Engineering, Iowa State University, USA. Email: $\lbrace$jiabin, mshana $\rbrace$@iastate.edu.}
}

\begin{document}
\pagenumbering{arabic}
\maketitle

\begin{abstract}%   <- trailing '%' for backward compatibility of .sty file
%We introduce the problem of conservative distributed multi-task learning in stochastic linear contextual bandits with heterogeneous agents-a generalization of conservative linear bandits- where each agent learns a separate but related task and is subject to performance constraints.
We present conservative multi-task learning in stochastic linear contextual bandits with {\em heterogeneous} agents. This extends conservative linear bandits to a distributed setting where $M$ agents tackle {\em different but related} tasks while adhering to stage-wise performance constraints.
%Each agent observes identical side information (context), selects an action, and receives a reward. 
The exact context is {\em unknown}, and only a context distribution is available to the agents as in many practical applications that involve a prediction mechanism to infer context, such as stock market prediction and weather forecast. 
%The reward parameters differ among agents, resulting in some features being irrelevant, indicated by their values being set to zero in the feature vector.
We propose a distributed upper confidence bound (UCB)  algorithm, DiSC-UCB.
Our algorithm dynamically constructs a pruned action set for each task in every round, guaranteeing compliance with the constraints.
%Our algorithm constructs a pruned action set for every task during each round to ensure the constraints are met. 
Additionally, it includes synchronized sharing of estimates among agents via a central server using well-structured synchronization steps.
%%(i) an upper bound for the regret of the standard distributed linear UCB algorithm,
%(ii) a term that captures the loss since the contexts are unknown, and 
%(iii) a  term that accounts for the loss for being conservative to satisfy the performance constraint. 
For $d$-dimensional linear bandits, we prove an $\widetilde{O}(d\sqrt{MT})$ regret bound and  an $O(M^{1.5}d^3)$ communication bound on the algorithm.
We extend the problem to a setting where the agents are unaware of the baseline reward. We provide a modified algorithm, DiSC-UCB-UB, and show that it achieves the same regret and communication bounds.
We empirically validated the performance of our algorithm on synthetic data and real-world Movielens-100K and LastFM data and also compared it with some existing benchmark algorithms.
\end{abstract}
%\vspace{-4 mm}
\begin{IEEEkeywords}
Distributed learning,  online learning, multi-arm bandits, constrained contextual bandits
\end{IEEEkeywords}
\vspace{-2 mm}
\section{Introduction}
In Contextual Bandits (CB), an agent engages in a series of interactions with an environment over multiple rounds. At the start of each round, the environment presents a context, and in response, the agent selects an action that yields a reward. The agent's objective is to choose actions to maximize cumulative reward over a time horizon of $T$. 
%This introduces the exploration-exploitation dilemma, as the agent must balance exploratory actions to estimate the environment's reward function and exploitative actions that maximize the overall return. 
CB algorithms find applications in various fields, including robotics, clinical trials, communications, and recommender systems. This paper extends the standard CB problem in three ways.

First, the standard CB model assumes precise context observation, which does not always hold in real-world applications. For instance, contexts can be noisy measurements or predictions, like weather forecasting or stock market analysis \cite{kirschner2019stochastic}. In recommender systems, privacy constraints might limit access to certain user features, but we can infer a distribution over these features \cite{lamprier2018profile}. To address context uncertainty, motivated by the prior work \cite{kirschner2019stochastic}, we study a scenario where the environment provides a context distribution. The exact context is treated as a sample from this distribution and is hidden.
Second, given the increasing demand for safe learning in various real-world systems, particularly those with safety-critical applications, this paper delves into the impact of stage-wise safety constraints on the linear stochastic bandit problem. Drawing inspiration from prior works \cite{kazerouni2017conservative}, \cite{wu2016conservative, yang2021robust}, our approach builds upon the safety constraint introduced by \cite{khezeli2020safe} and later studied in \cite{moradipari2020stage}.
 In our scenario, the agent has a baseline policy suggesting actions with guaranteed expected rewards derived from historical data or legacy policies. We enforce a safety constraint that requires the agent's chosen actions to yield expected rewards no less than a predetermined fraction of those recommended by the baseline policy. 
 %This framework is useful in recommender systems to prevent highly unfavorable recommendations. 
 %Our stage-wise conservative constraints ensure user satisfaction and competitive rewards with the baseline policy.
%
Third, multi-task learning enables models to simultaneously tackle multiple related tasks, leveraging common patterns and improving overall performance \cite{zhang2018overview, wang2016distributed, lin2023distributed1, lin2023distributed2, lin2023federated}. By sharing knowledge across tasks, multi-task learning can lead to more efficient and effective models, especially when data is limited or expensive to acquire. Multi-task bandit learning has gained interest recently \cite{deshmukh2017multi, fang2015active, cella2023multi, hu2021near, yang2020impact, cella2021multi, OurICML, lu2022provable}.
%\cite{deshmukh2017multi, fang2015active, cella2023multi, hu2021near, yang2020impact, cella2021multi, OurICML, lu2022provable}. 
%Many applications of bandit learning, such as recommending movies or TV shows to users and suggesting personalized treatment plans for patients with various medical conditions, involve related tasks.

In this paper, we consider heterogeneous multi-task linear stochastic CB problem with hidden contexts and stage-wise performance constraints. A set of $M$ agents collaborate to solve related but different tasks jointly; while the exact contexts are hidden, only a context distribution is known to the agents, and the agents are subject to performance constraints at every decision round.
%{\cblue The problem addressed in this work finds applications in many practical settings of bandit learning. For instance, recommender systems that suggest movies or TV shows to users and provide personalized treatment plans for patients with various medical conditions often involve related tasks.  While numerous features are shared between movies and TV shows, certain features might not hold relevance across both domains. Moreover, the context (user/patient information) in these problems can be noisy. For example, incorrect information might be added while creating a profile, or someone else might use the user's profile to log in to Netflix. Often, there is a baseline policy that the recommender system uses based on past experience. While adapting and improving the learning algorithm is a priority, the streaming platform may impose a constraint that daily viewership must be at least a certain fraction of the existing viewership to ensure revenue stability.}
 The problem addressed in this work applies to various bandit learning scenarios, such as recommender systems and personalized medical treatments, where tasks are related but not identical. While movies and TV shows share features, some may be irrelevant across domains. Additionally, context (e.g., user or patient data) can be noisy due to profile errors or shared accounts. Often, a baseline policy guides recommendations based on past data, but platforms may impose constraints, ensuring daily viewership remains above a threshold to maintain revenue stability.
These applications can significantly benefit from our approach, as demonstrated in our empirical analysis in Section~\ref{sec:sim}.

%\vspace{-2 mm}
\subsection{Our Contributions}  
The paper makes four key contributions.
\begin{enumerate}
    \item We formulate the constrained multi-task contextual bandit problem with hidden contexts and propose a distributed UCB algorithm. 
    A key aspect of our approach is the construction of a safe action set at each learning round to filter out actions that fail to meet performance constraints, which is rather challenging since the contexts are unknown. The unsafe action elimination method in existing work \cite{moradipari2020stage}, which assumes known contexts, is not directly applicable to our setting with unknown contexts.
%   We present a novel approach to eliminate the unsafe actions.
Lemma~\ref{L2} proves that our safe action set meets the performance constraints, while Lemma~\ref{L4} proves that the optimal action in each round remains within the pruned set and is not eliminated.
 %Inspired by the synchronization step in \cite{wang2022multi}, the individual agents share their estimates with a central server at carefully designed synchronization rounds to reduce communication costs.
 \item We show in Theorem~\ref{prop1} that the regret can be decomposed into three terms:
(i) an upper bound for the regret of the distributed linear UCB algorithm,
 (ii) a term that captures the loss since the contexts are unknown and
(iii) a term that accounts for the loss for being conservative to satisfy the performance constraint.
A key aspect of our analysis is bounding the number of rounds the learner's actions ($N_T$) and the conservative actions ($N^c_T$) are played, which we establish in Theorem~\ref{L7}.

 In Theorem~\ref{T1}, we provide an $\widetilde{O}(d\sqrt{MT})$ regret bound and an $O(M^{1.5}d^3)$ communication bound on the algorithm, where $d$ is the dimension of the feature vector. Our proposed approach significantly outperforms the naive method of solving the $M$ tasks independently, which results in a regret of $\widetilde{O}(dM\sqrt{T})$ \cite{moradipari2020stage, abbasi2011improved}, thereby demonstrating the effectiveness of multi-task learning.
\item We extend the problem to a setting where the agents are unaware of the baseline reward value. We show that the algorithm can be modified slightly to address this case, and regret and communication bounds remain the same.
\item We conducted numerical experiments using both synthetic and real-world datasets (MovieLens and LastFM) to validate our algorithms. Our methods consistently outperformed benchmark algorithms across all evaluations.
\end{enumerate}

\vspace{-1 mm}
\subsection{Related Work}
\noindent{\bf Multi-task bandit learning:} Collaborative bandit learning has been explored in both homogeneous and heterogeneous settings \cite{lin2023distributed1, lin2023distributed2, lin2023federated, Jiabin_Shana_ACC, wang2019distributed,huang2021federated, korda2016distributed, zhu2023byzantine}. In the homogeneous setting, all tasks are identical, whereas in the heterogeneous setting, tasks are distinct yet share underlying similarities.
In the homogeneous setting, distinct agents solve a common task collaboratively using a shared reward parameter $\ths$ \cite{lin2023distributed1, lin2023distributed2, wang2019distributed, korda2016distributed, zhu2025decentralized}. Our work, on the contrary, introduces a multi-task bandit learning in which heterogeneous agents perform distinct tasks governed by heterogeneous reward parameters $\Theta=\{\theta_i^{\star}\}_{i\in [M]}$. 
This model finds relevance in various real-world scenarios where agents, possessing {\em distinct yet interrelated objectives}, function within a collaborative setting. 
%The majority of existing research in the distributed linear bandits, including  \cite{wang2019distributed, korda2016distributed, lin2023distributed1, lin2023distributed2}, consider agents with homogeneous feature vectors. 
\cite{lin2023federated, Jiabin_Shana_ACC, huang2021federated}
considered tasks with unique feature vectors.
However, their scenario treats each agent as a user/context, resulting in a time-invariant situation. In contrast, our work addresses a time-varying case.
%, rendering the previous methodology unsuitable. 
 Another related line of work is multi-task representation learning of bandits \cite{deshmukh2017multi, fang2015active, cella2023multi, hu2021near, yang2020impact, OurICML, lu2022provable} that focuses on learning a shared low-dimensional representation between the tasks. Bandit problems with sparsity constraints are studied in \cite{cella2021multi, hao2020high}.
None of these works consider constrained learning or unobserved contexts.
%Further, we do not assume a low-dimensional representation in this paper.
Contextual bandits with partial observations have been studied in \cite{kirschner2019stochastic,lin2023distributed1, lin2023distributed2}. A related line of work is partial monitoring, where the rewards for the chosen actions are unobserved \cite{lin2014combinatorial,kirschner2020information, kirschner2023linear}. Instead, the learner receives indirect feedback that correlates with the rewards.

\noindent{\bf Constrained bandits:} Constrained bandits have been well studied, including \cite{amani2019linear,mansour2015bayesian, katariya2019conservative, sui2015safe, varma2023stochastic, lin2022stochastic} under various modeling assumptions. The two primary types of constraints are: (i) budget constraints, where each arm incurs a random resource consumption, and (ii) safety/performance constraints, which ensure that the reward in each round meets or exceeds a specified fraction of a baseline performance.

\cite{kazerouni2017conservative, wu2016conservative,lin2022stochastic} considered a cumulative performance constraint, which ensures that the total reward over the learning horizon meets or exceeds a specified threshold, in contrast to stage-wise constraints that must be satisfied in every round.
 \cite{amani2019linear} considered linear constraint of the form  $x^\top B \ths \leqslant C$, where $B, C> 0$ are a known matrix and a positive constant, respectively.
The most related works are \cite{khezeli2020safe, moradipari2020stage}, which also considered stage-wise performance constraints, however, in a single task setting with observable contexts.
 In the multi-task setting considered in our paper, each task is associated with a distinct constraint, resulting in $M$ constraints.
%These distinctions necessitate significant modifications to our analytical approach.
 Constrained Markov Decision Process (CMDP) extends traditional MDPs by incorporating a cost function to ensure compliance with constraints \cite{altman2021constrained}. The safety specifications in the majority of the CMDP methods are expected discounted cumulative costs, and the goal is to maximize the expected performance while minimizing the cost incurred from constraints.
Constrained restless bandits, as studied by \cite{kaza2019constrained} addresses scenarios where resource availability varies over time, necessitating adaptive strategies to optimize decision-making under dynamic constraints.
%These studies are not directly extendable to the case distributed multi-task setting with heterogeneous feature vectors and performance constraints. 
%
%
\section{Stochastic Linear Contextual Bandits}\label{sec:back}
In this section, we will first introduce the standard linear bandit, then the stochastic stage-wise constrained setting, and finally, the distributed stochastic stage-wise constrained setting. Let $\A$ be the action set and $\C$ be the context set. \\

\noindent{\bf Stochastic linear CBs:}
In linear bandits, at round $t \in \bN$, the agent observes a context $c_t \in \C$ and selects an action $x_t \in \A$. The context-action pair $(x_t, c_t)$ is mapped onto a feature vector $\phi_{x_t, c_t} \in \bR^d$. Based on this feature vector $\phi_{x_t, c_t}$, the agent receives a reward $y_t \in \bR$ from the environment, where $y_t = \phi_{x_t, c_t}^{\top} \ths + \eta_t$, with $\ths$ representing an unknown reward parameter, and $\eta_t$ is a $\sigma-$Gaussian noise with zero mean. Thus the expected reward $r(x_t, c_t) = \bE [y_t]$. The goal of the linear bandit problem is to maximize the cumulative reward or, equivalently, minimizes the cumulative (pseudo) regret \cite{abbasi2011improved}, as
\begin{align}
\min \R_T = \sum_{t = 1}^T \phi_{x_t^{\star}, c_t}^{\top} \ths - \sum_{t = 1}^T \phi_{x_t, c_t}^{\top} \ths. 
\end{align}
%\vspace{-1 mm}
Here $x_t^{\star}$ is the optimal/best action for context $c_t$, and $x_t$ is the action chosen by the agent for context $c_t$. 
\cite{abbasi2011improved, dani2008stochastic} showed that the UCB policy achieves sublinear regret $\tilde{O}(d\sqrt{T})$.
\vspace{2 mm}

\noindent{\bf Stochastic linear CBs with stage-wise constraints:}
Here, the agent is provided with a baseline policy. The baseline action $x_{b_t}$ has an expected reward $r_{b_t} = \phi_{x_{b_t}, c_t}^{\top} \ths$. The agent's actions are subject to the stage-wise conservative constraint $r_t = \phi_{x_t, c_t}^{\top} \ths \geqslant (1 - \alpha) r_{b_t}$, where $\alpha\in [0,1]$, such that the agent's action will be chosen only if it satisfies the constraint, otherwise, the baseline action will be performed \cite{khezeli2020safe, moradipari2020stage}. This guarantees that the expected reward for the agent at any round $t$ remains at least a pre-defined fraction $(1 - \alpha)$ of the baseline reward. The goal is to maximize the cumulative reward while satisfying the stage-wise constraint. Formally, minimize the cumulative (pseudo) regret while meeting the constraint
\begin{align}
\min \R_T \mbox{~such~that~}\phi_{x_t, c_t}^{\top} \ths \geqslant (1 - \alpha) r_{b_t}, \mathrm{~for~all}~ t \in [T].
\end{align}
%Here $x_t^{\star}$ is the optimal/best action for context $c_t$, and $x_t$ is the action suggested by the agent for context $c_t$. 

For this setting, \cite{khezeli2020safe} presented an algorithm,  referred to as SEGE with regret bound $\widetilde{O}(d\sqrt{T})$  when the actions are constrained to be from an ellipsoid, and  \cite{moradipari2020stage} presented a UCB algorithm, referred to as SCLUCB, and a Thompson sampling algorithm, referred to as SCLTS, with regret bounds $\widetilde{O}(d\sqrt{T})$.

% \noindent{\bf Distributed Stochastic linear CBs:}
% In this setting, a system of $M$ agents works collaboratively to choose optimal actions under the coordination of a central server in order to minimize cumulative regret. The communication network consists of a central server and a set of $M$ agents interacting by sending and receiving packets with zero latency. The goal is to learn the optimal policy for mapping each context $c \in \C$ to an action $x \in \A$ so as to maximize the cumulative reward \cite{wang2019distributed, huang2021federated, korda2016distributed}. The aim is to minimize the cumulative regret
% $R_T = \sum_{i = 1}^M \sum_{t = 1}^T \phi_{x_t^{\star}, c_t}^{\top} \ths - \sum_{i = 1}^M \sum_{t = 1}^T \phi_{x_t, c_t}^{\top} \ths.$
% A UCB algorithm is proposed in \cite{wang2019distributed} with $O(d \sqrt{MT}\log^2 T)$ regret bound.
%
\section{Problem Formulation and Notation}

\noindent{\bf Notations:}
The norm of a vector $z \in \bR^d$ with respect to a matrix $V \in \bR^{d \times d}$ is defined as $\|z \|_{V} : = \sqrt{z^\top V z}$. Further, $\top$ denotes matrix or vector transpose. $\F_t = (\F_1, \sigma (x_1, \eta_1, \cdots x_t, \eta_t))$ be the filtration ($\sigma-$algebra) that represents the information up to round $t$. For an integer $Z$, we denote $[Z]=\{1,2,\ldots, Z\}$.\\

\noindent{\bf Problem Formulation:}
In this work, we study multi-task stochastic linear CBs with stage-wise constraints and context distribution. We consider a set of $M$ heterogeneous agents performing {\em different but related tasks} such that their reward parameters $\theta_i^{\star} \in \bR^{d_i}$, for $i\in [M]$, satisfy a sparse structural constraint. Our problem models a joint multi-tasking bandit problem, where the set $\Theta = \{\theta_1^{\star}, \theta_2^{\star}, \cdots, \theta_M^{\star}\}$ satisfies the sparse structural constraints described below. 
\begin{constraint} \label{c1} (Sparse structural constraint on $\Theta$)
 We define an index set for tasks $\I=\{I_1, I_2, \ldots, I_M\}$,  where $|I_i|=d_i$.
For $i \in \{1,\ldots, M\}$, $I_i$ is the index set of features that are relevant to agent $i$.
Without loss of generality, we assume $d_1 \geqslant d_2 \geqslant \cdots \geqslant d_M$. Thus, $I_i = \{k \in \{1, 2, \cdots, d_1\}: k^{\rm th} \mbox{~feature~of~}\theta_{1}^{\star} \in \theta_i^{\star}\}$, where $\theta_i^{\star} := \theta_{1}^{\star}|_{I_i}$. 
\end{constraint}

Further, we consider that at round $t$, the context $c_t$ is unobservable and only a distribution of the context denoted as the agents observe $\mu_t$. At round $t$, the environment chooses a distribution $\mu_t \in \P(\C)$ over the context set and samples a context realization $c_t \sim \mu_t$. The agents observe only $\mu_t$ and not $c_t$ and each agent selects an action, say action chosen by agent $i$ is $x_{\ti}$, and receive reward $y_{\ti}$, where $y_{\ti} = \phi_{x_{\ti}, c_t}^{\top} \theta_i^{\star} + \eta_{\ti}$. Moreover, each agent possesses a baseline policy, $\pi_{b_t, i}$,  shaped by their past experiences and domain expertise. This baseline policy guides the derivation of a baseline action $x_{\bti}$ aligned with the specific context $c_t$. This interplays with the performance constraint, wherein each agent 
$i$ is bound to select an action meeting the condition defined $\pi_{b_t, i}$ as described below. 
\begin{constraint} \label{c2} (Performance constraint on agent $i$ w.r.t $\pi_{b_t, i}$)
Given a context $c_t$ and the baseline policy  $\pi_{b_t, i}$,  agent $i$ can select an action $x_{\ti}$ only if $\phi_{x_{\ti}, c_t}^{\top} \theta_i^{\star} \geqslant (1 - \alpha) \rbti$, where $\rbti$ is the baseline action derived from $\pi_{b_t, i}$ for $c_t$.
\end{constraint}
 Our aim is to maximize the cumulative reward, $\sum_{i = 1}^M \sum_{t = 1}^T y_{\ti}$, while simultaneously satisfying  constraints~\ref{c1} and~\ref{c2}. Formally, our aim is to minimize the cumulative regret
\begin{align}
\hspace{-1 mm}\scalemath{0.9}{\R_T = \sum_{i = 1}^M \sum_{t = 1}^T (\phi_{\xs, c_t} -  \phi_{x_{\ti}, c_t})^{\top} \theta_i^{\star}}\mbox{~s.t~constraints}  \ref{c1}~\&~\ref{c2}\mathrm{~hold.}
\end{align}
Here $\xs= \argmax_{x_{\ti} \in \A} \bE_{c \sim \mu_t} [r_{x_{\ti}, c}]$ is the best action provided we know $\mu_t$, but not $c_t$, and $T$ is the total number of rounds, and $(1 - \alpha) \in (0, 1)$ is the maximum fraction of loss in the performance compared to the baseline policy the decision maker is willing to accept during learning. Let $\kappa_{\bti} = \phi_{\xs, c_t}^{\top} \theta_i^{\star} - \rbti$ be the difference between the expected reward of the optimal action $\xs$ and the baseline action $x_{\bti}$ at round $t$. 

To solve the multi-task problem, we first transform the problem as follows.
We transform the distributed multi-task problem with feature vector $\phi_{x_{\ti}, c_t} \in \bR^{d_i}$ and heterogeneous reward parameters $\theta_i^{\star} \in \bR^{d_i}$ for each agent into a distributed linear bandit model with the heterogeneous feature vectors $\pixc \in \bR^d$ and shared reward parameter $\ths$.
We perform the mapping of the feature vector $\phi_{x_{\ti}, c_t} \in \bR^{d_i}$ to a new feature vector $\pixc \in \bR^{d_1}$. Henceforth for notational brevity, we define $d:=d_1$ after dropping the subscript. This is achieved by retaining the features corresponding to the index set $I_i$ and setting the value to zero for features not in the index set $I_i$ resulting a $d$-dimensional feature vector $\pixc$. We elaborate on this process in the example given in Section~\ref{sec:rel}. Henceforth, our analysis use heterogeneous features and shared $\ths$.
%
%Below we make certain standard assumptions regarding the unknown parameter $\theta_i^{\star}$, the feature vector $\norm{\phi_{x_{\ti}, c_t}}$, the noise $\eta_{\ti}$, and the reward gap $\kappa_{\bti}$. 
\begin{assume}
Each element $\eta_{\ti}$ of the noise sequence $\{\eta_{\ti}\}_{t=1, i=1}^{\infty, M}$ is conditionally $\sigma-$subGaussian, i.e.,
$\mathbb{E}[e^{\lambda \eta_{\ti}}|\F_{t-1}] \geqslant exp(\dfrac{\lambda^2 \sigma^2}{2})$, for all $\lambda \in \bR$.
%\vspace{-2 mm}
%\begin{eqnarray*}
%\mbox{For~all~} \lambda \in \bR, \mathbb{E}[e^{\lambda \eta_{\ti}}|\F_{t-1}] \geqslant exp(\dfrac{\lambda^2 \sigma^2}{2}).
%\end{eqnarray*}
\end{assume}
\begin{assume}
For simplicity, we assume  $\phi_{x_{\ti}, c_t}^{\top} \theta_i^{\star} \in [0,1]$, $\norm{\phi_{x_{\ti},c_t}}_2 \leqslant 1$, and $\norm{\theta_i^\star}_2 \leqslant 1$, for all $i \in[M]$ and all $x \in \A$. 
\end{assume}
\begin{assume}\label{assume:bound}
There exist constants $\kappa_l \geqslant \kappa_h \geqslant 0$, $r_h \geqslant r_l \geqslant 0$ such that at each round $t$, $\kappa_l \leqslant \kappa_{\bti} \leqslant \kappa_h$, $r_l \leqslant \rbti \leqslant r_h$. 
\end{assume}

\section{Distributed Stage-wise Contextual Bandits with Context Distribution}\label{sec:sol}
\subsection{Proposed Algorithm}
In this section, we introduce our proposed UCB algorithm, referred to as {\em distributed stage-wise contextual bandits with context distribution} algorithm (DiSC-UCB). We present the pseudocode in Algorithm~\ref{alg:TV}.

Given the distribution $\mu_t$, as in \cite{kirschner2019stochastic} we first construct the heterogeneous feature vectors $\Psi_{\ti} = \{\psiixc: x \in \A\}$, where $\{\psiixc := \bE_{c \sim \mu_t} [\piixc]\}$ is the expected feature vector of action $x$ under $\mu_t$. We use $\Psi_{\ti}$ as the feature set at round $t$. DiSC-UCB is based on the optimization in the face of the uncertainty principle. In each round $t \in [T]$, each agent $i \in [M]$ maintains confidence set $\B_{\ti} \subseteq \bR^d$ that contains the unknown reward parameter $\ths$ with high probability and constructs a pruned action set $\X_{\ti}$  by eliminating a subset of the actions that violate the constraints. After $\X_{\ti}$ is determined, agents derive the corresponding pruned feature set $\set_{\ti} = \{\psiixc: x \in \X_{\ti}\}$, which is a subset of $\Psi_{\ti}$. 
Each agent then chooses an optimistic estimate $\tilde{\theta}_{\ti} \in \mathop{\arg\max}_{\theta \in \B_{\ti}} (\max_{x \in \X_{\ti}} \psiixc^{\top} \theta))$ and chooses an action $x_{\ti}^{\prime} \in \mathop{\arg\max}_{x \in \X_{\ti}} \psiixc^{\top} \theta$. Equivalently, the agent chooses 
$(x_{\ti}^{\prime}, \tilde{\theta}_{\ti}) \in \mathop{\arg\max}_{(x, \theta) \in \X_{\ti} \times \B_{\ti}} \psiixc^{\top} \theta$.
%which jointly maximizes the reward.

\begin{algorithm}[!ht]
\caption{Distributed Stage-wise Contextual Bandits with Context Distribution (DiSC-UCB)}\label{alg:TV}
%\begin{algorithmic}
%\State \textit {\bf Input:} 
%\end{algorithmic}
\begin{algorithmic}[1] %[1] enables line numbers
\State \textit {\bf  Initialization:} $ B = (\frac{T \log MT}{d M}) $, $\lambda = 1$, $W_{\isyn} = 0, U_{\isyn} = 0, W_{\inew} = 0, U_{\inew} = 0, t_{last} = 0, V_{last} = \lambda I$
%\State The agents observe the index set $\I$ and map the feature vector $\phi_{x_{\ti}, c_t}$ to $\pixc$
\For{$t = 1,2,\ldots, T$}
\State Nature chooses $\mu_t \in \P(\C)$ and agent observes $\mu_t$
\State Set $\Psi_{\ti} = \{\psiixc: x \in \A\}$ where $\{\psiixc := \mathbb{E}_{c \sim \mu_t}[\piixc]\}$ for each agent $i$ \label{step:psi}
\For{Agent $\ i=1,2,\ldots,M, $}
\State $\overline{V}_{\ti} = \lambda I + W_{\isyn} + W_{\inew}, \hth_{\ti} = \bar{V}^{-1}_{\ti} (U_{\isyn} + U_{\inew})$
\State Construct confidence set $\B_{\ti}$ using $\overline{V}_{\ti}, \hth_{\ti}$
\State Compute pruned action set $\X_{\ti}$ using $\psiixc, \hth_{\ti}$\label{pas}
\State Construct feature set $\set_{\ti} = \{\psiixc: x \in \X_{\ti}\}$
\If {the following optimization is feasible: $(x_{\ti}^{\prime}, \tilde{\theta}_{\ti}) = \mathop{\arg\max}_{(x, \theta) \in \X_{\ti} \times \B_{\ti}} \left\langle \psiixc, \theta \right\rangle$ \label{line:12}}
\State Set $F = 1$, {\bf else} $F = 0$
\EndIf
\If {$F = 1$ and $\lambda_{\min} (\bar{V}_{\ti}) \geqslant (\frac{2 \beta_{\ti}}{\alpha \rbti})^2$} \label{cons}
\State Choose $x_{\ti} = x_{\ti}^{\prime}$, get the feature vector $\psi_i(x_{\ti}^{\prime}, \mu_t)$, and receive the reward $y_{\ti}$ \label{line:opt}
\Else
\State Choose $x_{\ti} = x_{\bti}$, get conservative feature vector $\psixc = (1 - \rho) \psixbic + \rho \zeta_{\ti}$ and $y_{\bti}$ 
\EndIf
\State Update $U_{\inew} = U_{\inew} + \psixc y_{\ti}$, $W_{\inew} = W_{\inew}+\psixc \psixc^{\top}$,  $V_{\ti} = \lambda I + W_{\isyn} + W_{\inew}$ \label{line:update}
%\State $V_{\ti} = \lambda I + W_{\isyn} + W_{\inew}$
\If {$\log(\det(V_{\ti})/\det(V_{last}))\cdot(t-t_{last}) \geqslant B$}
\State Send a synchronization signal to the server to start a communication round
\EndIf
\State {\bf Synchronization round:}
\If {a communication round is started}
\State All agents $i \in [M]$ send $W_{\inew}, U_{\inew}$ to server
\State Server computes $W_{\mathrm{syn}} = W_{\mathrm{syn}} + \sum_{i = 1}^{M}W_{\inew}, U_{\mathrm{syn}} = U_{\mathrm{syn}} + \sum_{i = 1}^{M} U_{\inew}$
\State All agents receive $W_{\mathrm{syn}}, U_{\mathrm{syn}}$ from server
\State  Set $ W_{\inew} = U_{\inew} = 0, t_{last} = t, $ for all $i$, $ V_{last} = \lambda I + W_{\mathrm{syn}} $
\EndIf
\EndFor
\EndFor
\end{algorithmic}
\end{algorithm}
If the optimization is feasible, the agents choose their respective optimistic action $x_{\ti}^{\prime}$ and get the feature vector $\psi_i(x_{\ti}^{\prime}, \mu_t)$ under a certain condition; otherwise, the agents choose their baseline action $x_{\bti}$ and get the conservative feature vector $\psixc = (1 - \rho) \psixbic + \rho \zeta_{\ti}$. This approach, introduced in \cite{khezeli2020safe} and later applied in \cite{moradipari2020stage, chaudhary2022safe}, ensures the agents learn even when not satisfying the performance constraint in a round. We assume the agents know the baseline action $x_{\bti}$ and its corresponding expected reward $\rbti$ \cite{khezeli2020safe, moradipari2020stage, chaudhary2022safe}. Here, $\alpha$ is a known parameter, similar to \cite{kazerouni2017conservative} and \cite{wu2016conservative}. After receiving their reward $y_{\ti}$, agents update their local parameters, which are then used as the basis for updating their respective confidence set and pruned action sets. 

In the synchronization phase,  at predetermined time intervals, agents exchange all the latest gathered estimates. We refer to the rounds between these synchronization points as {\it epochs}. Such a synchronization method was introduced in \cite{wang2019distributed}, designed based on the observation in \cite{abbasi2011improved} that the change in the determinant of $\bar{V}_{\ti}=\lambda I + \psixc \psixc^{\top}$ is a good indicator of the learning progress. Specifically, synchronization happens only when agent $i$ recognizes that the log-determinant of $\bar{V}_{\ti}$ has changed by more than a constant factor since the last synchronization. This method reduces the communication cost of the algorithm. 
Next, we will explain the construction of the confidence set and the pruned action set.

\noindent{\bf Construction of the Confidence Set $\B_{\ti}$:}
%We will now present the procedure for constructing the confidence set. 
After obtaining the estimate $\hth_{\ti}$ of the unknown parameter $\ths$, we construct the confidence set $\B_{\ti}$ as follows.
%\vspace{-1 mm}
\begin{align}
&  { \B_\ti = \Big\{\theta \in \mathbb{R}^d:\|\hth_{\ti}-\theta\|_{\bar{V}_{\ti}} \leqslant \beta_{\ti} \Big\}, \mbox{~where~}} \label{eq:conf} \\
   &  { \beta_\ti  \hspace{-0.75 mm}=\hspace{-0.75 mm} \beta_\ti(\sigma, \delta) \hspace{-0.7 mm}=\hspace{-0.7 mm} \sigma \sqrt{2\log\Big(\dfrac{\det(\bar{V}_\ti)^{1/2}\det(\lambda I)^{-1/2}}{\delta} \Big)} + \lambda^{1/2}}\nonumber\\
   &{ \mbox{~and~} \bar{V}_\ti = \lambda I + W_{\inew}, \hth_\ti = \bar{V}_\ti^{-1}U_{\inew}}.\nonumber
  \end{align}
In Lemma~\ref{confidence} we show that by setting $\beta_{\ti} = \beta_{\ti}(\sqrt{1 + \sigma^2}, \delta / 2)$, we can construct the confidence set $\B_{\ti}$ such that the reward parameter $\ths$ will always be contained within the confidence set $\B_{\ti}$ with a high probability. 

\noindent{\bf Construction of Pruned Action Set $\X_{\ti}$:}
In standard constrained bandit settings \cite{moradipari2020stage}, verifying whether an action satisfies the constraints is straightforward because the context is known. As a result, the safe action set is typically constructed by directly checking each action against the constraints and eliminating those that fail to meet them. However, in our setting, the context is unknown, making such a direct approach to constructing the safe action set infeasible.
%Given the agent's lack of knowledge regarding the parameter $\ths$, it becomes incapable of distinguishing between actions that are safe and those that are not prior to their execution.
Since the agents only observe the context distribution and the exact contexts are unknown, we need to use estimated feature maps to construct a feasible action set. 
%We cannot directly adapt the approach in \cite{moradipari2020stage} since it requires an exact feature vector.
We now present our approach to tackling this issue by constructing a pruned action set, $\X_{\ti}$, to eliminate unsafe actions. 

In every iteration, each agent refines its action set by excluding actions that fail to satisfy the baseline condition. This is further complicated by the unknown nature of the actual context, rendering the utilization of the feature vector $\pixc$ impractical for constructing the pruned action set. Our approach utilizes $\hth_{\ti}$, $\bar{V}_{\ti}$, and $\psixc$.
The pruned action set aims to eliminate actions violating the baseline constraint, necessitating criteria for excluding unsafe actions.
While constructing the pruned action set, we analyze two cases. 1)~$\pixc^{\top} \hth_{\ti} \geqslant \psixc^{\top} \hth_{\ti}$ and 2)~$\pixc^{\top} \hth_{\ti} \leqslant \psixc^{\top} \hth_{\ti}$. We address case 1) first and explain the process of constructing a subset of actions that satisfy the constraint for all $v \in \B_{\ti}$.
We provide alternate sufficient conditions for deriving the safe action set. We use the symbol `$\Leftarrow$' to show that a certain condition implies another.
Define
\begin{align}
&\scalemath{0.95}{\X_{\ti}^1 := \{x_{\ti} \in \A: \pixc^{\top} v \geqslant (1 - \alpha) \rbti, \forall~ v \in \B_{\ti} \}}\label{eq:prun1}\hspace{-5 mm}\\
\hspace{-1 mm}&\scalemath{0.95}{\Leftarrow \{x_{\ti} \in \A: \pixc^{\top} (v - \hth_{\ti}) + \psixc^{\top} \hth_{\ti}}\nonumber \\
\hspace{-1 mm}&\scalemath{0.95}{+ (\pixc - \psixc)^{\top} \hth_{\ti}\geqslant (1 - \alpha) \rbti, \forall v \in \B_{\ti} \}} \nonumber \\
\hspace{-1 mm}&\scalemath{0.95}{\Leftarrow \{x_{\ti} \in \A: \psixc^{\top} \hth_{\ti}\geqslant \hspace{-1 mm} \frac{\beta_{\ti}}{\scalemath{0.95}{\sqrt{\lambda_{\min} (\bar{V}_{\ti})}}} \hspace{-0.7 mm}+\hspace{-0.7 mm} (1 \hspace{-0.7 mm}-\hspace{-0.7 mm} \alpha) \rbti \}} \label{eq:prun2}
\hspace{-2 mm}
\end{align}
where the last step follows from $\pixc^{\top} \hth_{\ti} \geqslant \psixc^{\top} \hth_{\ti}$ and $\pixc^{\top} (v - \hth_{\ti}) \geqslant - \frac{\beta_{\ti}}{\sqrt{\lambda_{\min} (\bar{V}_{\ti})}}$  from Lemma~\ref{L1}. 
All actions that meet the conditions in  Eq.~\eqref{eq:prun2} also fulfill the requirements of Eq.~\eqref{eq:prun1}, thus ensuring safety. 

Now we consider case~2), where $\pixc^{\top} \hth_{\ti} \leqslant \psixc^{\top} \hth_{\ti}$. In this case, our approach is to first identify actions that violate the baseline constraint, $\bar{\X}_{\ti}^2$, and then eliminate those actions from the action set $\A$.  
\begin{align}
&\scalemath{0.95}{\bar{\X}_{\ti}^2 := \{x_{\ti} \in \A: \pixc^{\top} v \leqslant (1 - \alpha) \rbti, \forall v \in \B_{\ti} \}}\label{eq:prun3}\hspace{-5 mm} \\
&\scalemath{0.95}{\Leftarrow \{x_{\ti} \in \A: \pixc^{\top} (v - \hth_{\ti})} \nonumber \\
&\scalemath{0.95}{+ \pixc^{\top} \hth_{\ti} \leqslant (1 - \alpha) \rbti, \forall v \in \B_{\ti} \} }\nonumber \\
&\scalemath{0.95}{\Leftarrow \{x_{\ti} \in \A: \psixc^{\top} \hth_{\ti} \leqslant \hspace{-1 mm}\frac{-\beta_{\ti}}{\sqrt{\lambda_{\min} (\bar{V}_{\ti})}}\hspace{-1 mm} +\hspace{-1 mm} (1 \hspace{-1 mm}-\hspace{-1 mm} \alpha) \rbti \}}\label{eq:prun4} 
\end{align}
where the last step follows from $\pixc^{\top} \hth_{\ti} \leqslant \psixc^{\top} \hth_{\ti}$ and $\pixc^{\top} (v - \hth_{\ti}) \leqslant \frac{\beta_{\ti}}{\sqrt{\lambda_{\min} (\bar{V}_{\ti})}}$  from Lemma~\ref{L1}.
Note that, all actions that meet the conditions in  Eq.~\eqref{eq:prun4} also fulfill the requirements of Eq.~\ref{eq:prun3}, consequently rendering them unsafe. 
By taking the difference between $\A$ and $\bar{\X}_{\ti}^2$, we determine $\X_{\ti}^2 = \A \setminus \bar{\X}_{\ti}^2$
\begin{align*}
%&\X_{\ti}^2 = \A \setminus \bar{\X}_{\ti}^2 \\
&\scalemath{0.95}{= \{x_{\ti} \in \A: \psixc^{\top} \hth_{\ti} \geqslant\hspace{-1 mm} \frac{-\beta_{\ti}}{\sqrt{\lambda_{\min} (\bar{V}_{\ti})}} + (1 - \alpha) \rbti \}}.
\end{align*}
% \[\X_{\ti}^2 = \A \setminus \bar{\X}_{\ti}^2 = \{x_{\ti} \in \A: \psixc^{\top} \hth_{\ti} \geqslant - \frac{\beta_{\ti}}{\sqrt{\lambda_{\min} (\bar{V}_{\ti})}} + (1 - \alpha) \rbti \}.\] 
%
Given $\X_{\ti}^1$ and $\X_{\ti}^2$, we obtain the pruned action set by taking the intersection between $\X_{\ti}^1$ and $\X_{\ti}^2$, given by 
\[\X_{\ti} = \{x_{\ti} \in \A: \psixc^{\top} \hth_{\ti} \geqslant \hspace{-1 mm}\frac{\beta_{\ti}}{\sqrt{\lambda_{\min} (\bar{V}_{\ti})}} \hspace{-0.5 mm}+\hspace{-0.5 mm} (1 - \alpha) \rbti \}.\] 
At round $t$, each agent chooses a pair $(x_{\ti}^{\prime}, \tilde{\theta}_{\ti})$, where  $x_{\ti}^{\prime} \in \X_{\ti}$ and $\tilde{\theta}_{\ti} \in \B_{\ti}$ that jointly maximizes the current reward while ensuring that the baseline constraint is met. That is,
$$
\vspace{-1 mm}
(x_{\ti}^{\prime}, \tilde{\theta}_{\ti}) = \mathop{\arg\max}_{(x, \theta) \in \X_{\ti} \times \B_{\ti}} \left\langle \psiixc, \theta \right\rangle.
\vspace{-1 mm}
$$
The pruned action set $\X_{\ti}$ is constructed by considering all $\theta \in \B_{\ti}$, not just $\ths$. If the pruned action set is non-empty and satisfies a constraint $\lambda_{\min} (\bar{V}_{\ti})\geqslant (\frac{2 \beta_{\ti}}{\alpha \rbti})^2$, the agent chooses $x_{\ti}^{\prime}$ and get the feature vector $\psi_i(x_{\ti}^{\prime}, \mu_t)$; otherwise, the agent chooses the baseline action $x_{\bti}$ and get the conservative feature vector $\psixc = (1 - \rho) \psixbic + \rho \zeta_{\ti}$, which is detailed below. We consider the constraint on $\lambda_{\min} (\bar{V}_{\ti}) \geqslant (\frac{2 \beta_{\ti}}{\alpha \rbti})^2$, as it ensures that,  with high probability, the best action $\xs$  always belongs to the pruned action set. We will provide detailed proof for this claim later in Lemma~\ref{L4}. A similar condition on the smallest eigenvalue of the Gram matrix is used in \cite{moradipari2020stage}, yet our proof approach is different from \cite{moradipari2020stage} due to the reliance on known feature vectors (contexts), which contrasts with our scenario.
We note that since $\bar{\X}_{\ti}^2$ is a subset of the actions that are unsafe, $\X_{\ti}^2$ may contain some unsafe actions. 
%The pruned action set might include some actions that do not satisfy the baseline constraint. 
We will demonstrate in Lemma~\ref{L2} that when the agent's action is played, the chosen action from the pruned set will not violate the baseline constraint with high probability. 
\vspace{2 mm}

\noindent{\bf Conservative feature vector:}
In our problem, each agent $i$ is assigned a baseline policy, and in each round $t$, a baseline action $x_{\bti} $ is recommended based on that policy. The agent's objective is to carry out explorations, ensuring that the rewards achieved from exploratory action $x_{\ti}$ 
remain reasonably comparable to the rewards from the baseline action. Our approach draws inspiration from \cite{khezeli2020safe} and \cite{moradipari2020stage}, where the conservative feature vectors for the baseline actions are combined with random exploration while maintaining adherence to stage-wise safety constraints. We construct a conservative feature vector $\psixc$ to be a convex combination of the baseline action's feature vector $\psixbic$ and a random noise vector $\zeta_{\ti}$, expressed as
$
\psixc = (1 - \rho) \psixbic + \rho \zeta_{\ti},
$
where $\zeta_{\ti}$ is a sequence of independent random vectors with zero means, and we assume that $\|\zeta_{\ti}\|_2 = 1$. $\rho$ is a constant value in the range $(0, \frac{\alpha r_l}{1 + r_h}]$, and we can ensure that when $\rho \in (0, \frac{\alpha r_l}{1 + r_h}]$, the conservative feature vector is always in the feature vector set $\set_{\ti}$ (Lemma~\ref{L3}). By identifying the conservative feature vector through this convex combination, we can guarantee that the agent continues to 
%gain insights even when selecting the conservative feature vector, thereby ensuring that the agent will 
learn
in every round. 

\subsection{Theoretical Analysis on Safety Guarantees}\label{app:sup}
In this section, we establish safety guarantees for the DiSC-UCB algorithm. We begin with two preliminary results in Lemmas~\ref{confidence} and \ref{L1}. Then, in Lemmas~\ref{L2}, \ref{L3}, and \ref{L4}, we prove that the pruned safe action set satisfies the performance constraints while ensuring the optimal actions are preserved.
\begin{lemma}\label{confidence}
For any $\delta > 0$, with a probability of $1 - M \delta$, $\ths$ will always exist inside the confidence set $\B_{\ti}$ defined by Eq.~\eqref{eq:conf} where $\beta_\ti = \beta_\ti(\sqrt{1 + \sigma^2}, \delta / 2)$ for all value of $t$ and $i$. 
\end{lemma}
\begin{proof}
Drawing inspiration from Theorem 1 in \cite{kirschner2019stochastic}, we implement a similar approach to analyze the reward $y_{\ti}$, In particular, we observe that the reward $y_{\ti} = \pixc^{\top} \ths + \eta_{\ti}$ can be alternatively represented as 
$$
y_{\ti} = \psixc^{\top} \ths + \xi_{\ti} + \eta_{\ti},
$$
where $\xi_{\ti} = (\pixc - \psixc)^{\top} \ths$. In this representation, $(\xi_{\ti} + \eta_{\ti})$ serves as the noise component associated with $\psixc^{\top} \ths$. Given that $|\xi_{\ti}| \leqslant 1$, $\xi_{\ti}$ is a $1$-subgaussian. Therefore, $y_{\ti}$ can be viewed as an observation of reward obtained from $\psixc^{\top} \ths$ with the presence of noise parameterized by $\sqrt{1 + \sigma^2}$. Thus, we can define the confidence bound for the least squares estimator of $\psixc$ while guaranteeing that $\ths$ is always present within it with probability $1 - \delta$, using updated parameters. The parameter $\rho$ for $\beta_{\ti}$ is given by $\sqrt{1 + \sigma^2}$. Further, when considering the upper bound of cumulative regret, we use the Azuma-Hoeffding inequality to determine the upper bound of the regret gap term from the context distribution with probability $1 - \delta$. By using the union bound, the probabilities can be combined, providing a resulting probability of $1 - 2 \delta$. We propose a substitution such that $\delta^{\prime} = 2 \delta$. Therefore, the parameter $\delta$ for our $\beta_{\ti}$ is changed to $\frac{\delta^{\prime}}{2}$. Finally, our proof is completed by considering the presence of $M$ agents and using the union bound. 
\end{proof}
According to Lemma~\ref{confidence}, with a probability of $1 - M \delta$, $\ths$ is always included inside the confidence set $\B_{\ti}$ for all values of $t$ and $i$. Thus, Lemma~\ref{confidence} guarantees that for each round $t$, $\ths \in \B_{\ti}$ holds for every agent $i$ with probability at least $1-M\delta$. 
We have the following result for Algorithm~\ref{alg:TV}.
\begin{lemma} \label{L1}
In the DiSC-UCB algorithm, Algorithm~\ref{alg:TV}, all values of  $ v \in \B_{\ti}$, satisfy the following two inequalities.
$$
- \frac{\beta_{t, i}}{\sqrt{\lambda_{\min} (\bar{V}_{t, i})}} \leqslant \pixc^{\top} (\hth_{\ti} - v) \leqslant \frac{\beta_{t, i}}{\sqrt{\lambda_{\min} (\bar{V}_{t, i})}} \mbox{~and}
$$  
$$
- \frac{\beta_{t, i}}{\sqrt{\lambda_{\min} (\bar{V}_{t, i})}} \leqslant \psixc^{\top} (\hth_{\ti} - v) \leqslant \frac{\beta_{t, i}}{\sqrt{\lambda_{\min} (\bar{V}_{t, i})}}.
$$
\end{lemma}
\begin{proof}
We begin by noticing that $\bar{V}_{t, i}$ is a symmetric semi-positive definite matrix thus, it can be decomposed by eigenvalues. From this, we can derive the inequality as follows.
\begin{align}
&\|\pixc\|_{\bar{V}_{t, i}^{-1}} = \sqrt{\pixc^{\top} \bar{V}_{t, i}^{-1} \pixc} \nonumber \\
&\leqslant \sqrt{\pixc^{\top} (Q \Sigma Q^{\top})^{-1} \pixc} \nonumber \\
%&= \sqrt{\pixc^{\top} Q \Sigma^{-1} Q^{\top} \pixc} \nonumber \\
%&\leqslant \sqrt{\pixc^{\top} Q \frac{1}{\lambda_{\min} (\bar{V}_{t, i})} I Q^{\top} \pixc} \nonumber \\
&\leqslant \sqrt{\frac{\pixc^{\top} Q Q^{\top} \pixc}{\lambda_{\min} (\bar{V}_{t, i})}} \nonumber \\
&= \sqrt{\frac{\pixc^{\top} \pixc}{\lambda_{\min} (\bar{V}_{t, i})}} = \frac{\|\pixc\|_2}{\sqrt{\lambda_{\min} (\bar{V}_{t, i})}} \leqslant \frac{1}{\sqrt{\lambda_{\min} (\bar{V}_{t, i})}}. \nonumber
\end{align}
Then, by using the Cauchy–Schwarz inequality, we can write
\begin{align}
&\pixc^{\top} (\hth_{\ti} - v) \leqslant \|\pixc\| \|\hth_{\ti} - v\| \nonumber \\
&= \|\bar{V}_{t, i}^{1 / 2} \pixc\|_{\bar{V}_{t, i}^{-1}} \|\bar{V}_{t, i}^{- 1 / 2} (\hth_{\ti} - v)\|_{\bar{V}_{t, i}} \nonumber \\
&\leqslant \|\bar{V}_{t, i}^{1 / 2}\|_{\bar{V}_{t, i}^{-1}} \|\pixc\|_{\bar{V}_{t, i}^{-1}} \|\bar{V}_{t, i}^{- 1 / 2} \|_{\bar{V}_{t, i}} \|(\hth_{\ti} - v)\|_{\bar{V}_{t, i}} \nonumber 
\end{align}
\begin{align}
&= \|\pixc\|_{\bar{V}_{t, i}^{-1}} \|(\hth_{\ti} - v)\|_{\bar{V}_{t, i}} \leqslant \frac{\beta_{t, i}}{\sqrt{\lambda_{\min} (\bar{V}_{t, i})}} \nonumber
\end{align}
Similarly, we derive the second condition. 
%that $- \frac{\beta_{t, i}}{\sqrt{\lambda_{\min} (\bar{V}_{t, i})}} \leqslant \psixc^{\top} (\hth_{\ti} - v) \leqslant \frac{\beta_{t, i}}{\sqrt{\lambda_{\min} (\bar{V}_{t, i})}}$ and this completes the proof. 
\end{proof}

In the next lemma, we show that an action chosen by the learning agent in line~\ref{line:opt} of Algorithm~\ref{alg:TV} satisfies the baseline constraint. Let us define $\xsb = \argmax_{\xsb \in \A} r_{x_{\ti}, c}$.

\begin{lemma} \label{L2}
In the DiSC-UCB algorithm, Algorithm~\ref{alg:TV}, with probability $1 - M \delta$, any action chosen by the agent from the pruned action set $\X_{\ti}$ satisfies the performance constraint if $\lambda_{\min} (\bar{V}_{\ti}) \geqslant (\frac{2 \beta_{\ti}}{\alpha \rbti})^2$. 
\end{lemma}
\begin{proof}
Let $\xb$ denote an arbitrary action in the pruned action set $\X_{\ti}$ that does not meet the performance constraint. Our goal is to show that when the agent's action is played, an action in $\X_{\ti}$ that satisfies the performance constraint will be selected, i.e., $\xb$ is not selected. Based on the definition of $\xsb$, it can be observed that $\xsb$ is always contained within the pruned action set $\X_{\ti}$ and meets the performance constraint. To this end, if we show 
$$
\max_{\theta_1 \in \B_{\ti}} \psixsbc^{\top} \theta_1 > \max_{\theta_2 \in \B_{\ti}} \psixbc^{\top} \theta_2,
$$
it ensures that actions within $\X_{\ti}$ that violate the baseline constraint are never selected. Since $\xb$ does not satisfy the baseline constraint, we know
$
\pixbc^{\top} \ths < (1 - \alpha) \rbti
$
which leads to 
\begin{equation}
  \label{L2_3}
  \psixbc^{\top} \ths = \bE [\pixbc^{\top} \ths] < (1 - \alpha) \rbti.
\end{equation}
Moreover, we have
\begin{equation}
  \label{L2_1}
  \psixsbc^{\top} \ths - \rbti = \bE [\pixsbc^{\top} \ths - \rbti] \geqslant 0. 
\end{equation}
To show 
$
\max_{\theta_1 \in \B_{\ti}} \psixsbc^{\top} \theta_1 > \max_{\theta_2 \in \B_{\ti}} \psixbc^{\top} \theta_2, 
$
based on Lemma~\ref{confidence}, it is sufficient to demonstrate that with probability $1 - M \delta$, 
% \begin{align*}
% \psixsc^{\top} \ths &\geqslant \max_{\theta_2 \in \B_{\ti}} \psixbc^{\top} \ths + \psixbc^{\top} (\theta_2 - \hth_{\ti}) \\
% &+ \psixbc^{\top} (\hth_{\ti} - \ths)
% \end{align*}
$
\psixsbc^{\top} \ths > \max_{\theta_2 \in \B_{\ti}} [\psixbc^{\top} \ths + \psixbc^{\top} (\theta_2 - \hth_{\ti}) + \psixbc^{\top} (\hth_{\ti} - \ths)].
$
By using Eq.~\eqref{L2_3}, Eq.~\eqref{L2_1} and Lemma~\ref{L1}, we prove the aforementioned inequality always holds if $\lambda_{\min} (\bar{V}_{\ti}) \geqslant (\frac{2 \beta_{\ti}}{\alpha \rbti})^2$. 
% $$
% \rbti \geqslant (1 - \alpha) \rbti + \frac{2 \beta_{\ti}}{\sqrt{\lambda_{\min} (\bar{V}_{\ti})}}
% $$
% Then, we need to show $\lambda_{\min} (\bar{V}_{\ti}) \geqslant (\frac{2 \beta_{\ti}}{\alpha \rbti})^2$ which is always true according to Eq.~\eqref{L2_2}. 
Therefore, we conclude that with probability $1 - M \delta$, any action chosen by Algorithm~\ref{alg:TV} from the pruned action set $\X_{\ti}$ satisfies the performance constraint if $\lambda_{\min} (\bar{V}_{\ti}) \geqslant (\frac{2 \beta_{\ti}}{\alpha \rbti})^2$. 
\end{proof}

\begin{lemma} \label{L3}
At each round $t$, given the fraction $\alpha$, for any $\rho \in (0, \bar{\rho}]$, where $\bar{\rho} = \frac{\alpha r_l}{1 + r_h}$, the conservative feature vector $\psixc = (1 - \rho) \psixbic + \rho \zeta_{\ti}$ is safe. 
\end{lemma}
\begin{proof}
To demonstrate the safety of the conservative feature vector $\psixc = (1 - \rho) \psixbic + \rho \zeta_{\ti}$, we need to show that $((1 - \rho) \pixbic + \rho \zeta_{\ti})^{\top} \ths \geqslant (1 - \alpha) \rbti$ always holds. This can be shown by verifying the following condition
$$
\rbti - \rho \rbti + \rho \zeta_\ti^{\top} \ths \geqslant (1 - \alpha) \rbti,
$$
which is equivalent to
$
\rho (\rbti - \zeta_\ti^{\top} \ths) \leqslant \alpha \rbti
$
By applying Cauchy Schwarz inequality, we deduce
\begin{equation}
\rho \leqslant \frac{\alpha \rbti}{1 + \rbti} \label{range: rho}
\end{equation}
Consequently, by setting a lower bound for the right-hand side of Eq.~\eqref{range: rho} with the assumption that $r_l \leqslant \rbti \leqslant r_h$, we get
$
\rho \leqslant \frac{\alpha r_l}{1 + r_h}. 
$
Therefore, for any $\rho \leqslant \frac{\alpha r_l}{1 + r_h}$, the conservative feature vector $\psixc = (1 - \rho) \psixbic + \rho \zeta_{\ti}$ is  safe. 
\end{proof}
\begin{remark}
  Lemma~\ref{L2} and Lemma~\ref{L3} thus jointly prove that all the actions chosen by the proposed DiSC-UCB algorithm guarantees safety constraints.  
\end{remark}

Next, we show optimal action $\xs$ always exists within the pruned action set when $\lambda_{\min} (\bar{V}_{\ti}) \geqslant (\frac{2 \beta_{\ti}}{\alpha \rbti})^2$. We extend the approach in Lemma C.1 in \cite{moradipari2020stage} to the unknown context case.
%s are {\em known} in \cite{moradipari2020stage}. We present Lemma~\ref{L4} and proof below for completeness.

\begin{lemma} \label{L4}
Let $\lambda_{\min} (\bar{V}_{\ti}) \geqslant (\frac{2 \beta_{\ti}}{\alpha \rbti})^2$. Then, with probability $1 - M \delta$, the optimal action $\xs$ lies in the pruned action set $\X_{\ti}$ for all $M$ agent, i.e., $\xs \in \X_{\ti}$. 
\end{lemma}
\begin{proof}
To prove the optimal action $\xs$ always exists in the pruned action set under the condition on the smallest eigenvalue of the Gram matrix, we need to show 
$$
\psixsc^{\top} \hth_{\ti} \geqslant \frac{\beta_{\ti}}{\sqrt{\lambda_{\min} (\bar{V}_{\ti})}} + (1 - \alpha) \rbti,
$$
which is equivalent to demonstrating 
% \begin{align*}
% &\psixsc^{\top} (\hth_{\ti} - \ths) + \psixsc^{\top} \ths \\
% &\geqslant \frac{\beta_{\ti}}{\sqrt{\lambda_{\min} (\bar{V}_{\ti})}} + (1 - \alpha) \rbti.
% \end{align*}
$
\psixsc^{\top} (\hth_{\ti} - \ths) + \psixsc^{\top} \ths \geqslant \frac{\beta_{\ti}}{\sqrt{\lambda_{\min} (\bar{V}_{\ti})}} + (1 - \alpha) \rbti.
$
By using Lemma~\ref{confidence}, it can be determined that with probability $1 - M \delta$, $\ths$ lies within the confidence set $\B_{\ti}$. Subsequently, applying Lemma~\ref{L1}, we recognize that with probability $1 - M \delta$, $\psixsc^{\top} (\hth_{\ti} - \ths) \geqslant - \frac{\beta_{\ti}}{\sqrt{\lambda_{\min} (\bar{V}_{\ti})}}$, and it is also known that $\psixsc^{\top} \ths - \rbti \geqslant \psixsbc^{\top} \ths - \rbti = \bE [\pixsbc^{\top} \ths - \rbti] \geqslant 0$. 
Hence, the sufficient condition for our result is $\lambda_{\min} (\bar{V}_{\ti}) \geqslant (\frac{2 \beta_{\ti}}{\alpha \rbti})^2$. 
Thus, the above analysis verifies that with probability $1 - M \delta$, the optimal action $\xs$ is always present in the pruned action set with probability $1 - M \delta$ under the condition $\lambda_{\min} (\bar{V}_{\ti}) \geqslant (\frac{2 \beta_{\ti}}{\alpha \rbti})^2$. 
\end{proof}

\subsection{Regret Analysis}\label{sec:reg}
%{\cblue Further, since the contexts are unknown, it introduces an additional regret gap. To bound this gap, we use concepts like Azuma-Hoeffding inequality, as shown in \cite{kirschner2019stochastic}, and we lower bound the number of times conservative actions are played in contrast to \cite{moradipari2020stage} where only the upper bound is enough. To determine an upper bound of cumulative regret, we utilize the upper and lower bounds for the number of times conservative actions are played.}
%
% Our approach in this study begins with transforming cumulative regret into separate and independent terms. Following this, we proceed with determining the upper bound for each term. The main challenge in Algorithm~\ref{alg:TV} arises from the need to determine the number of choices made by the agent in the learner's action and the conservative feature vector. After figuring out these upper bounds, we can determine the upper bound of the cumulative regret. The following discussions present the main findings of this paper, including the lemmas and theorems that contribute to the upper bound of cumulative. 
In this section, we derive regret and communication bounds for DiSC-UCB. Let $|N_{t - 1}|$ denote the set of rounds $j < t$ where our algorithm chooses the safe action, and similarly, $|N_{t - 1}^c| = \{1, \cdots, t - 1\} - |N_{t - 1}|$ represents the set of rounds $j < t$ where our algorithm selects conservative actions. 

In Theorem~\ref{prop1}, we decompose our cumulative regret into three terms. 
The first two terms reflect the regret of choosing the agents' suggested actions. 
%This scenario refers to the unconstrained case studied in our earlier work \cite{lin2023distributed1}. 
%For details and proof, we refer to \cite{lin2023distributed1}.
%We present the analysis of these two terms in Appendix~\ref{app_3}. 
%The upper bound for Term~1 is established by bounding the term $|N_T|$. 
Term~2 is due to the agent's limitation of only observing the context distribution $\mu_t$ without accessing the exact context $c_t$. 
To quantify this term, we reduce it to a martingale difference sequence and then apply the Azuma-Hoeffding inequality.
%The upper bound for this term is determined by reducing it to a martingale difference sequence and then applying the Azuma-Hoeffding inequality.
 Term~3 results from selecting the conservative feature vector whenever the actions suggested by the agents do not meet the constraints.
%By using the definition of the conservative feature vector, it is achievable to simplify this term to an upper bound that includes the term $|N_T^c|$. 
\begin{theorem} \label{prop1}
The regret of the DiSC-UCB algorithm, Algorithm~\ref{alg:TV}, can be decomposed into three terms as follows
\begin{align*}
\R_T &\leqslant \underbrace{4 \beta_T \sqrt{M |N_T| d \log (M |N_T|)} (1 + \log (M |N_T|))}_{Term~1} \\
&+ \underbrace{4 \sqrt{2 M |N_T| \log (\frac{2}{\delta})}}_{Term~2} + \underbrace{\sum_{i = 1}^M \sum_{t \in |N_T^c|} (\kappa_h + \rho r_h + \rho)}_{Term~3}.
\end{align*}
\end{theorem}
\begin{proof}
%Now, we demonstrate the upper bound of our cumulative regret.
Let $\tau$ be the last round in which Algorithm~\ref{alg:TV} plays the agent's action, $\tau = \max \{1 \leqslant t \leqslant T \mid x_{\ti} = x_{\ti}^{\prime}\}$. 
By the definition of cumulative regret, we have
\begin{align}
&\R_T= \sum_{i = 1}^M \sum_{t = 1}^T (\pixsc^{\top} \ths - \pixc^{\top} \ths) \nonumber \\
&= \sum_{i = 1}^M \sum_{t \in |N_T|} (\pixsc^{\top} \ths - \pixc^{\top} \ths)\nonumber \\
&+ \sum_{i = 1}^M \sum_{t \in |N_T^c|} (\pixsc^{\top} \ths - \pixc^{\top} \ths) \nonumber \\
&= \sum_{i = 1}^M \sum_{t \in |N_T|} (\pixsc^{\top} \ths - \pixc^{\top} \ths)\nonumber \\
&+ \sum_{i = 1}^M \sum_{t \in |N_T^c|} (\pixsc^{\top} \ths - (1 - \rho) \pixbic^{\top} \ths - \rho \zeta_{\ti}^{\top} \ths) \nonumber \\
&\leqslant \scalefont{0.8}{\underbrace{\sum_{i = 1}^M \sum_{t \in |N_T|} (\pixsc^{\top} \ths - \pixc^{\top} \ths)}_{Agents'~term}+ \sum_{i = 1}^M \sum_{t \in |N_T^c|} (\kappa_{\bti} + \rho \rbti + \rho)} \label{R_1}
\end{align}
\begin{align}
&\leqslant {\scalefont{0.9}\underbrace{4 \beta_T \sqrt{M|N_T|d \log (MT)} + 4\beta_T \sqrt{MTd \log MT} \log(MT)}_{Term~1}}\nonumber \\
&{\scalefont{0.9}+ \underbrace{\sqrt{2 M |N_T| \log \frac{2}{\delta}}}_{Term~2} + \underbrace{\sum_{i = 1}^M \sum_{t \in |N_T^c|} (\kappa_h + \rho r_h + \rho)}_{Term~3}} \label{R_2},
\end{align}
where Eq.~\eqref{R_1} follows from definition of $\kappa_{\bti} := \pixsc^{\top} \ths - \rbti$, and $ - \zeta_{\ti}^{\top} \ths \leqslant |\zeta_{\ti}^{\top} \ths| \leqslant \|\zeta_{\ti}\| \|\ths\| \leqslant 1$. Eq.~\eqref{R_2} is derived through the analysis of the {\em Agents' term} `in Eq.~\eqref{R_1} and $\kappa_{\bti} \leqslant \kappa_h$, $\rbti \leqslant r_h$. The Agents' term captures the unconstrained case studied in our earlier work \cite{lin2023distributed1} (see Theorem~4.1, \cite{lin2023distributed1}, presented in  Proposition~\ref{ECC}).  Given that $|N_T| = T - |N_T^c|$, the remaining section of our proof focuses on determining the upper bound and lower bound for $|N_T^c|$. 
\end{proof}
Consider any round $t$ during which the agent plays the agent's action, i.e., at round $t$, both condition $F = 1$ is met and $\lambda_{\min} (\bar{V}_{\ti}) \geqslant (\frac{2 \beta_{\ti}}{\alpha \rbti})^2$ is satisfied. 
%When $F = 1$, we are assured that there exists an action $x_{\ti}$ in the pruned action set $\X_{\ti}$ such that $\psixc^{\top} \hth_{\ti} \geqslant \frac{\beta_{\ti}}{\sqrt{\lambda_{\min} (\bar{V}_{\ti})}} + (1 - \alpha) \rbti$. 
By Lemma~\ref{L4}, if $\lambda_{\min} (\bar{V}_{\ti}) \geqslant (\frac{2 \beta_{\ti}}{\alpha \rbti})^2$, it is guaranteed that $\xs \in \X_{\ti}$. Consequently, $\lambda_{\min} (\bar{V}_{\ti}) \geqslant (\frac{2 \beta_{\ti}}{\alpha \rbti})^2$ is sufficient to guarantee that $\X_{\ti}$ is non-empty. To this end, our analysis of Algorithm~\ref{alg:TV} will henceforth only focuses on the condition $\lambda_{\min} (\bar{V}_{\ti}) \geqslant (\frac{2 \beta_{\ti}}{\alpha \rbti})^2$. 

\begin{lemma} \label{L5}
The smallest eigenvalue of the Gram matrix $\lambda_{\min} (\bar{V}_{\ti})$ satisfies $\lambda_{\min}  (\bar{V}_{\ti})$ is upper bounded by
\begin{align*}
 & (\lambda + M T) + 2 M \rho (\rho - 1) |N_T^c|+ \sqrt{32 M \rho^2 (1 - \rho)^2 |N_T^c| \log (\frac{d}{\delta})}.
\end{align*}
\end{lemma}
\begin{proof}
We start with the definitions of $\bar{V}_{\ti}, W_{\isyn}, W_{\inew}$
\begin{align}
&\bar{V}_{\ti} = \lambda I + W_{\isyn} + W_{\inew} \nonumber \\
&= \lambda I + \sum_{i = 1}^M \sum_{t = 1}^{t_{last}} \psixc \psixc^{\top} +\hspace{-2 mm} \sum_{t = t_{last}}^T \hspace{-1 mm} \psixc \psixc^{\top} \nonumber \\
&\preceq \lambda I + \sum_{i = 1}^M \sum_{t = 1}^T \psixc \psixc^{\top}\nonumber \\
% \end{align}
% \begin{align}
%&{\cblue= \lambda I + \sum_{i = 1}^M (\sum_{t \in |N_T|}\psixc \psixc^{\top} + \sum_{t \in |N_T^c|}) \psixc \psixc^{\top}} \nonumber \\
&\preceq  \lambda I + \sum_{i = 1}^M \sum_{t \in |N_T|} \psixc \psixc^{\top}+ \sum_{i = 1}^M \sum_{t \in |N_T^c|} ((1 - \rho)  \nonumber \\
&\psixbic + \rho \zeta_{\ti}) ((1 - \rho) \psixbic + \rho \zeta_{\ti})^{\top} \label{eq:newV} \\
&\preceq \lambda I +\hspace{-2 mm} \sum_{i = 1}^M \sum_{t \in |N_T|} I + \sum_{i = 1}^M \sum_{t \in |N_T^c|} ((1 - \rho)^2 I + \rho (1 - \rho) \psixbic \zeta_{\ti}^{\top}\nonumber \\
&+ \rho (1 - \rho) \zeta_{\ti} \psixbic^{\top} + \rho^2 I), \nonumber
\end{align}
where Eq.~\eqref{eq:newV} follows by substituting the feature vectors corresponding to the rounds in which the learner's chosen action is played and those in which the conservative action is played. The last step follows since $\psixc \psixc^{\top} \preceq I$ are rank-1 matrices with $\|\psixc\|_2 \leqslant 1$, and thus their eigenvalues are either $1$ or $0$. By considering the above relationships and using $|N_T| = T - |N_T^c|$, we  obtain 
\begin{equation}
  % \label{V_1}
  \nonumber
  \bar{V}_{\ti} \preceq \lambda I + M (T - |N_T^c|) I + M |N_T^c| (2 \rho^2 - 2 \rho + 1) I + \sum_{i = 1}^M \sum_{t \in |N_T^c|} U_{\ti},
\end{equation}
% \begin{equation}
%   \label{V_2}
%   \bar{V}_{\ti} \preceq \lambda I + M |N_T| I + M (T - |N_T|) (2 \rho^2 - 2 \rho + 1) I + \sum_{i = 1}^M \sum_{t \in |N_T^c|} U_{\ti},
% \end{equation}
where $U_{\ti} = \rho (1 - \rho) \psixbic \zeta_{\ti}^{\top} + \rho (1 - \rho) \zeta_{\ti} \psixbic^{\top}$. 
By applying Weyl's inequality, 
\begin{align}
\lambda_{\min} (\bar{V}_{\ti}) &\leqslant (\lambda + M T) + 2 M \rho (\rho - 1) |N_T^c| + \lambda_{\max} (\sum_{i = 1}^M \sum_{t \in |N_T^c|} U_{\ti}). \label{V_1_1}
\end{align}
% \begin{equation}
%   \label{V_1_1}
%   \lambda_{\min} (\bar{V}_{\ti}) \leqslant (\lambda + M T) + 2 M \rho (\rho - 1) |N_T^c| + \lambda_{\max} (\sum_{i = 1}^M \sum_{t \in |N_T^c|} U_{\ti}).
% \end{equation}
% \begin{equation}
%   \label{V_2_1}
%   \lambda_{\min} (\bar{V}_{\ti}) \leqslant \lambda + M T (2 \rho^2 - 2 \rho + 1) + 2 M \rho (1 - \rho) |N_T| + \lambda_{\max} (\sum_{i = 1}^M \sum_{t \in |N_T^c|} U_{\ti})
% \end{equation}
Next, we use the matrix Azuma inequality to determine the upper bound of $\lambda_{\max} (\sum_{i = 1}^M \sum_{t \in |N_T^c|} U_{\ti})$. From the definition of $U_{\ti}$, it follows that $\bE [U_{\ti} |F_{s - 1}] = 0$ and 
\begin{align}
&\max_{\|u\|_2 = \|v\|_2 = 1} u^{\top} U_{\ti} v = \rho (1 - \rho) (u^{\top} \psixbic) (v^{\top} \zeta_{\ti})^{\top}\nonumber\\
&+ \rho (1 - \rho) (u^{\top} \zeta_{\ti}) (v^{\top} \psixbic)^{\top} \label{s_4_1} 
\end{align}
\begin{align}
&\leqslant \rho (1 - \rho) \|\psixbic\| \|\zeta_{\ti}\| + \rho (1 - \rho) \|\zeta_{\ti}\| \|\psixbic\| \nonumber \\
&\leqslant 2 \rho (1 - \rho) \label{s_4_2},
\end{align}
where Eq.~\eqref{s_4_1} follows from Cauchy-Schwarz inequality, Eq.~\eqref{s_4_2} follows from $\|\psixbic\| \leqslant 1$ and $\|\zeta_{\ti}\| = 1$. Further, we utilize the property that for any matrix A, we have $A^2 \preceq \|A\|_2^2 I$, where $\|A\|_2$ is the maximum singular value of A given by $\sigma_{\max} (A) = \max_{\|u\|_2 = \|v\|_2 = 1} u^{\top} A v$. Thus
$$
U_{\ti}^2 \preceq \sigma_{\max} (U_{\ti})^2 I \preceq 4 \rho^2 (1 - \rho)^2 I.
$$
Moreover, by using the triangular inequality, we can express 
$$
\|\sum_{i = 1}^M \sum_{t \in |N_T^c|} U_{\ti}^2\| \leqslant \sum_{i = 1}^M \sum_{t \in |N_T^c|} \|U_{\ti}^2\| \leqslant 4 M \rho^2 (1 - \rho)^2 |N_T^c|.
$$
Now, by applying the matrix Azuma inequality,  for any $c \geqslant 0$, 
$$
\mP \left(\lambda_{\max} (\sum_{i = 1}^M \sum_{t \in |N_T^c|} U_{\ti}) \geqslant c \right) \leqslant d \exp \big(- \frac{c^2}{32 M \rho^2 (1 - \rho)^2 |N_T^c|}\big).
$$
Thus we have, with probability $1 - \delta$, 
$$
\lambda_{\max} (\sum_{i = 1}^M \sum_{t \in |N_T^c|} U_{\ti}) \leqslant \sqrt{32 M \rho^2 (1 - \rho)^2 |N_T^c| \log (\frac{d}{\delta})}.
$$
Combining the above result with Eq.~\eqref{V_1_1}, we obtain 

\begin{align}
\lambda_{\min} (\bar{V}_{\ti}) &\leqslant (\lambda + M T) + 2 M \rho (\rho - 1) |N_T^c|\nonumber\\
&+ \sqrt{32 M \rho^2 (1 - \rho)^2 |N_T^c| \log (\frac{d}{\delta})}. \label{V_1_2}
\end{align}
This concludes the proof of Lemma~\ref{L5}, which demonstrates the upper bound for $\lambda_{\min} (\bar{V}_{\ti})$ based on the stated inequalities. 
\end{proof}
\begin{lemma} \label{L6}
For any, $a, b, c > 0$, if $a x + c \leqslant \sqrt{b x}$, then the following holds for $x \geqslant 0$
$$
\frac{b - 2 a c - \sqrt{b^2 - 4 a b c}}{2 a^2} \leqslant x \leqslant \frac{b - 2 a c + \sqrt{b^2 - 4 a b c}}{2 a^2}.
$$
\end{lemma}
\begin{proof}
Let $a, b, c > 0$, and consider the inequality $a x + c \leqslant \sqrt{b x}$. We can square both sides and rearrange them to obtain the quadratic inequality
$
a^2 x^2 - (b - 2 a c) x + c^2 \leqslant 0.
$
Since $a^2 > 0$, we can apply the solution formula for quadratic inequalities to express the solution for $x$
$
\frac{b - 2 a c - \sqrt{b^2 - 4 a b c}}{2 a^2} \leqslant x \leqslant \frac{b - 2 a c + \sqrt{b^2 - 4 a b c}}{2 a^2}.$
\end{proof}

%We present the proof of Theorem~\ref{prop1} in Appendix~\ref{app_1}.
To determine the upper bound of the cumulative regret in Theorem~\ref{prop1}, we determine the upper bounds of $|N_T|$ and $|N_T^c|$. Given that $|N_T^c| = T - |N_T|$, we compute the upper and lower bounds for $|N_T^c|$. This is given in Theorem~\ref{L7}.
\begin{theorem} \label{L7}
In the DiSC-UCB algorithm, Algorithm~\ref{alg:TV}, the upper bound and lower bound of $|N_T^c|$ are given by
\begin{align*}
\scalemath{0.95}{|N_T^c|} &\scalemath{0.95}{ \geqslant \frac{4}{M} \log(\frac{d}{\delta}) - \frac{(\frac{2 \beta_{0, i}}{\alpha \rbtaui})^2 - (\lambda + M T)}{2 M \rho (1 - \rho)} }\\
&\scalemath{0.95}{- \sqrt{\frac{16}{M^2} \log^2(\frac{d}{\delta}) - \frac{4 ((\frac{2 \beta_{0, i}}{\alpha \rbtaui})^2 - (\lambda + M T)) \log(\frac{d}{\delta})}{M^2 \rho (1 - \rho)}}} 
\end{align*}
\begin{align*}
\scalemath{0.95}{|N_T^c| }&\scalemath{0.95}{\leqslant \sqrt{(\frac{4 \sqrt{6}}{M} \log(\frac{d}{\delta}))^2 - \frac{16 ((\frac{2 \beta_{0, i}}{\alpha \rbtaui})^2 - (\lambda + M T)) \log(\frac{d}{\delta})}{M^2 \rho (1 - \rho)}}},
\end{align*}
where $\beta_{0, i} = \sigma \sqrt{2 \log \frac{1}{\delta}} + \lambda^{\frac{1}{2}}$ and $T \leqslant \frac{1}{M} [(\frac{2 \beta_{\taui}}{\alpha \rbtaui})^2 - \lambda]$.
\end{theorem}
\begin{proof}
Recall that $\tau$ is the last round in which Algorithm~\ref{alg:TV} plays the agent's action.  Given any round $t$ that the agent's action is played, we have $\lambda_{\min} (\bar{V}_{\ti}) \geqslant (\frac{2 \beta_{\ti}}{\alpha \rbti})^2$. By using $\rbti \geqslant r_l$, it follows that $\lambda_{\min} (\bar{V}_{\ti}) \geqslant (\frac{2 \beta_{\ti}}{\alpha r_l})^2$. From Eq.~\eqref{V_1_2}, we derive
\begin{align*}
&(\frac{2 \beta_{\taui}}{\alpha \rbtaui})^2 \leqslant (\lambda + M \tau) + 2 M \rho (\rho - 1) |N_{\tau}^c|\\
&+ \sqrt{32 M \rho^2 (1 - \rho)^2 |N_{\tau}^c| \log (\frac{d}{\delta})}
\end{align*}
By setting the gradient of $2 \rho (\rho - 1)$ to zero, $2 \rho (\rho - 1) \in [0, \frac{1}{2}]$. Moreover, it is obvious that $|N_T^c| - |N_{\tau}^c| = T - \tau$. Consequently, we can obtain the inequalities
% \begin{align}
% (\frac{2 \beta_{\taui}}{\alpha \rbtaui})^2 &\leqslant (\lambda + M T) + 2 M \rho (\rho - 1) |N_T^c| \nonumber \\
% &+ \sqrt{32 M \rho^2 (1 - \rho)^2 |N_T^c| \log (\frac{d}{\delta})} \nonumber
% \end{align}
% \begin{align}
% \sqrt{32 M \rho^2 (1 - \rho)^2 |N_T^c| \log (\frac{d}{\delta})} &\geqslant 2 M \rho (1 - \rho) |N_T^c| \nonumber \\
% &+ (\frac{2 \beta_{\taui}}{\alpha \rbtaui})^2 - (\lambda + M T) \label{V_1_3}.
% \end{align}
$(\frac{2 \beta_{\taui}}{\alpha \rbtaui})^2 \leqslant (\lambda + M T) + 2 M \rho (\rho - 1) |N_T^c|+ \sqrt{32 M \rho^2 (1 - \rho)^2 |N_T^c| \log (\frac{d}{\delta})} 2 M \rho (1 - \rho) |N_T^c| + (\frac{2 \beta_{\taui}}{\alpha \rbtaui})^2 - (\lambda + M T)$ which is
\begin{align}
&\leqslant \sqrt{32 M \rho^2 (1 - \rho)^2 |N_T^c| \log (\frac{d}{\delta})} \label{V_1_3}.
\end{align}
By using the substitution $|N_T^c| = T - |N_T|$, we have the following equivalent inequality
\begin{align}
(\frac{2 \beta_{\taui}}{\alpha \rbtaui})^2 &\leqslant \lambda + M T (2 \rho^2 - 2 \rho + 1) + 2 M \rho (1 - \rho) |N_T|\nonumber \\
&+ \sqrt{32 M \rho^2 (1 - \rho)^2 (T - |N_T|) \log (\frac{d}{\delta})}. \nonumber
\end{align}
By rearranging the terms of these inequalities we get, 
\begin{align}
&\sqrt{32 M \rho^2 (1 - \rho)^2 (T - |N_T|) \log (\frac{d}{\delta})}\nonumber \
\end{align}
\begin{align}
&\geqslant - 2 M \rho (1 - \rho) |N_T| + (\frac{2 \beta_{\taui}}{\alpha \rbtaui})^2 - (\lambda + M T (2 \rho^2 - 2 \rho + 1)). \label{V_2_3}
\end{align}
% \begin{equation}
%   \label{V_2_3}
%   - 2 M \rho (1 - \rho) |N_T| + (\frac{2 \beta_{\ti}}{\alpha \rbti})^2 - (\lambda + M T (2 \rho^2 - 2 \rho + 1)) \leqslant \sqrt{32 M \rho^2 (1 - \rho)^2 (T - |N_T|) \log (\frac{d}{\delta})}.
% \end{equation}
If $|N_T^c| = 0$, it implies that the agent's action will be played in every iteration and the baseline constraint is not active. In this case, cumulative regret follows with the findings in Proposition~\ref{ECC} (Theorem~4.1, \cite{lin2023distributed1}), making this a trivial case. Hence, the focus of this paper is when $|N_T^c| \neq 0$. Recognizing that Eq.~\eqref{V_2_3} is not always true for every round, we must have $(\frac{2 \beta_{\taui}}{\alpha \rbtaui})^2 - (\lambda + M T (2 \rho^2 - 2 \rho + 1)) \geqslant 0$. Given $\frac{1}{2} \leqslant 2 \rho^2 - 2 \rho + 1 \leqslant 1$, it follows that $(\frac{2 \beta_{\taui}}{\alpha \rbtaui})^2 - (\lambda + M T) \geqslant 0$. 
% Now we focus on Eq.~\eqref{V_1_3}. By applying Lemma~\ref{L6}, we can obtain the upper bound and lower bound for $|N_T^c|$, represented by
% $$
% |N_T^c| \geqslant\frac{4}{M} \log (\frac{d}{\delta}) - \frac{(\frac{2 \beta_{\ti}}{\alpha \rbti})^2 - \lambda - M T}{2 M \rho (1 - \rho)} - \sqrt{\frac{16}{M^2} \log^2 (\frac{d}{\delta}) - \frac{(\frac{4 \beta_{\ti}}{\alpha \rbti})^2 - 4 \lambda - 4 M T}{M^2 \rho (1 - \rho)} \log (\frac{d}{\delta})}
% $$
% $$
% |N_T^c| \leqslant\frac{4}{M} \log (\frac{d}{\delta}) - \frac{(\frac{2 \beta_{\ti}}{\alpha \rbti})^2 - \lambda - M T}{2 M \rho (1 - \rho)} + \sqrt{\frac{16}{M^2} \log^2 (\frac{d}{\delta}) - \frac{(\frac{4 \beta_{\ti}}{\alpha \rbti})^2 - 4 \lambda - 4 M T}{M^2 \rho (1 - \rho)} \log (\frac{d}{\delta})}.
% $$
Recall Eq.~\eqref{V_1_3}. We know
\begin{align*}
a &:= 2 M \rho (1 - \rho) \geqslant 0 \mbox{~and~} b := 32 M \rho^2 (1 - \rho)^2 \log (\frac{d}{\delta}) \geqslant 0.
\end{align*}
From Theorem~\ref{L7}, we know
$
c := ((\frac{2 \beta_{\taui}}{\alpha \rbtaui})^2 - (\lambda + M T)) \geqslant 0.
$
Using this Eq.~\eqref{V_1_3} can be rewritten as $a |N_T^c| + c \leqslant \sqrt{b |N_T^c|}$, which gives 
$
a^2 |N_T^c|^2 - (b - 2 a c) |N_T^c| + c^2 \leqslant 0.
$
Recognizing the existence of a solution for this inequality and since $|N_T^c| \geqslant 0$, it implies that
$
b - 2 a c \geqslant 0.
$
By using Lemma~\ref{L6}, we get $\frac{b - 2 a c - \sqrt{b^2 - 4 a b c}}{2 a^2} \leqslant |N_T^c| \leqslant \frac{b - 2 a c + \sqrt{b^2 - 4 a b c}}{2 a^2}$. 

First, we analyze the upper bound. Given the fact that $b - 2 a c \geqslant 0$ and using the identity $\sqrt{x} + \sqrt{y} \leqslant \sqrt{2x + 2y}$ for any $x, y \geqslant 0$, it follows that $b - 2 a c + \sqrt{b^2 - 4 abc} = \sqrt{(b - 2ac)^2} + \sqrt{b^2 - 4abc} \leqslant \sqrt{4b^2 - 16abc + 8a^2c^2} \leqslant \sqrt{6b^2 - 16abc}$. Thus the upper bound can be simplified as 
$$
\frac{b - 2 a c + \sqrt{b^2 - 4 a b c}}{2 a^2} \leqslant \sqrt{\frac{3}{2} (\frac{b}{a^2})^2 - \frac{4bc}{a^3}}. 
$$
Subsequently, when considering the partial derivative with regard to $c$ for $\frac{b - 2 a c - \sqrt{b^2 - 4 a b c}}{2 a^2}$ and $\sqrt{\frac{3}{2} (\frac{b}{a^2})^2 - \frac{4bc}{a^3}}$, the gradient of the former is non-negative, while the gradient of the latter is non-positive. Given the fact that $\beta_{\ti}$ is an increasing sequence, we choose to replace $\beta_{\taui}$ with $\beta_{0, i}$ in both the upper and lower bounds. Based on the above analysis and the given definition of $\beta_{0, i} = \sigma \sqrt{2 \log \frac{1}{\delta}} + \lambda^{\frac{1}{2}}$, we obtain the lower and upper bounds for $|N_T^c|$ as given in Theorem~\ref{L7}.
\end{proof}
\begin{remark}
   We note that in the proof of  Theorem~\ref{L7} we get $T \leqslant \frac{1}{M} [(\frac{2 \beta_{\taui}}{\alpha \rbtaui})^2 - \lambda]$. This ensures that the lower bound is always smaller than the upper bound in Theorem~\ref{L7}. 
   %We have revised the statement of Theorem~1 by incorporating the condition $T \leqslant \frac{1}{M} [(\frac{2 \beta_{\taui}}{\alpha \rbtaui})^2 - \lambda]$.
\end{remark}
Combining Theorem~\ref{prop1} and Theorem~\ref{L7}, we get the following bound for the cumulative regret for Alg.~\ref{alg:TV} (DiSC-UCB). 
\begin{theorem} \label{T1}
The cumulative regret of DiSC-UCB algorithm, Algorithm~\ref{alg:TV}, with $\beta_{\ti} = \beta_{\ti}(\sqrt{1 + \sigma^2}, \delta / 2)$ is bounded at round $T$ with probability at least $1 - M \delta$ by 
\begin{align*}
\R_T &\leqslant 4 \beta_T \sqrt{d c^{\prime} \log (MT)} + 4\beta_T \sqrt{MTd \log MT} \log(MT) \nonumber\\
&+ \sqrt{2 c^{\prime} \log (\frac{2}{\delta})} + c^{\prime \prime} (\kappa_h + \rho r_h + \rho),
\end{align*}
where $c', c'' >0$ are given by $\scalemath{0.9}{c^{\prime} = O(MT)}$ and $\scalemath{0.9}{c^{\prime \prime} = O(\sqrt{MT})}$. 
Further, for $\delta = \frac{1}{M^2 T}$, Algorithm~\ref{alg:TV} achieves a regret of $O(d \sqrt{MT} \log^2 T)$ with $O(M^{1.5} d^3)$ communication cost. 
\end{theorem}
%\subsection{Proof of Theorem~\ref{T1}}\label{app_thm2}
\begin{comment}
We first present the complete version of Theorem~\ref{T1}.
\begin{theorem}[\bf Complete form of Theorem~\ref{T1}]\label{th:complete}
{\em The cumulative regret of Algorithm~\ref{alg:TV} with $\beta_{\ti} = \beta_{\ti}(\sqrt{1 + \sigma^2}, \delta / 2)$ is bounded at round $T$ with probability at least $1 - M \delta$ by 
\begin{align*}
\R (T) &\leqslant 4 \beta_T \sqrt{d c^{\prime} \log (MT)} + 4\beta_T \sqrt{MTd \log MT} \log(MT)\\
&+ \sqrt{2 c^{\prime} \log (\frac{2}{\delta})} + c^{\prime \prime} (\kappa_h + \rho r_h + \rho),
\end{align*}
where $c', c'' >0$ is given by
\begin{align*}
&{c^{\prime} = M T - 4 \log(\frac{d}{\delta}) + \frac{(\frac{2}{\alpha \rbtaui} (\sigma \sqrt{2 \log \frac{1}{\delta}} + \lambda^{\frac{1}{2}}))^2 - (\lambda + M T)}{2 \rho (1 - \rho)}}\\
&{+ \sqrt{16 \log^2(\frac{d}{\delta}) - \frac{4 ((\frac{2}{\alpha \rbtaui} (\sigma \sqrt{2 \log \frac{1}{\delta}} + \lambda^{\frac{1}{2}}))^2 - (\lambda + M T)) \log(\frac{d}{\delta})}{\rho (1 - \rho)}} }
\end{align*}
\begin{align*}
&{c^{\prime \prime} = \sqrt{(4 \sqrt{6} \log(\frac{d}{\delta}))^2 - \frac{16 ((\frac{2}{\alpha \rbtaui} (\sigma \sqrt{2 \log \frac{1}{\delta}} + \lambda^{\frac{1}{2}}))^2 - (\lambda + M T)) \log(\frac{d}{\delta})}{\rho (1 - \rho)}}}.
\end{align*}
Further, for $\delta = \frac{1}{M^2 T}$, Algorithm~\ref{alg:TV} achieves a regret of $O(d \sqrt{MT} \log^2 T)$ with $O(M^{1.5} d^3)$ communication cost. }
\end{theorem}
\end{comment}
\begin{proof}
From Eq.~\ref{R_2}, we know cumulative regret is bounded 
\begin{align*}
\R_T &\leqslant 4 \beta_T \sqrt{M|N_T|d \log (MT)} + 4\beta_T \sqrt{MTd \log MT} \log(MT)\\
&+ \sqrt{2 M |N_T| \log \frac{2}{\delta}} + M |N_T^c| (\kappa_h + \rho \rbti + \rho).
\end{align*}
From Theorem~\ref{L7}, we get the upper and lower bound for $|N_T^c|$.
\begin{comment}
, denoted by
\begin{align*}
|N_T^c| &\geqslant \frac{4}{M} \log(\frac{d}{\delta}) - \frac{(\frac{2}{\alpha \rbtaui} (\sigma \sqrt{2 \log \frac{1}{\delta}} + \lambda^{\frac{1}{2}}))^2 - (\lambda + M T)}{2 M \rho (1 - \rho)}\\
&- \sqrt{\frac{16}{M^2} \log^2(\frac{d}{\delta}) - \frac{4 ((\frac{2}{\alpha \rbtaui} (\sigma \sqrt{2 \log \frac{1}{\delta}} + \lambda^{\frac{1}{2}}))^2 - (\lambda + M T)) \log(\frac{d}{\delta})}{M^2 \rho (1 - \rho)}} \\
|N_T^c| &\leqslant \sqrt{(\frac{4 \sqrt{6}}{M} \log(\frac{d}{\delta}))^2 - \frac{16 ((\frac{2}{\alpha \rbtaui} (\sigma \sqrt{2 \log \frac{1}{\delta}} + \lambda^{\frac{1}{2}}))^2 - (\lambda + M T)) \log(\frac{d}{\delta})}{M^2 \rho (1 - \rho)}}. 
\end{align*}
\end{comment}
Therefore, the cumulative regret is bound by
\begin{align*}
\R_T &\leqslant 4 \beta_T \sqrt{d c^{\prime} \log (MT)} + 4\beta_T \sqrt{MTd \log MT} \log(MT)\\
&+ \sqrt{2 c^{\prime} \log (\frac{2}{\delta})} + c^{\prime \prime} (\kappa_h + \rho r_h + \rho),
\end{align*}
where $c^{\prime}$ and $c^{\prime \prime}$ are as given in \footnote{The constants $c', c'' >0$ is given by
\begin{align*}
&{c^{\prime} = M T - 4 \log(\frac{d}{\delta}) + \frac{(\frac{2}{\alpha \rbtaui} (\sigma \sqrt{2 \log \frac{1}{\delta}} + \lambda^{\frac{1}{2}}))^2 - (\lambda + M T)}{2 \rho (1 - \rho)}}\\
&{+ \sqrt{16 \log^2(\frac{d}{\delta}) - \frac{4 ((\frac{2}{\alpha \rbtaui} (\sigma \sqrt{2 \log \frac{1}{\delta}} + \lambda^{\frac{1}{2}}))^2 - (\lambda + M T)) \log(\frac{d}{\delta})}{\rho (1 - \rho)}} }
\end{align*}
\begin{align*}
&{c^{\prime \prime} = \sqrt{(4 \sqrt{6} \log(\frac{d}{\delta}))^2 - \frac{16 ((\frac{2}{\alpha \rbtaui} (\sigma \sqrt{2 \log \frac{1}{\delta}} + \lambda^{\frac{1}{2}}))^2 - (\lambda + M T)) \log(\frac{d}{\delta})}{\rho (1 - \rho)}}}.
\end{align*}}.
\begin{comment}
\begin{align*}
c^{\prime} &= M T - 4 \log(\frac{d}{\delta}) + \frac{(\frac{2}{\alpha \rbtaui} (\sigma \sqrt{2 \log \frac{1}{\delta}} + \lambda^{\frac{1}{2}}))^2 - (\lambda + M T)}{2 \rho (1 - \rho)}\\
&+ \sqrt{16 \log^2(\frac{d}{\delta}) - \frac{4 ((\frac{2}{\alpha \rbtaui} (\sigma \sqrt{2 \log \frac{1}{\delta}} + \lambda^{\frac{1}{2}}))^2 - (\lambda + M T)) \log(\frac{d}{\delta})}{\rho (1 - \rho)}} \\
c^{\prime \prime} &= \sqrt{(4 \sqrt{6} \log(\frac{d}{\delta}))^2 - \frac{16 ((\frac{2}{\alpha \rbtaui} (\sigma \sqrt{2 \log \frac{1}{\delta}} + \lambda^{\frac{1}{2}}))^2 - (\lambda + M T)) \log(\frac{d}{\delta})}{\rho (1 - \rho)}}.
\end{align*}
\end{comment}
%
Furthermore, given that $c^{\prime} = O(MT)$, $c^{\prime \prime} = O(\sqrt{MT})$, and $\beta_T = O(\sqrt{d \log{\frac{T}{\delta}}}) = O(\sqrt{d \log{MT}})$, for $\delta = \frac{1}{M^2 T}$, the cumulative regret $\R_T$ can be bounded as
\begin{align*}
\R_T &\leqslant 4 \beta_T \sqrt{d c^{\prime} \log (MT)} + 4\beta_T \sqrt{MTd \log MT} \log(MT)\\
&+ \sqrt{2 c^{\prime} \log (\frac{2}{\delta})} + c^{\prime \prime} (\kappa_h + \rho r_h + \rho) \\
&= O(d \sqrt{MT} \log{MT}) + O(d \sqrt{MT} \log^2{MT})\\
&+ O(\sqrt{MT \log{M^2 T}}) + O(\sqrt{MT}) \\
&= O(d \sqrt{MT} \log^2{MT}) = \widetilde{O}(d \sqrt{MT}),
\end{align*}
where the last step is followed by the condition $T > M$. 

{\bf Communication:} The communication cost for our algorithm follows the approach in \cite{wang2019distributed}. 
\begin{comment}
We present it here for completeness. Recall $ R = O(d \log (MT)) $. Let $ \alpha = \sqrt { \frac{BT}{R} } $. This implies that there could be at most $ \lceil T / \alpha\rceil = \lceil \sqrt { \frac{TR}{B} } \rceil$ epochs that contains more than $ \alpha $ time steps. Consider epochs $p$ and $p+1$.  If the $ p$-th epoch contains less than $ \alpha $ time steps, 
\begin{align*}
\log (\frac{ \det (V_{p+1})}{\det (V_{p})}) (t_{p+1} - t_{p}) & > B\\
\log (\frac{\det (V_{p+1})}{\det (V_{p})}) > \frac{B}{t_{p+1} - t_{p}} &> \frac{B}{\alpha}.
\end{align*}
Further, 
\[
\sum_{p=0}^{P-1} \log ( \frac {\det(V_{p+1})} {\det(V_{p})} ) = \log \frac{\det(V_{P})}{\det(V_{0})} \leqslant R
\]
As a result, there could be at most $ \left\lceil\frac{R}{B / \alpha}\right\rceil = \left\lceil\frac{R \alpha}{B}\right\rceil = \lceil \sqrt { \frac{TR}{B} } \rceil $ epochs with less than $ \alpha $ time steps. Therefore, the total number of epochs is at most
\[
\left\lceil\frac{T}{\alpha}\right\rceil + \left\lceil\frac{R \alpha}{B}\right\rceil = \left\lceil \sqrt { \frac{TR}{B} } \right\rceil + \left\lceil \sqrt { \frac{TR}{B} } \right\rceil = O (\sqrt{\frac{TR}{B}} ) 
\]
By selecting $ B = (\frac{T \log MT}{d M}) $ and $ R = O(d \log M T)) $, the right-hand-side is $ O(M^{0.5} d) $. The agents communicate only at the end of each epoch when each agent sends $ O(d^{2}) $ numbers to the server and then receives $ O(d^{2}) $ numbers from the server. Therefore, in each epoch, the communication cost is $ O(M d^{2}) $. Hence, total communication cost is $ O(M^{1.5} d^{3}) $.
\end{comment}
\end{proof}
\begin{remark}
A naive adaptation of solving the $M$ tasks separately would result in  $\widetilde{O}(dM\sqrt{T})$ regret. In contrast, our proposed method achieves a sublinear regret of  $\widetilde{O}(d\sqrt{MT})$ which validates the effectiveness of the proposed approach.
\end{remark}
 
\section{Unknown Baseline Reward}\label{sec:unknown}
This section considers the unknown baseline reward setting.
Such a setting was first studied in \cite{kazerouni2017conservative} for conservative CB with cumulative performance constraint and in \cite{moradipari2020stage} for CB with stage-wise constraints. We extend the model in \cite{moradipari2020stage} to the distributed setting with unknown contexts. We assume the agents know the lower bound $r_\ell$ in Assumption~\ref{assume:bound}. We describe the modifications to the DiSC-UCB algorithm to handle the unknown baseline (DiSC-UCB-UB). 
Then, we prove that the regret and communication bounds for DiSC-UCB-UB are in the same order as DiSC-UCB.
The approach is primarily based on the observation that $\ths$ lies in $\B_{\ti}$  with high probability. Hence we use the upper bound to replace the value of $\rbti$
$$
\max_{v \in \B_{\ti}} \psixbic^{\top} v \geqslant \psixbic^{\top} \ths = \bE [\pixbic^{\top} \ths] = \rbti.
$$
Hence, the safety constraint can be formulated as
$$
\min_{v \in \B_{\ti}} \pixc^{\top} v \geqslant (1 - \alpha) \max_{v \in \B_{\ti}} \psixbic^{\top} v.
$$

Given the agent's lack of knowledge regarding the reward provided by the baseline policy and its only knowledge of the lower bound $r_l$ we can construct the pruned action set as
\begin{align}
\Z_{\ti} &= \Big\{x_{\ti} \in \A: \psixc^{\top} \hth_{\ti} \geqslant \hspace{-1 mm}\frac{\beta_{\ti}}{\sqrt{\lambda_{\min} (\bar{V}_{\ti})}} \hspace{-0.5 mm}\nonumber\\
&+\hspace{-0.5 mm} (1 - \alpha) \max_{v \in \B_{\ti}} \psixbic^{\top} v \Big\}. \label{UBR_1}
\end{align}
It can be observed that when the condition
\begin{equation}
\lambda_{\min} (\bar{V}_{\ti}) \geqslant (\frac{2 (2 - \alpha) \beta_{\ti}}{\alpha r_l})^2 \label{UBR_2}
\end{equation}
is satisfied, the optimal action $\xs$ is contained within the pruned action set $\Z_{\ti}$ with high probability. The details of these two derivations are presented below. The necessary modifications to the DiSC-UCB algorithm can be made by updating lines~\ref{pas} and~\ref{cons} with Eqs.~\eqref{UBR_1} and \eqref{UBR_2}, respectively. 
%The modified DiSC-UCB-UB algorithm is presented in Alg.~\ref{alg:TV_1}. \\

\noindent{\bf Construction of the Pruned Action Set $\Z_{\ti}$:}
We analyze two cases. 1)~$\pixc^{\top} \hth_{\ti} \geqslant \psixc^{\top} \hth_{\ti}$ and 2)~$\pixc^{\top} \hth_{\ti} \leqslant \psixc^{\top} \hth_{\ti}$. Let us first address case 1), and explain the process of constructing a subset of actions that satisfy the constraint for all $v \in \B_{\ti}$. Define 
$\Z_{\ti}^1 $
\begin{align}
&\scalemath{0.9}{:= \{x_{\ti} \in \A: \min_{v \in \B_{\ti}} \pixc^{\top} v\geqslant (1 - \alpha) \max_{v \in \B_{\ti}} \psixbic^{\top} v\}}\label{eq:prun5}\\
&\Leftarrow \scalemath{0.9}{\Big\{x_{\ti} \in \A: \min_{v \in \B_{\ti}} \pixc^{\top} (v - \hth_{\ti}) + \psixc^{\top} \hth_{\ti}}\nonumber\\
&\scalemath{0.9}{+ (\pixc - \psixc)^{\top} \hth_{\ti}\geqslant (1 - \alpha) \max_{v \in \B_{\ti}} \psixbic^{\top} v\Big\}} \nonumber \\
&\Leftarrow \Big\{x_{\ti} \in \A: \psixc^{\top} \hth_{\ti}\geqslant  \frac{\beta_{\ti}}{\sqrt{\lambda_{\min} (\bar{V}_{\ti})}}\nonumber\\
&+ (1 - \alpha) \max_{v \in \B_{\ti}} \psixbic^{\top} v\Big\} \label{eq:prun6}
\end{align}
where the last step follows from $\pixc^{\top} \hth_{\ti} \geqslant \psixc^{\top} \hth_{\ti}$ and $\pixc^{\top} (v - \hth_{\ti}) \geqslant - \frac{\beta_{\ti}}{\sqrt{\lambda_{\min} (\bar{V}_{\ti})}}$  from Lemma~\ref{L1}. 
All actions that meet the conditions in  Eq.~\eqref{eq:prun6} also fulfill the requirements of Eq.~\eqref{eq:prun5}, thus ensuring safety. 
Now we consider case~2), where $\pixc^{\top} \hth_{\ti} \leqslant \psixc^{\top} \hth_{\ti}$. In this case, our approach is to first identify actions that violate the baseline constraint, $\bar{\Z}_{\ti}^2$, and then eliminate those actions from the action set $\A$.  
\begin{align}
&\scalemath{0.94}{\bar{\Z}_{\ti}^2 := \Big\{x_{\ti} \in \A: \min_{v \in \B_{\ti}} \pixc^{\top} v\leqslant (1 - \alpha) \max_{v \in \B_{\ti}} \psixbic^{\top} v\Big\}}\label{eq:prun7}\\
&\Leftarrow \Big\{x_{\ti} \in \A: \min_{v \in \B_{\ti}} \pixc^{\top} (v - \hth_{\ti}) + \pixc^{\top} \hth_{\ti} \nonumber\\
&\leqslant (1 - \alpha) \max_{v \in \B_{\ti}} \psixbic^{\top} v\Big\} \nonumber \\
&\Leftarrow \Big\{x_{\ti} \in \A: \psixc^{\top} \hth_{\ti} \leqslant (1 - \alpha) \max_{v \in \B_{\ti}} \psixbic^{\top} v\Big\}\label{eq:prun8} 
\end{align}
where the last step follows from $\pixc^{\top} \hth_{\ti} \leqslant \psixc^{\top} \hth_{\ti}$.
Note that all actions that meet the conditions in  Eq.~\eqref{eq:prun8} also fulfill the requirements of Eq.~\eqref{eq:prun7}, consequently rendering them unsafe. 
By taking the difference between $\A$ and $\bar{\Z}_{\ti}^2$, we determine $\Z_{\ti}^2 = \A \setminus \bar{\Z}_{\ti}^2 $
\begin{align*}
%\Z_{\ti}^2 &= \A \setminus \bar{\Z}_{\ti}^2 \\
&= \Big\{x_{\ti} \in \A: \psixc^{\top} \hth_{\ti} \geqslant (1 - \alpha) \max_{v \in \B_{\ti}} \psixbic^{\top} v \Big\}.
\end{align*}
% \[\X_{\ti}^2 = \A \setminus \bar{\X}_{\ti}^2 = \{x_{\ti} \in \A: \psixc^{\top} \hth_{\ti} \geqslant - \frac{\beta_{\ti}}{\sqrt{\lambda_{\min} (\bar{V}_{\ti})}} + (1 - \alpha) \rbti \}.\] 
%
Given $\Z_{\ti}^1$ and $\Z_{\ti}^2$, we obtain the pruned action set by taking the intersection between $\Z_{\ti}^1$ and $\Z_{\ti}^2$, given by 
\begin{align*}
\Z_{\ti} &= \Big\{x_{\ti} \in \A: \psixc^{\top} \hth_{\ti} \geqslant \hspace{-1 mm}\frac{\beta_{\ti}}{\sqrt{\lambda_{\min} (\bar{V}_{\ti})}} \hspace{-0.5 mm}\\
&+\hspace{-0.5 mm} (1 - \alpha) \max_{v \in \B_{\ti}} \psixbic^{\top} v \Big\}.
\end{align*}
%\begin{algorithm}[t]
%\caption{Distributed Stage-wise CBs with Context Distribution and Unknown Baseline Reward (DiSC-UCB-UB)}\label{alg:TV_1}
% Replace line~\ref{pas} and line~\ref{cons} of Algorithm~\ref{alg:TV} as below
%8: Compute pruned action set $\Z_{\ti}$ using $\psiixc, \hth_{\ti}$ 
%13: {\bf If} {$F = 1$ and $\lambda_{\min} (\bar{V}_{\ti}) \geqslant (\frac{2 (2 - \alpha) \beta_{\ti}}{\alpha r_l})^2$} {\bf then}
%\end{algorithm}

 We describe the modifications to the DiSC-UCB algorithm to handle the unknown baseline case. We refer to the modified algorithm for the unknown case as DiSC-UCB-UB. 
 DiSC-UCB-UB differs from DiSC-UCB only in two lines, mainly in the pruned action set construction.
 In the modified algorithm, DiSC-UCB-UB, we replace line~\ref{pas} of Algorithm~\ref{alg:TV} with the new pruned action set $\Z_{\ti}$ and line~\ref{cons} as {\bf If} {$F = 1$ and $\lambda_{\min} (\bar{V}_{\ti}) \geqslant (\frac{2 (2 - \alpha) \beta_{\ti}}{\alpha r_l})^2$}.
Then, we prove that the regret and the communication bounds for DiSC-UCB-UB are in the same order as those of DiSC-UCB.

Next, we present the key result that bounds the cumulative regret and communication cost for DiSC-UCB-UB.
\begin{theorem} \label{Thm1_1}
The cumulative regret of DiSC-UCB-UB (Unknown baseline setting) with $\beta_{\ti} = \beta_{\ti}(\sqrt{1 + \sigma^2}, \delta / 2)$ is bounded at round $T$ with probability at least $1 - M \delta$ by 
\begin{align*}
\R_T &\leqslant 4 \beta_T \sqrt{d \bar{c}^{\prime} \log (MT)} + 4\beta_T \sqrt{MTd \log MT} \log(MT)\\
&+ \sqrt{2 \bar{c}^{\prime} \log (\frac{2}{\delta})} + \bar{c}^{\prime \prime} (2 \rho + 1 - r_l),
\end{align*}
where $\bar{c}', \bar{c}'' >0$ are given by $\scalemath{0.9}{\bar{c}^{\prime} = O(MT)}$ and $\scalemath{0.9}{\bar{c}^{\prime \prime} = O(\sqrt{MT})}$. Further, for $\delta = \frac{1}{M^2 T}$, DiSC-UCB-UB achieves a regret of $O(d \sqrt{MT} \log^2 T)$ with $O(M^{1.5} d^3)$ communication cost.
\end{theorem}
\begin{comment}
%Next, we present the key result that bounds the number of rounds conservative action and agents' actions are played.
\begin{theorem} \label{unknown_thm}
In Algorithm~\ref{alg:TV_1}, the upper bound and lower bound of $|N_T^c|$ are given by
\begin{align*}
\scalemath{0.95}{|N_T^c|} &\scalemath{0.95}{\geqslant \frac{4}{M} \log(\frac{d}{\delta}) - \frac{(\frac{2 (2 - \alpha) \beta_{0, i}}{\kappa_l + \alpha \rbtaui})^2 - (\lambda + M T)}{2 M \rho (1 - \rho)}} \\
&\scalemath{0.95}{- \sqrt{\frac{16}{M^2} \log^2(\frac{d}{\delta}) - \frac{4 ((\frac{2 (2 - \alpha) \beta_{0, i}}{\kappa_l + \alpha \rbtaui})^2 - (\lambda + M T)) \log(\frac{d}{\delta})}{M^2 \rho (1 - \rho)}}} \\
\scalemath{0.95}{|N_T^c|} &\scalemath{0.95}{\leqslant \sqrt{(\frac{4 \sqrt{6}}{M} \log(\frac{d}{\delta}))^2 - \frac{16 ((\frac{2 (2 - \alpha) \beta_{0, i}}{\kappa_l + \alpha \rbtaui})^2 - (\lambda + M T)) \log(\frac{d}{\delta})}{M^2 \rho (1 - \rho)}}}. 
\end{align*}
\end{theorem}
\end{comment}
The proofs of DiSC-UCB-UB results follow a similar approach to DiSC-UCB.
We present all the proofs for DiSC-UCB-UB in Appendix~\ref{app_unknown} (supplementary material). 
\begin{comment}
% \begin{figure*}[t!]
%\vspace*{-0.8 in}
 \subcaptionbox{ \footnotesize Small network:  $L=2, T=600$\label{fig:c1}}{\includegraphics[width=1.8in]{Figures/compare1.pdf}\vspace{-0.5 in}}\hspace{-0.9 em}%
 \subcaptionbox{\footnotesize Small network:  $L=20, T=3000$\label{fig:c3}}{\includegraphics[width=1.8 in]{Figures/compare6b.pdf}\vspace{-0.5 in}}\hspace{-0.9 em}%
\subcaptionbox{\footnotesize Large network:  $L=100, T=3000$\label{fig:c5}}{\includegraphics[width=1.8 in]{Figures/compare4b.pdf}\vspace{-0.5 in}}\hspace{-0.9 em}%
 \subcaptionbox{\footnotesize Large network $L=100, T=3000$\label{fig:c6}}{\includegraphics[width=1.8 in]{Figures/compare5b.pdf}\vspace{-0.5 in}}
 %\vspace{-1 in}}
\hspace{-1.4em}%
\end{comment}
\section{Numerical Experiments}\label{sec:sim}
In this section, we validate the performance of our  DiSC-UCB algorithm on synthetic and real-world movielens and LastFM datasets and compare it with the SCLTS algorithm proposed in \cite{moradipari2020stage} and with the unconstrained distributed algorithm DisLinUCB in \cite{wang2019distributed},  DisLSB \cite{wang2019distributed, Jiabin_Shana_ACC}, and Fed-PE \cite{huang2021federated, lin2023federated}.
While SCLTS implements the Thompson algorithm for stage-wise conservative linear CB, DisLinUCB, DisLSB, and Fed-PE considers  
unconstrained distributed linear CB problem.
We note that all the baselines assume contexts are known and $M=1$, single agent.

\subsection {Datasets}
\begin{comment}
\noindent{\bf Synthetic data:} 
%In this subsection, we elaborate the data generation process for synthetic data.
We set the parameters $\lambda=1$, $d=2$, and the number of contexts $|\C|=100$. To compare with baselines, we set $M=1$. Features vectors $\phi$ are randomly generated from a standard normal distribution followed by  normalization to obtain a reward in $[0, 1]$.  We added noise to $\phi$ to construct $\psi$.
%Additionally, we ensure a significant gap between actions within the same context, with a gap of $0.25$ between actions with adjacent rewards.
%We considered noise with a mean of $0$ and a variance of $0.001$ to obtain $\psi$ from $\phi$, and the baseline action in each round is set as the $30^{\rm th}$ best action of that particular round. In experiments where  $M=\{3,5,10\}$, we set that the reward parameters lie in $\Theta=\{[1, 1], [1, 0],[0, 1]\}$.
%We transformed the multi-task problem with heterogeneous reward parameter to a distributed CB problem with common reward parameter $\ths=\ths_1=[1,1]$ and heterogeneous feature vectors $\psi_i$ ($\phi_i$s are unknown) for agent $i$ by setting the corresponding feature in $\psi$ to zero for agent $i$.
Experiments with $M=1$ are averaged over 100 independent trials, and experiments with $M>1$ are averaged over 25 trials.
%
\end{comment}
\begin{figure*}[t]
%\vspace{-3 mm}
\subcaptionbox{\footnotesize $M=1, \alpha=0.3$ \label{fig:1}}{\includegraphics[scale=0.15]{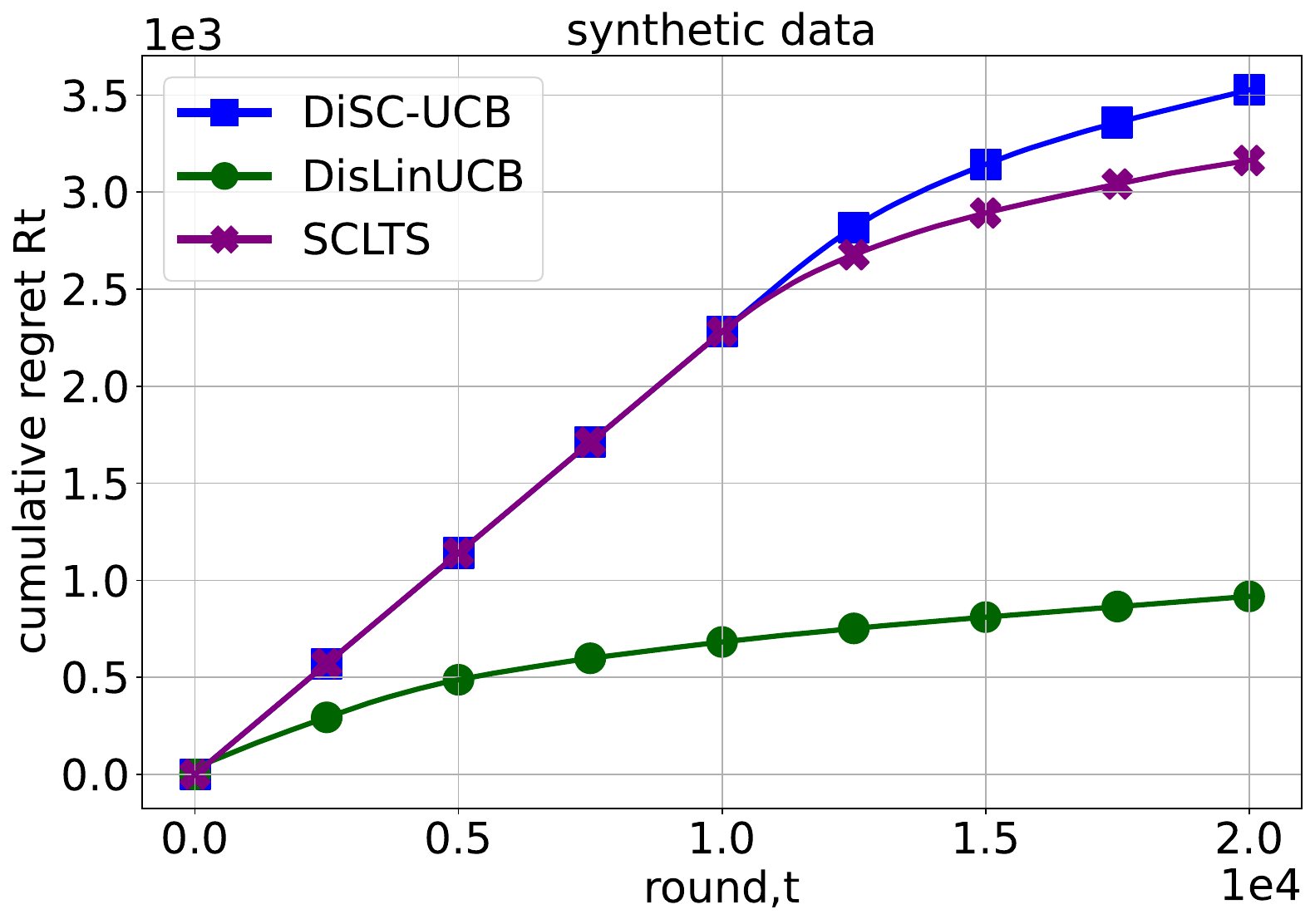}}\hspace{1.0 em}%
%\subcaptionbox{\footnotesize $M=1, \alpha=0.3$ \label{fig:1}}{\includegraphics[width=2.2 in]{plot_1.pdf}}\hspace{0.5 em}%
%width=2.2 in, height=1.5 in
\subcaptionbox{\footnotesize $M=1, \alpha=0.3$ \label{fig:2}}{\includegraphics[scale=0.15]{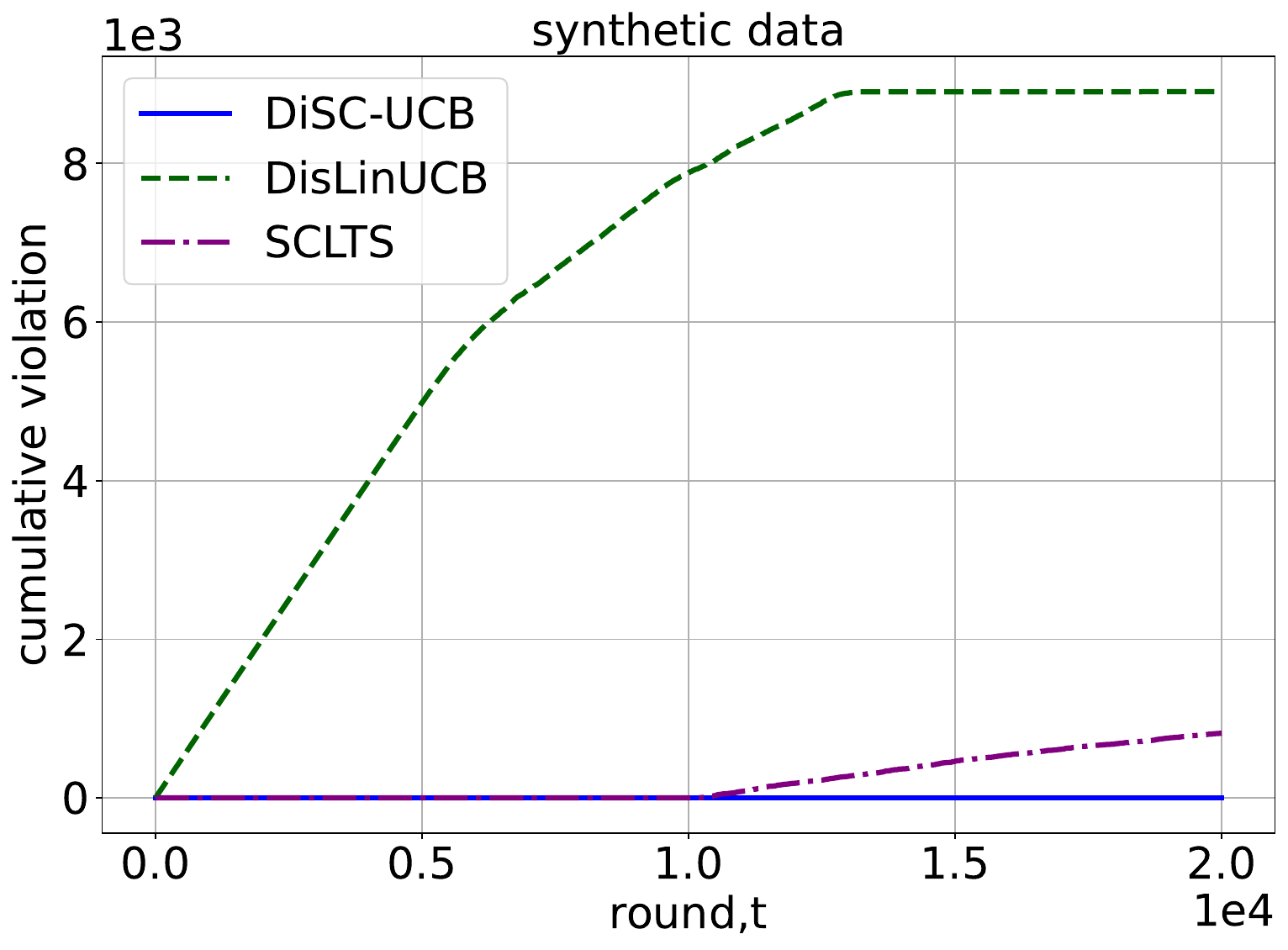}}\hspace{1.0 em}%
\vspace{-2 mm}
\subcaptionbox{\footnotesize $M=3, \alpha=0.25$ \label{fig:1b}}{\includegraphics[scale=0.15]{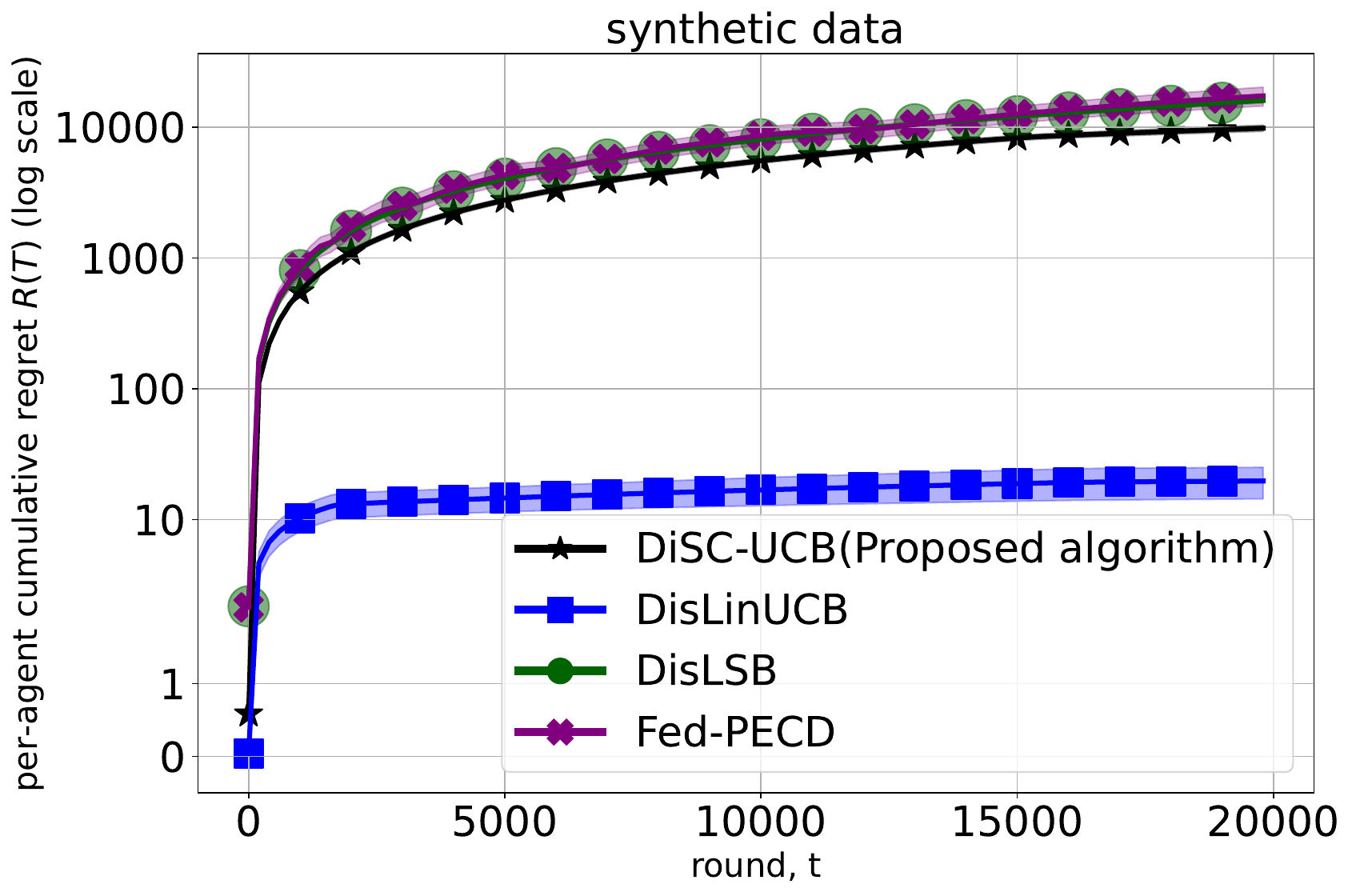}}\hspace{1.0 em}%
\vspace{-2 mm}
\subcaptionbox{\footnotesize $M=3, \alpha=0.25$ \label{fig:2b}}{\includegraphics[scale=0.15]{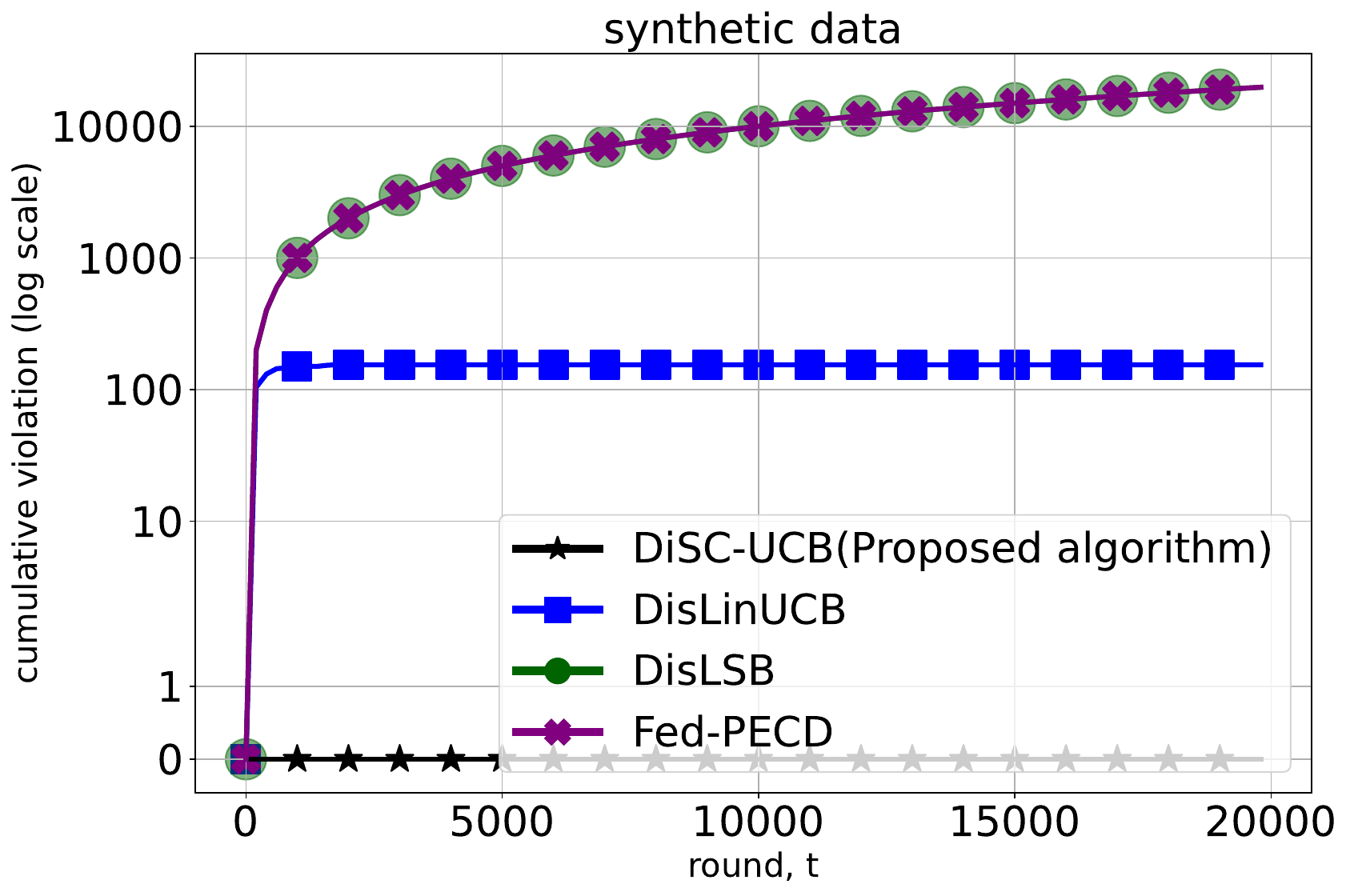}}
%\hspace{1.2 em}%
\vspace{2 mm}
\caption{\small Comparison of cumulative regret and cumulative violation of
 DiSC-UCB with SCLTS \cite{moradipari2020stage}, DisLinUCB \cite{wang2019distributed}, DisLSB \cite{Jiabin_Shana_ACC}, and Fed-PE \cite{huang2021federated} modified for unknown context Fed-PECD using \cite{lin2023federated}. Synthetic data: In Figs.~\ref{fig:1}, \ref{fig:2} we set the parameters as $\lambda=1$, $d=2$, $R=1$, $K=40$, $M=1$, $\ths = [0.9, 0.4]$, and noise variance $=2.5 \times 10^{-3}$, and the baseline action is set by the 10$^{\rm th}$ best action. In Figures~\ref{fig:1b}, \ref{fig:2b}, we set the parameters as: $\lambda=0.1$, $R=0.1$, $d=2$, number of contexts $|\C|=100$,  number of actions $K=10$, and number of agents $M=3$. We considered a noise with a mean of $0$ and a variance of $0.01$  to obtain $\psi$ from $\phi$. The true parameters are $\theta_1^\star= [0.9, 0.4]$, $\theta_2^\star= [0.9, 0]$, and $\theta_3^\star= [0, 0.4]$. The baseline is the 2nd best action, and $\alpha=0.25$. All plots were averaged over 100 independent trials.
}\label{fig:main1}
\end{figure*}

\begin{figure*}[t]
%\vspace{-3 mm}
\subcaptionbox{\footnotesize $M=1, \alpha=0.3$ \label{fig:7}}{\includegraphics[scale=0.21]{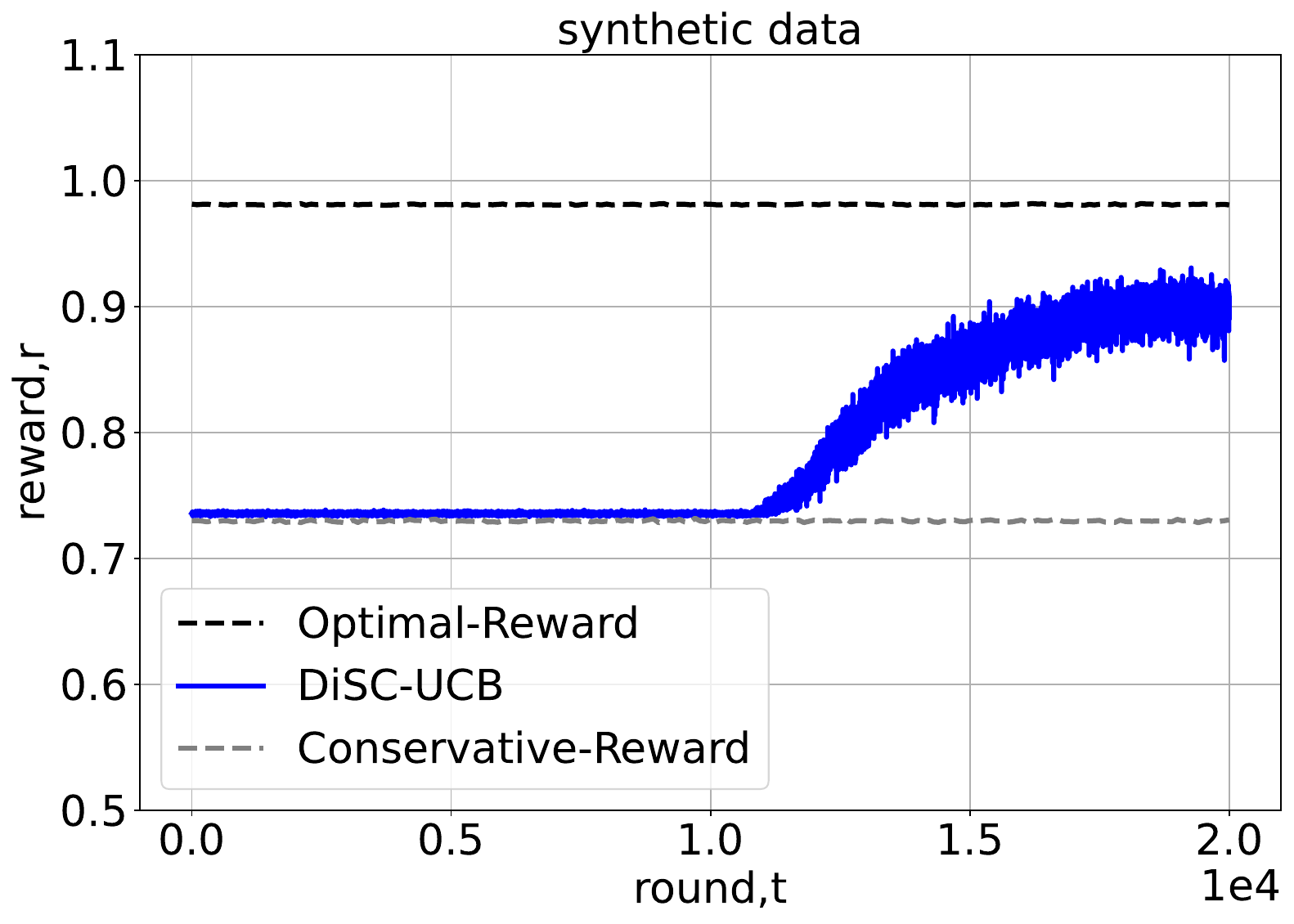}}\hspace{1.2 em}%
%\subcaptionbox{\footnotesize $M=1, \alpha=0.3$ \label{fig:1}}{\includegraphics[width=2.2 in]{plot_1.pdf}}\hspace{0.5 em}%
%width=2.2 in, height=1.5 in
\subcaptionbox{\footnotesize $M=1, \alpha=0.3$ \label{fig:8}}{\includegraphics[scale=0.21]{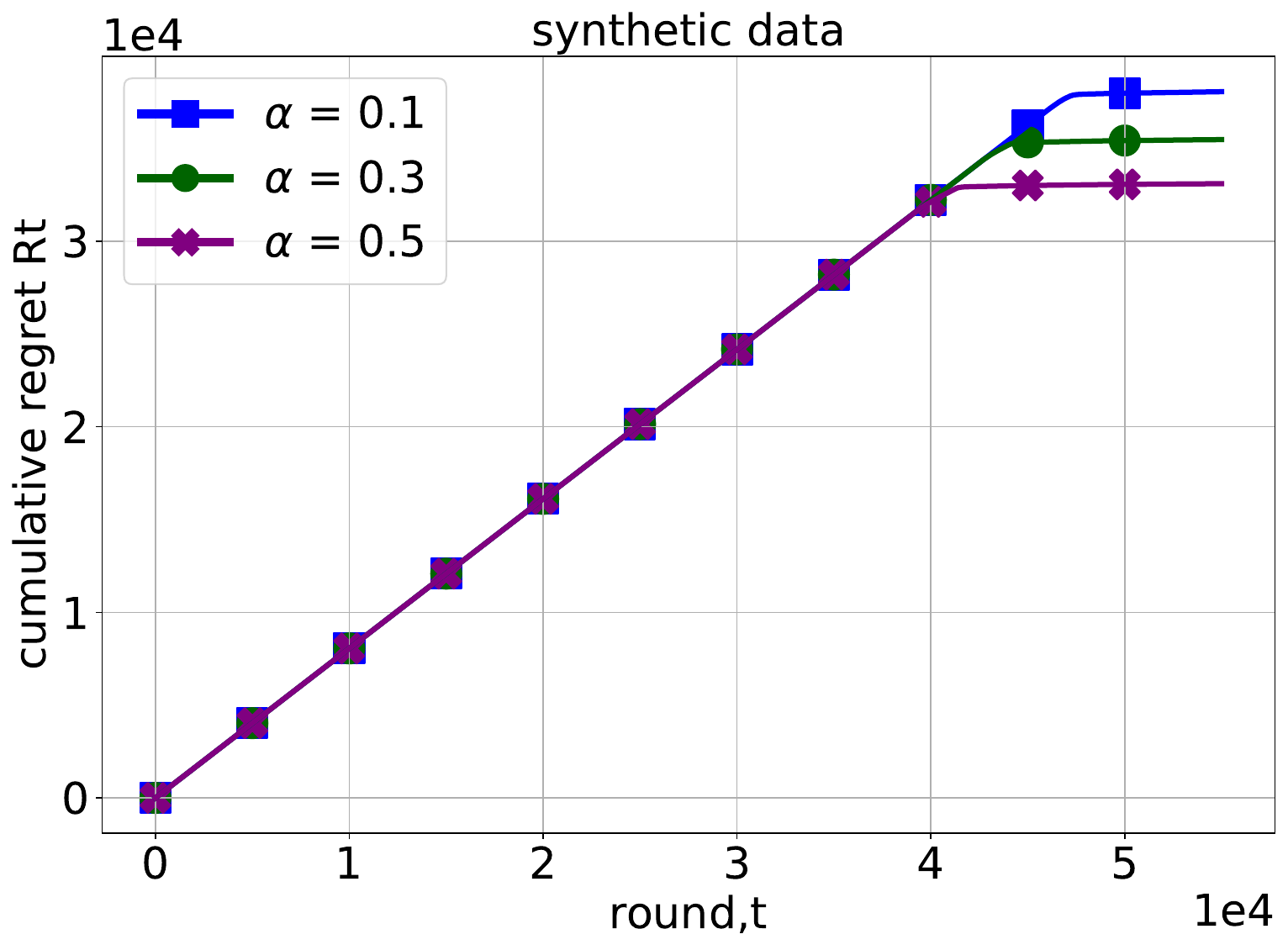}}\hspace{1.2 em}%
\subcaptionbox{\footnotesize $\alpha=0.3$ \label{fig:3}}{\includegraphics[scale=0.21]{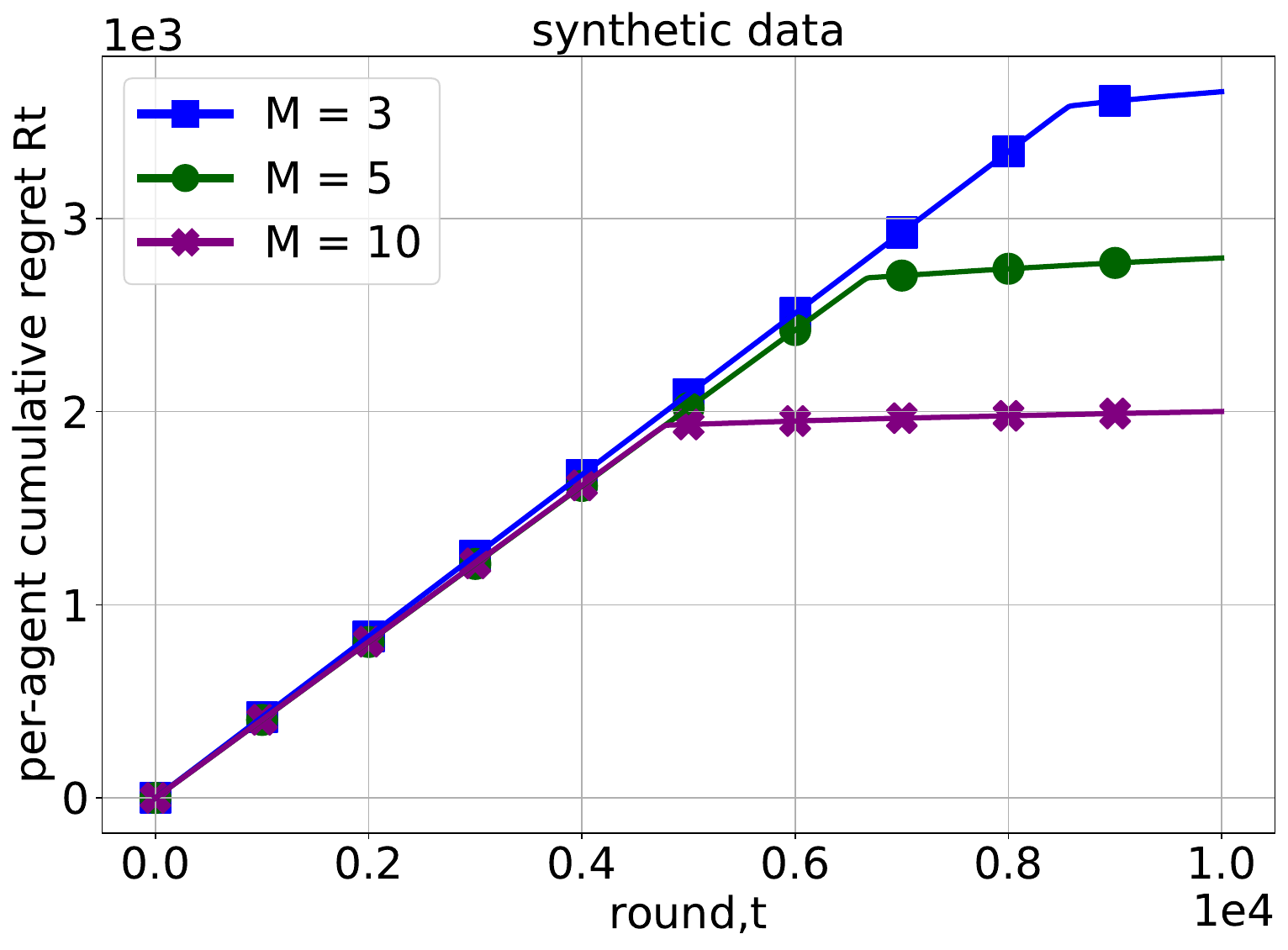}}\hspace{1.2 em}%
\subcaptionbox{\footnotesize $M=1, \alpha=0.5$ \label{fig:4}}{\includegraphics[scale=0.21]{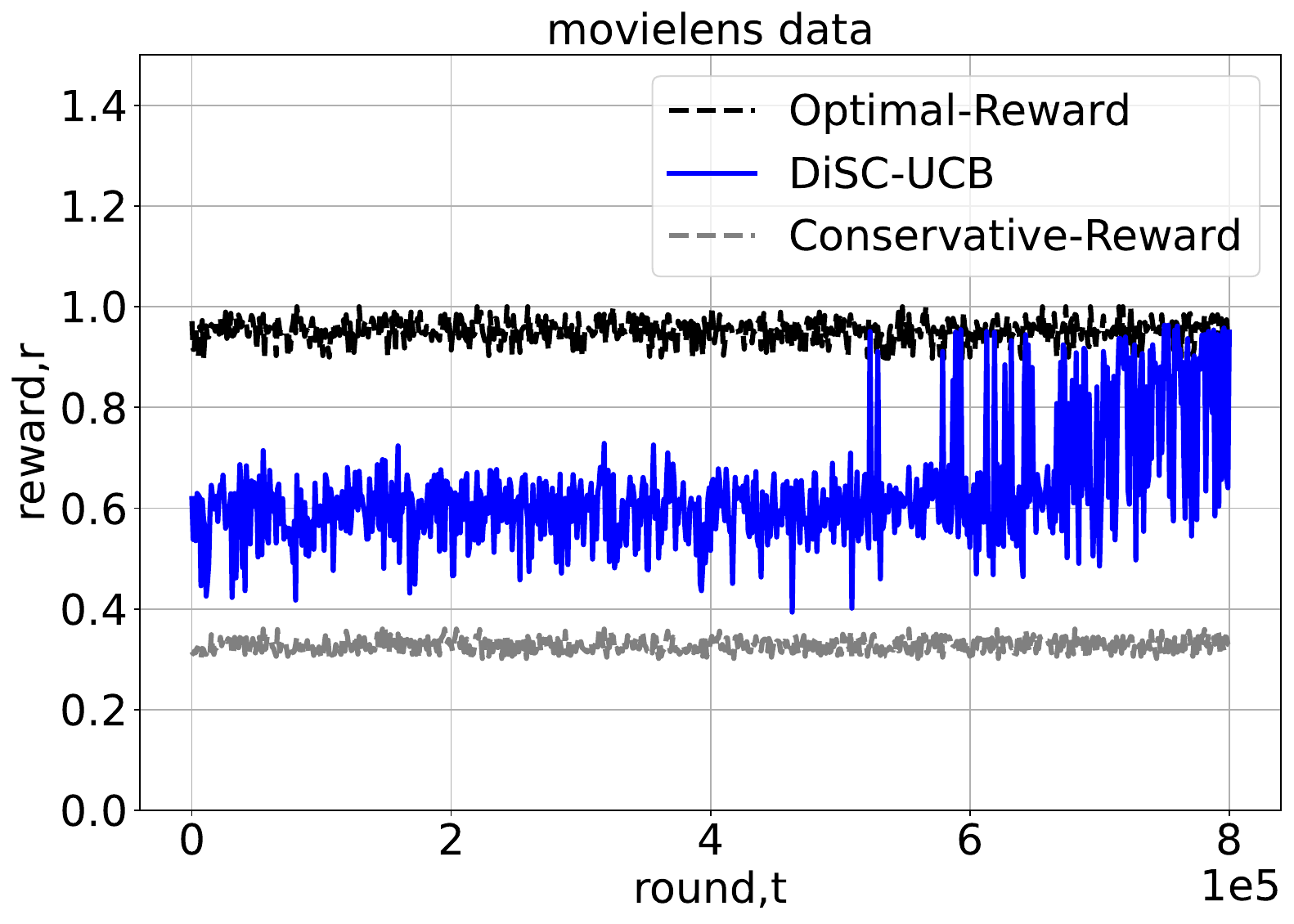}}\hspace{1.2 em}%
\subcaptionbox{\footnotesize $M=3$ \label{fig:6}}{\includegraphics[scale=0.21]{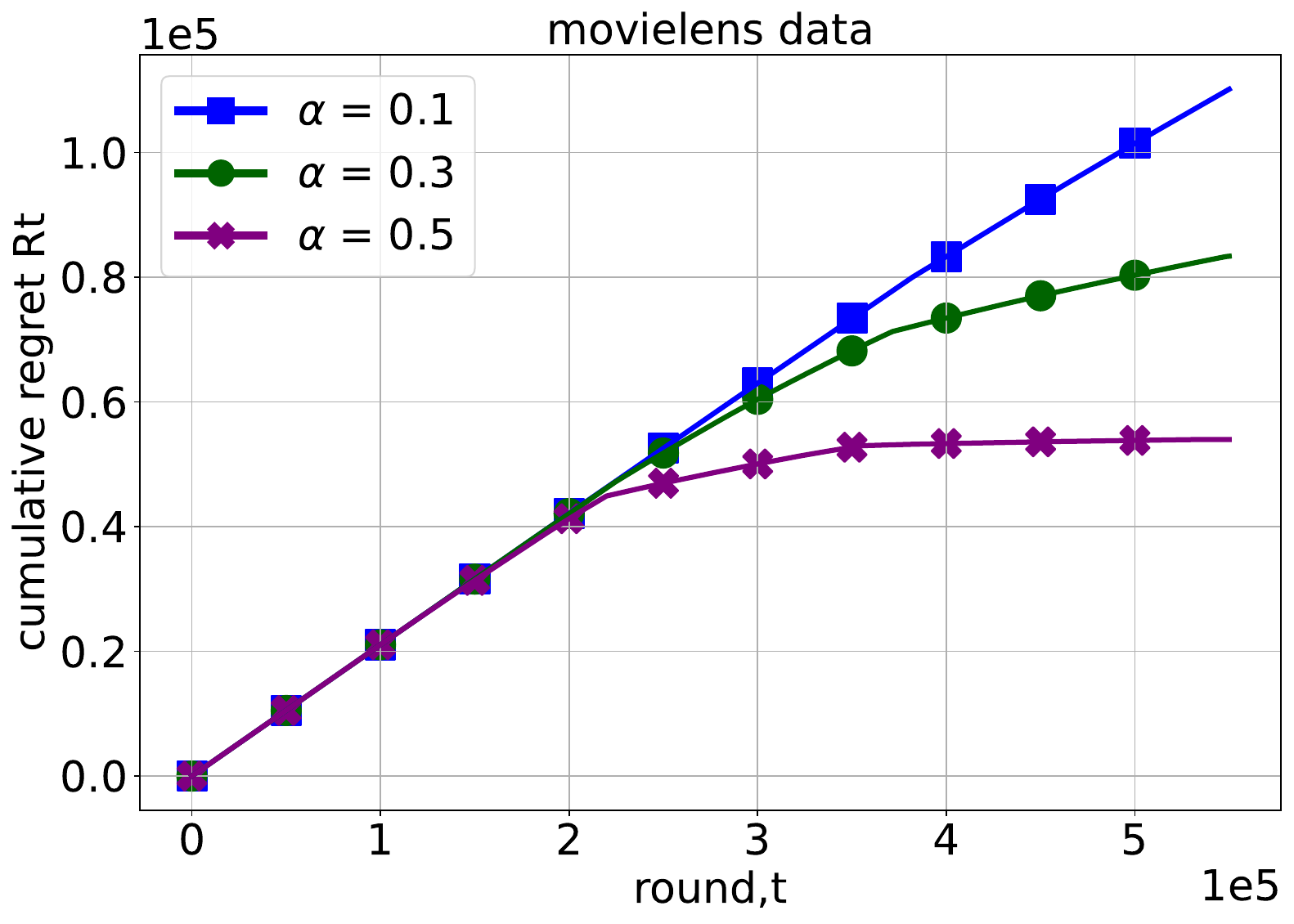}}\hspace{1.2 em}%
\subcaptionbox{\footnotesize $\alpha=0.5$ \label{fig:5}}{\includegraphics[scale=0.21]{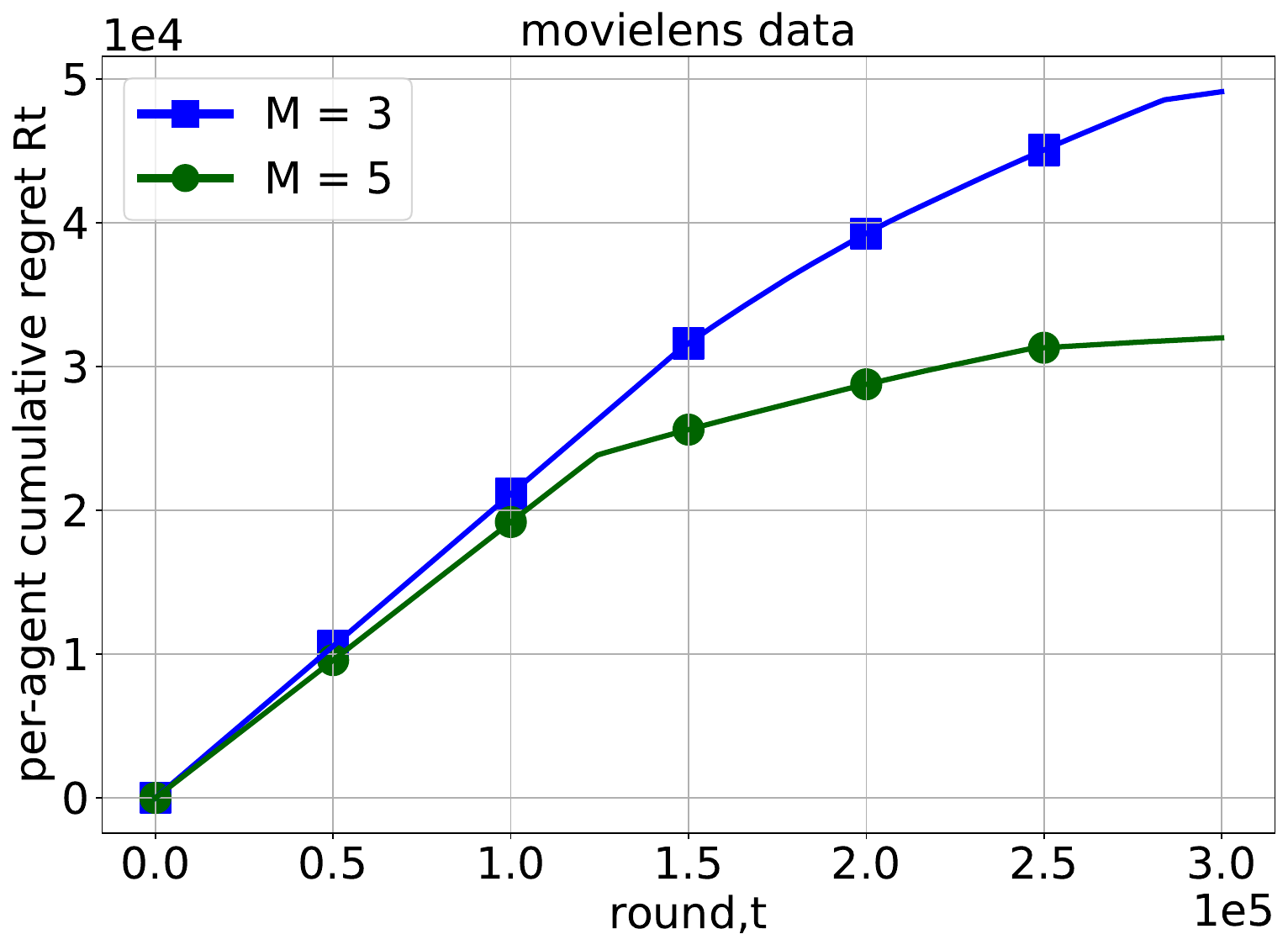}}\hspace{1.2 em}%
\subcaptionbox{\footnotesize $M=1, \alpha=0.7$ \label{fig:2_a}}{\includegraphics[scale=0.21]{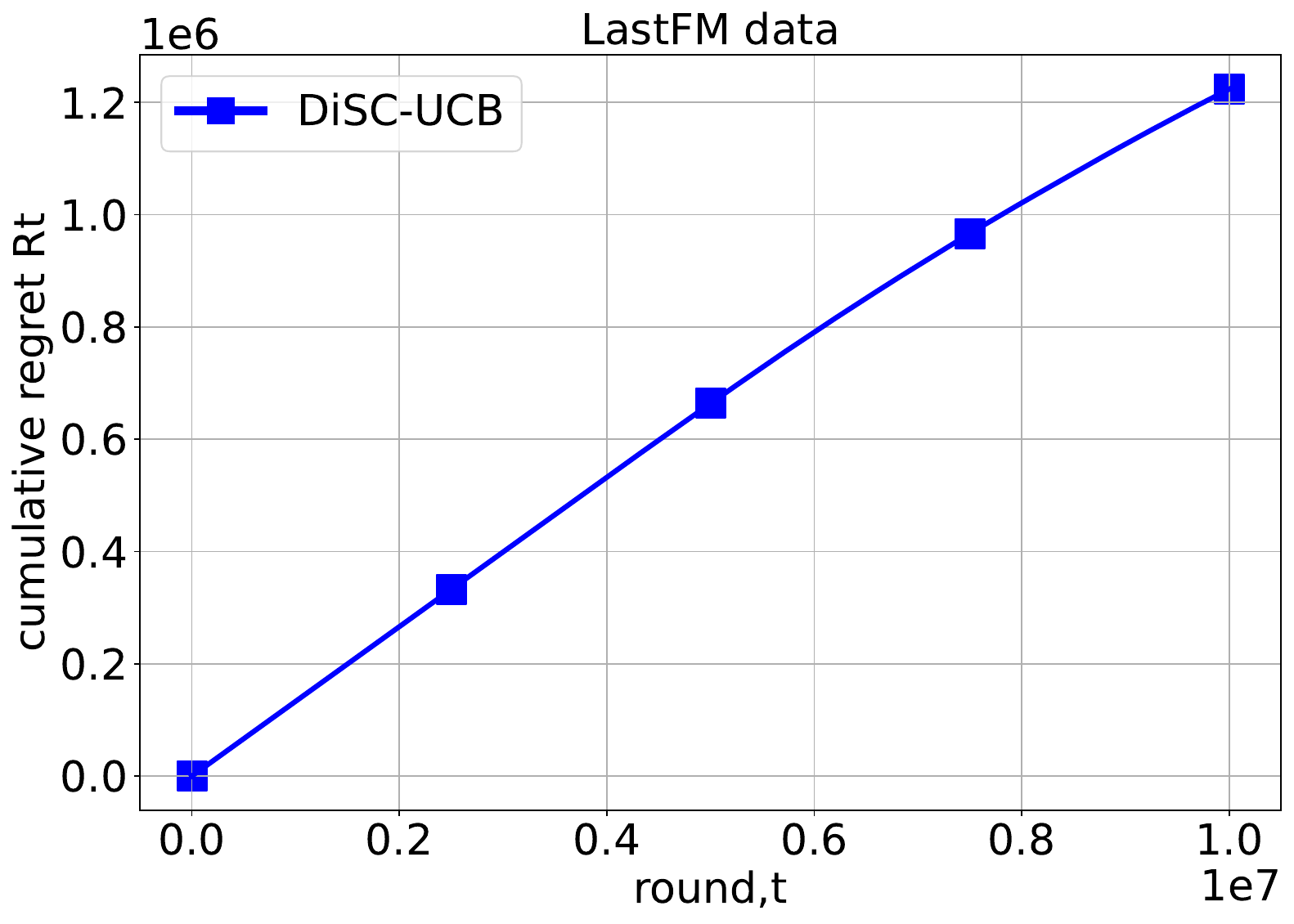}}\hspace{1.2 em}%
\subcaptionbox{\footnotesize $M=3$ \label{fig:2_b}}{\includegraphics[scale=0.21]{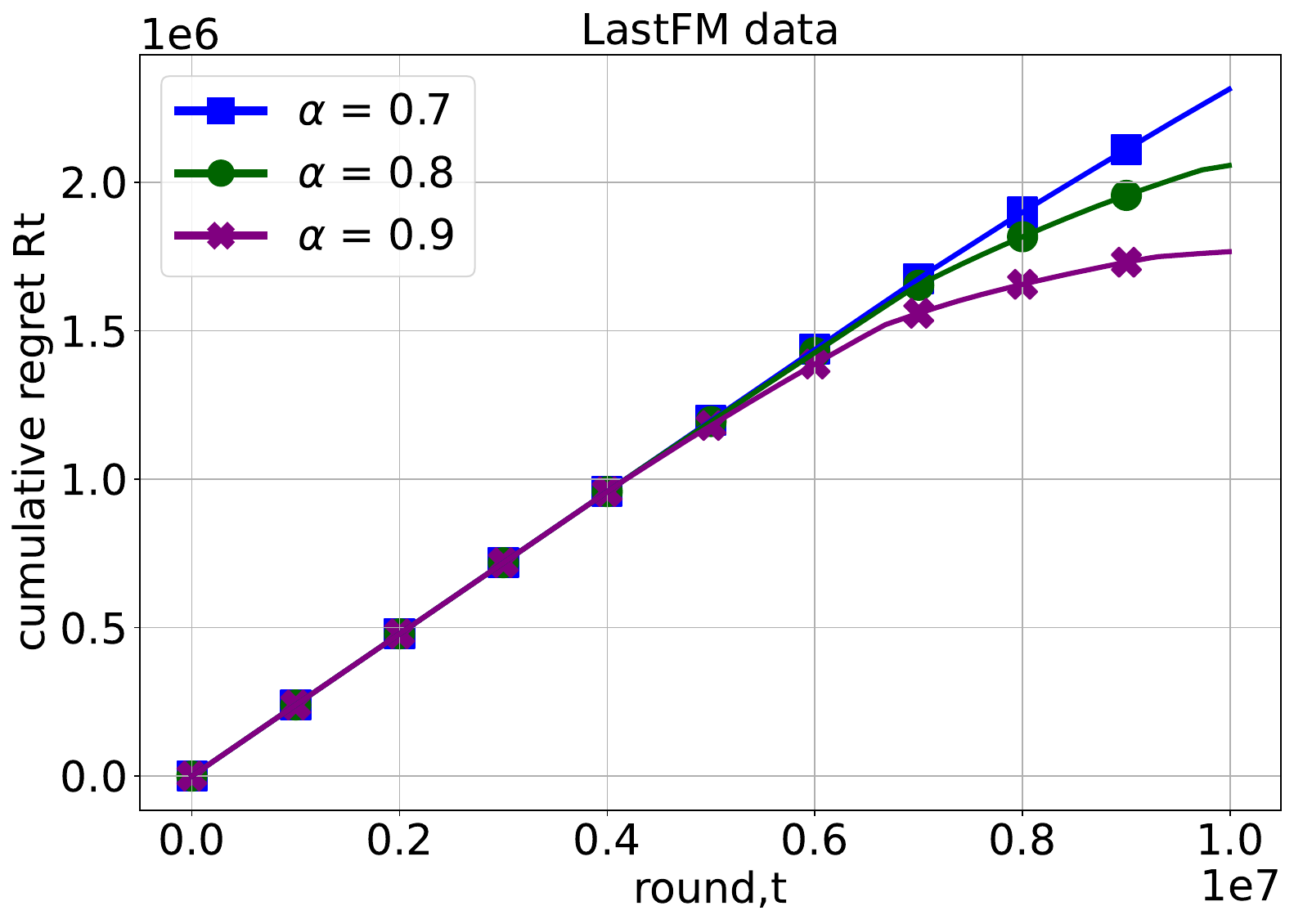}}\hspace{1.2 em}%
\subcaptionbox{\footnotesize $\alpha=0.7$ \label{fig:2_c}}{\includegraphics[scale=0.21]
{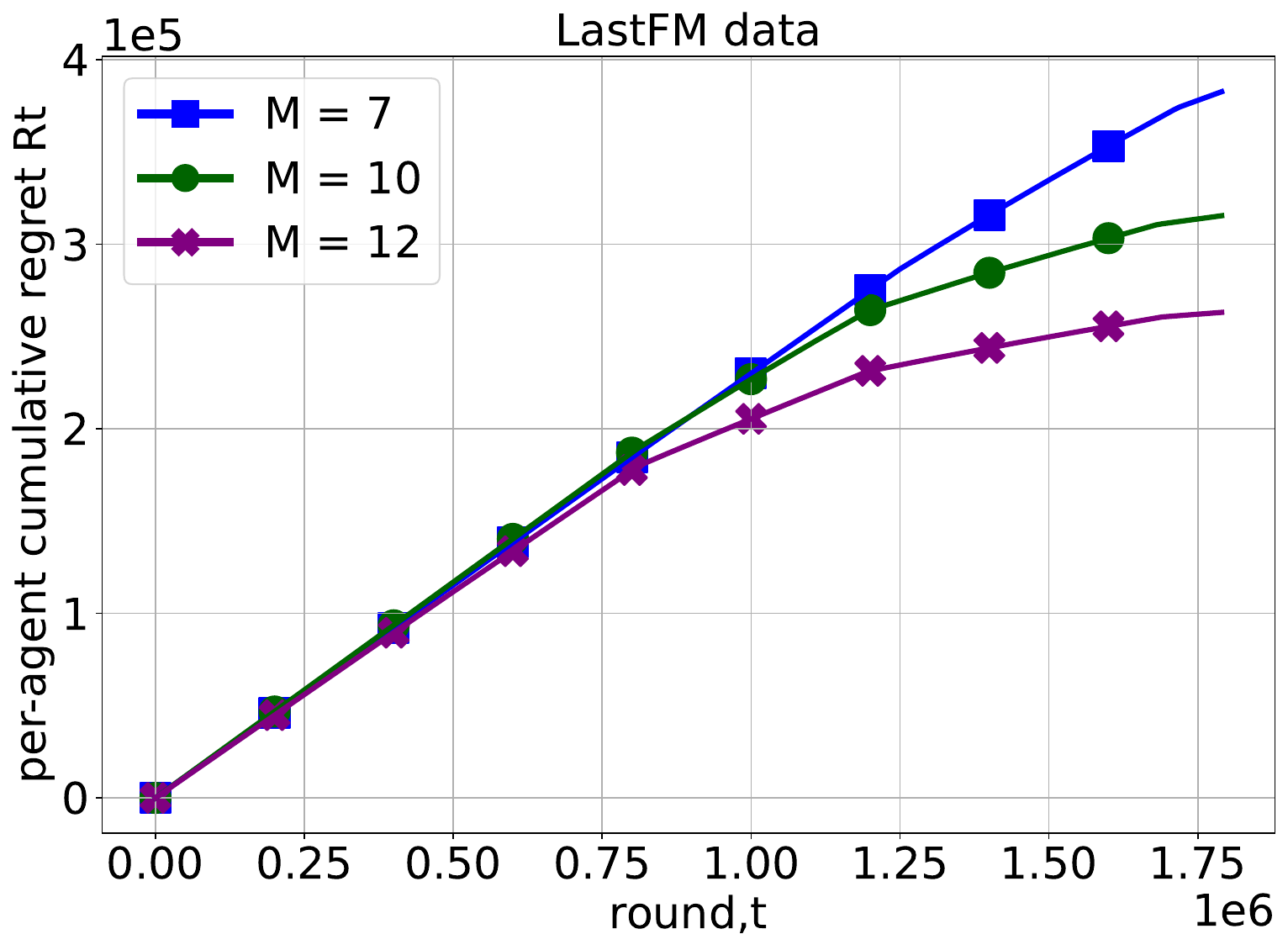}}
%\hspace{1.2 em}%
%\vspace{-2 mm}
\caption{\small 
 Synthetic data: In Figure~\ref{fig:7}, we set the parameters as $\lambda=1$, $d=2$, $R=1$, $K=40$, $M=1$, $\ths = [0.9, 0.4]$, and noise variance $=2.5 \times 10^{-3}$, and the baseline action is set by the 10$^{\rm th}$ best action . In  Figure~\ref{fig:8}, the parameters are set as $R=0.1$, $K=90$, $M=3$, $\ths = [1, 1]$, noise variance $=10^{-4}$, and the baseline action of a particular round is set as the 80$^{\rm th}$ best action of that round. The $\alpha$ values are varied as $\alpha = \{0.1, 0.3, 0.5\}$. In Fig.~\ref{fig:3}, $R=1$, $K=90$, $M=\{3,5,10\}$, $\ths = [1, 1]$, the reward parameters for the different tasks $\ths_i \in \Theta=\{[1, 1], [1, 0],[0, 1]\}$, noise variance $=10^{-2}$, and  baseline is the 30$^{\rm th}$ best action. Movielens data: In Figs.~\ref{fig:4}, \ref{fig:6}, and \ref{fig:5},  $R=0.1$, $K=50$,  noise variance$=10^{-2}$, $\ths = \frac{1}{\sqrt{3}}[1, 0, 0, 0, 1, 0, 0, 0, 1]$, and the baseline action is set as the 40$^{\rm th}$ best action. 
 LastFM data: In Figs.\ref{fig:2_a}, \ref{fig:2_b}, \ref{fig:2_c}, $R=\lambda=0.05$, $K=50$,  noise variance$=10^{-3}$, $\ths = \frac{1}{\sqrt{3}}[1, 0, 0, 0, 1, 0, 0, 0, 1]$, and the baseline action is set as the 5$^{\rm th}$ best action. 
 %In Figs.~\ref{fig:6}, \ref{fig:5}, the reward parameters for the different tasks $\ths_i \in \Theta=\{\frac{1}{\sqrt{3}}[1, 0, 0, 0, 1, 0, 0, 0, 1], \frac{1}{\sqrt{3}}[0, 0, 0, 0, 1, 0, 0, 0, 1], \frac{1}{\sqrt{3}}[1, 0, 0, 0, 0, 0, 0, 0, 1], \frac{1}{\sqrt{3}}[1, 0, 0, 0, 1, 0, 0, 0, 0]\}$, for $i=\{3,5\}$. 
%Figs.~\ref{fig:1}-\ref{fig:3}: synthetic data, and Figs.~\ref{fig:4}-\ref{fig:6}: movielens data.
}\label{fig:main2}
\end{figure*}

\noindent{\bf Movielens data:} 
%We used Movielens-100K data \cite{harper2015movielens} to evaluate the performance of our algorithm. The dataset consists of the rating matrix with entries between 0 and 5. We normalized the entries followed by a matrix factorization to extract the feature vectors for each user-movie pair. We chose 100 users and 50 movies randomly from the data,  $|\C| = 100$, $K = 50$, We elaborate on the process in the supplementary material, Appendix~\ref{app_4}.
We used movielens-100K data \cite{harper2015movielens} to evaluate the performance of our algorithm. 
 We first get the rating matrix $r_{x, c} \in \mathbb{R}^{943 \times 1682}$. The entries of the rating are between $0$ and $5$, which we normalized to be in $[0,1]$. 
We randomly choose a set of $50$ movies and a set of $100$ users for our analysis, i.e., $|\C|=100$ and $K=50$.
We then performed a non-negative matrix factorization of $r_{x, c} = WH$, where $W \in \mathbb{R}^{100 \times 3}$, $H \in \mathbb{R}^{3 \times 50}$. 
To construct the feature vector $\phi$ for the $g^{\rm th}$ user and $j^{\rm th}$ movie, we choose the $g^{\rm th}$ column of matrix $W$, $W_g \in \mathbb{R}^3$ and the $j^{\rm th}$ row of $H$, $H_j^\top \in \mathbb{R}^{1\times 3}$. We perform the outer product $W_g H_j^\top$ to obtain a $3 \times 3$ matrix. We vectorized this matrix to obtain $\phi \in \mathbb{R}^9$. We considered a noise with a mean of $0$ and a variance of $0.01$  to obtain $\psi$ from $\phi$. We set the number of agents as $M=3$, and the $\ths_1= (1/\sqrt{3})[1, 0,0, 0, 1,0,0,0,1]$, $\ths_3= (1/\sqrt{3})[1, 0,0, 0, 0,0,0,0,1]$, and $\ths_3= (1/\sqrt{3})[1, 0,0, 0, 1,0,0,0,0]$.
We transformed the multi-task problem with heterogeneous reward parameter $\Theta=\{\ths_1, \ths_2, \ths_3\}$ to a distributed CB problem with common reward parameter $\ths=\ths_1$ and heterogeneous feature vectors $\phi_i$ for agent $i$ by setting the respective feature in $\phi$ to zero.\\

{\noindent \bf LastFM data:}
The LastFM dataset is derived from the online music streaming service Last.fm \cite{fm} and includes data from 1892 users and 17632 artists. We consider each artist as an individual action. The reward is $1$ if the user has listened to an artist at least once, and $0$ otherwise. We keep only those users who have received at least $30$ interaction rewards. Consequently, we get a rating matrix $r_{x, c} \in \mathbb{R}^{741 \times 538}$. Following that, we performed non-negative matrix factorization on 
the rating matrix, denoted as $r_{x, c} = WH$, where $W \in \mathbb{R}^{741 \times 3}$ and $H \in \mathbb{R}^{3 \times 538}$. To construct the feature vector $\phi$ for the $g^{\rm th}$ user and $j^{\rm th}$ artist, we choose the $g^{\rm th}$ column of matrix $W$, $W_g \in \mathbb{R}^3$ and the 
$j^{\rm th}$ row of $H$, $H_j^\top \in \mathbb{R}^{1\times 3}$. We perform the outer product $W_g H_j^\top$ to obtain a $3 \times 3$ matrix. We vectorized this matrix to obtain $\phi \in \mathbb{R}^9$. 
We considered a noise with a mean of $0$ and a variance of $10^{-3}$  to obtain $\psi$ from $\phi$. The baseline action is set as the $5^{th}$ best action and $\lambda = R = 0.05$. 

\subsection{Comparison of DiSC-UCB with Existing  Constrained and Distributed Approaches}
In Figs.~\ref{fig:1},~\ref{fig:2}, we set $M=1$ and compare the performance of DiSC-UCB with SCLTS and DisLinUCB. SCLTS studied the constrained bandit problem with a single-agent and known contexts. To this end, we set $M=1$ in this comparison.
We report the cumulative expected regret and the cumulative number of constraint violations. SCLTS is the constrained approach developed in \cite{moradipari2020stage} for single-bandit problems with a known context. 
%We also vary the number of agents $M$ and obtain the cumulative expected regret.
%Additional details and more experiments are provided in the supplementary material, App.~\ref{app_4}.
%In Figure~\ref{fig:2}, we present the violations of the different algorithms.
Since the contexts are unknown SCLTS algorithm uses the noisy feature vectors for constructing the safe action set, which will lead to constraint violations as demonstrated in the toy example in Section~\ref{sec:rel} and in Fig.~\ref{fig:2}. DisLinUCB does not cater to constraints and, hence, will have more violations. On the other hand, SCLTS and DisLinUCB present smaller regrets given that they are loosely constrained and unconstrained, respectively, and hence perform more explorations in the initial stage. DiSC-UCB plays conservative actions in the initial rounds, resulting in a comparatively larger regret, however, zero violations as expected as shown in Fig.~\ref{fig:1} and Fig.~\ref{fig:2}.  In Figs.~\ref{fig:1b} and~\ref{fig:2b}, we compared our proposed DiSC-UCB algorithm with three other benchmark distributed/federated approaches. Figures show that the three benchmarks, which are unconstrained, result in large constraint violations, and hence, these approaches are unsuitable for the hard-constrained problem considered in this paper.

%{\cblue \subsection{Comparison of DiSC-UCB with Existing  Distributed Approaches}
%We set the parameters as follows: $\lambda=0.1$, $R=0.1$, $d=2$, number of contexts $|\C|=100$, and number of actions $K=10$. We also set $M=3$. We considered a noise with a mean of $0$ and a variance of $0.01$  to obtain $\psi$ from $\phi$. The true parameters are $\theta_1^\star= [0.9, 0.4]$, $\theta_2^\star= [0.9, 0]$, and $\theta_3^\star= [0, 0.4]$. The baseline is the 2nd best action, and $\alpha=0.25$. The experiment was averaged over 100 independent trials. }

\subsection{Regret versus System Parameters}
In addition to the results reported in the main paper, we plot the regret by varying $\alpha$ and the reward at every round of learning to ensure that the constraints are met. 
We present the plots showing the variation of the cumulative regret with respect to round $t$  for different values of $\alpha$ for the synthetic data in Figure~\ref{fig:8}.  We observe that as the value of $\alpha$ increases, the cumulative regret decreases, which is expected since a larger value of $\alpha$ implies less strict performance constraint. In Figure~\ref{fig:7}, we represent the reward plot. From the plot, we observe that the per-step reward is always larger than the baseline reward shown in the dotted line, which validates that our proposed algorithm DiSC-UCB satisfies the performance constraint at every stage of learning.  It also shows the improvement in the reward as the learning round progresses.
Figure~\ref{fig:3} demonstrates how agent collaboration improves the per-agent cumulative regret. 
For movielens data, Fig.~\ref{fig:4} shows the reward plot at each round verifying constraints are met at every round. Figs.~\ref{fig:5} and \ref{fig:6} provide the cumulative regret as $M$ and $\alpha$ are varied. For the LastFM data, Figs.~\ref{fig:2_a} presents the reward plot and  Figs.~\ref{fig:2_b} and~\ref{fig:2_c} rovide the cumulative regret as $M$ and $\alpha$ are varied.

\section{Conclusion}\label{sec:conclusion}
We studied the multi-task stochastic linear CB problem with stage-wise constraints when the agents observe only the context distribution and the exact contexts are unknown.  We proposed a UCB algorithm, referred to as DiSC-UCB.
%We showed that the regret of the algorithm decomposes into three terms.
%(i) an upper bound for the regret of the standard distributed linear UCB algorithm,
%(ii) a term that captures the loss since the contexts are unknown and 
%(iii) a  term that accounts for the loss for being conservative to satisfy the performance constraint. 
For $d$-dimensional linear bandits, we prove an  $O(d\sqrt{MT}\log^2 T)$ regret bound and an $O(M^{1.5}d^3)$ communication bound on the algorithm.
We extended to the setting where the baseline rewards are unknown and showed that the same bounds hold for regret and communication.
We empirically validated the performance of our algorithm on synthetic data and on real-world movielens-100K and LastFM data and compared with benchmarks, SCLTS and DisLinUCB. As part of the future work, we plan to investigate budget constraints where the goal is to minimize the regret before using the total budget.
%and other forms of noise modeling for the contexts.}

\appendices
\renewcommand\thetheorem{\Alph{section}.\arabic{theorem}}
\section{Toy Example}\label{sec:rel}
We present a toy example in Figure~\ref{Fig:toy}. 
Consider $M=3$ agents with  reward parameters $\theta_1^{\star} = [1, 2]^{\top}$, $\theta_2^{\star} = [1]$, and $\theta_3^{\star} = [2]$, and $\alpha \in [0, 1]$. The corresponding index set is $\I = \{\{1, 2\}, \{1\}, \{2\}\}$. That is, the number of features is $2$, and we consider $3$ tasks. While all three features are relevant to task~1, only the first feature is relevant to task~2, and only the second feature is relevant to task~3. For a context $c \in \C$, let each agent $i$ has three actions, denoted by $\A=\{x_i, x_i^{\prime}, x_i^{\prime \prime}\}$. 
\begin{figure}[t]
\vspace{-6 mm}
\begin{minipage}{\linewidth}
      \centering
      \begin{minipage}{0.6\linewidth}
          \begin{figure}[H]
              \includegraphics[width=1\linewidth, height=0.6\linewidth]{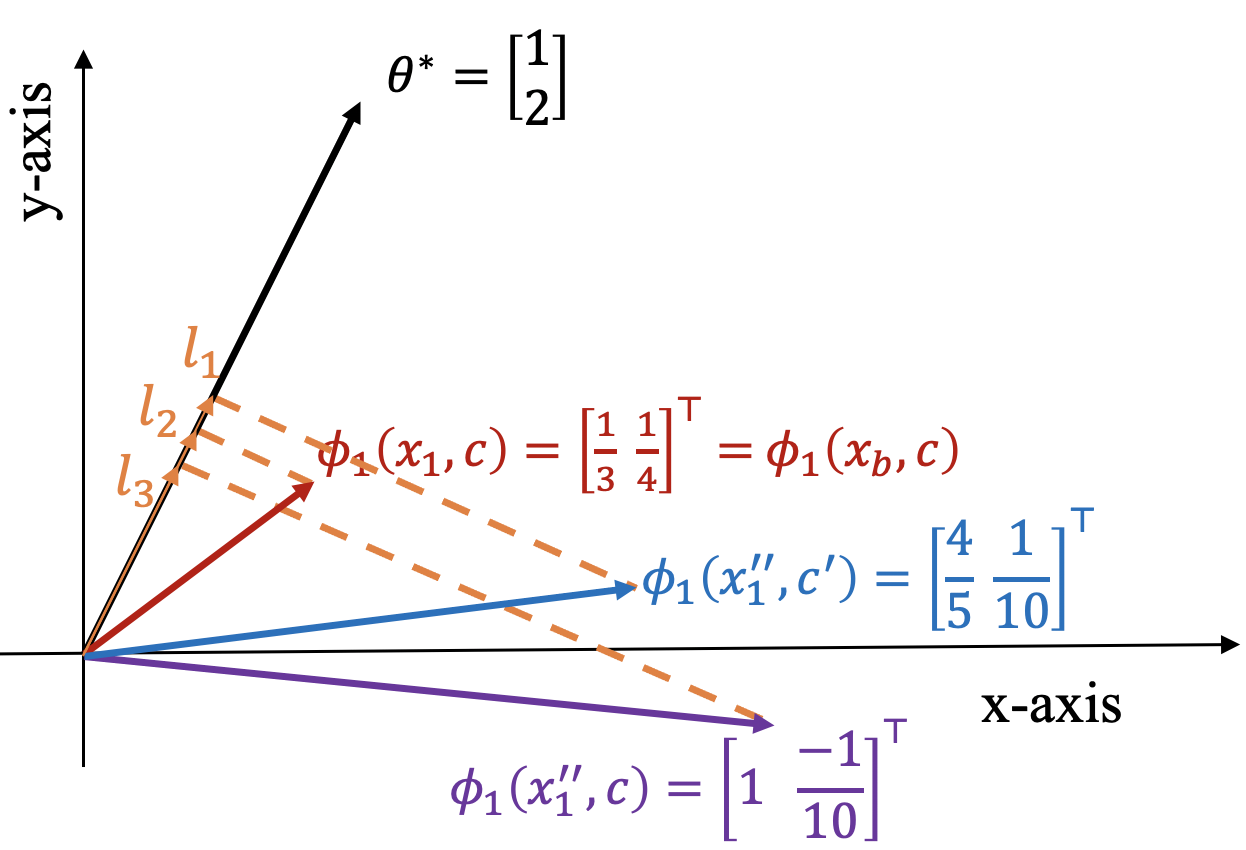}
          \end{figure}
      \end{minipage}
      \hspace{-0.08\linewidth}
      \begin{minipage}{0.45\linewidth}
          \begin{figure}[H]
              \includegraphics[width=0.8\linewidth, height=0.4\linewidth]{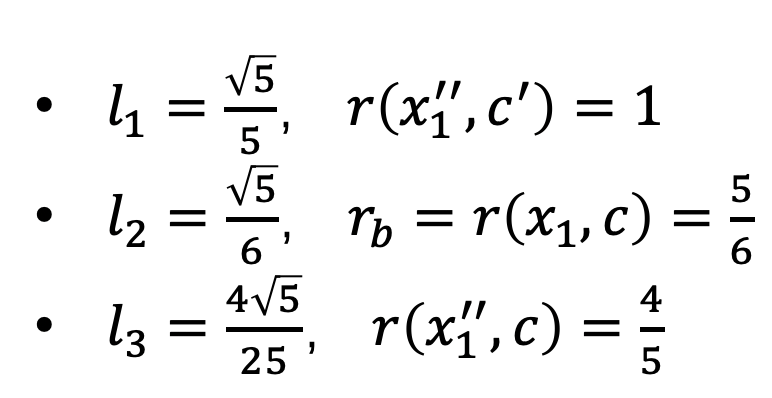}
          \end{figure}
      \end{minipage}
  \end{minipage}
  \vspace{-3 mm}
  \caption{\small An example demonstrating how an unsafe action $x_1''$ under true context appears to be safe under noisy context observation leading to incorrect conclusions about safe actions.}\label{Fig:toy}
  \vspace{-3 mm}
\end{figure}
Let the feature vector set  $\Phi_i$ for agent $i$ is
\vspace{-3 mm}
\begin{align*}
\scalefont{0.75}{
\begin{array}{c}
\arraycolsep=2pt % Set column separation to a smaller value for the indices
\begin{array}{ccc}
\hspace{1.7 em} \scriptstyle{\phi_{x_1, c}} & \scriptstyle{\phi_{x_1', c}} & \scriptstyle{\phi_{x_1'', c}} \\
\end{array} \\
\arraycolsep=5pt % Restore default column separation for the matrix
\hspace{-2 mm}\Phi_1 = 
\left[
\begin{array}{ccc}
\frac{1}{3} & 1 & 1 \\
\frac{1}{4} & 0 & -\frac{1}{10} \\
\end{array}
\right], 
\end{array}
\hspace{-0.8em}
\begin{array}{c}
\arraycolsep=2pt % Set column separation to a smaller value for the indices
\begin{array}{ccc}
\hspace{2.0em} \scriptstyle{\phi_{x_2, c}} & \scriptstyle{\phi_{x_2', c}} & \scriptstyle{\phi_{x_2'', c}} \\
\end{array} \\
\arraycolsep=5pt % Restore default column separation for the matrix
\hspace{-1 mm}\Phi_2 = 
\left[
\begin{array}{ccc}
\frac{1}{3} & 1 & 1 \\
\end{array}
\right], 
\end{array}
\hspace{-0.8em}
\begin{array}{c}
\arraycolsep=2pt % Set column separation to a smaller value for the indices
\begin{array}{ccc}
\hspace{2.0em} \scriptstyle{\phi_{x_3, c}} & \scriptstyle{\phi_{x_3', c}} & \scriptstyle{\phi_{x_3'', c}} \\
\end{array} \\
\arraycolsep=5pt % Restore default column separation for the matrix
\hspace{-1 mm}\Phi_3 = 
\left[
\begin{array}{ccc}
\frac{1}{4} & 0 & -\frac{1}{10} \\
\end{array}
\right].
\end{array}
}
\end{align*}
Using the index set $\I$ and appending zero elements, we map the reward parameter $\theta_i^{\star}$ and feature vector sets $\Phi_i$ to a common reward parameter $\ths$ and heterogeneous feature vector sets $\Phi_i^{\prime}$ as given below, where $\theta^{\star} = \theta_1^{\star} = [1 \; 2]^{\top}$. 
\begin{align*}
\scalefont{0.75}{
\begin{array}{c}
\arraycolsep=2pt % Set column separation to a smaller value for the indices
\begin{array}{ccc}
\hspace{1.0 em} \scriptstyle{\phi_1(x_1, c)} & \scriptstyle{\phi_1(x_1', c)} & \scriptstyle{\phi_1(x_1'', c)} \\
\end{array} \\
\arraycolsep=5pt % Restore default column separation for the matrix
\hspace{-3 mm}\Phi_1^{\prime} = 
\left[
\begin{array}{ccc}
\frac{1}{3} & 1 & 1 \\
\frac{1}{4} & 0 & -\frac{1}{10} \\
\end{array}
\right], 
\end{array}
\hspace{-1.1em}
\begin{array}{c}
\arraycolsep=2pt % Set column separation to a smaller value for the indices
\begin{array}{ccc}
\hspace{1.3 em} \scriptstyle{\phi_2(x_2, c)} & \scriptstyle{\phi_2(x_2', c)} & \scriptstyle{\phi_2(x_2'', c)} \\
\end{array} \\
\arraycolsep=5pt % Restore default column separation for the matrix
\hspace{-5 mm}\Phi_2^{\prime} = 
\left[
\begin{array}{ccc}
\frac{1}{3} & 1 & 1 \\
0 & 0 & 0 \\
\end{array}
\right], 
\end{array}
\hspace{-2em}
\begin{array}{c}
\arraycolsep=2pt % Set column separation to a smaller value for the indices
\begin{array}{ccc}
\hspace{1.0 em} \scriptstyle{\phi_3(x_3, c)} & \scriptstyle{\phi_3(x_3', c)} & \scriptstyle{\phi_3(x_3'', c)} \\
\end{array} \\
\arraycolsep=5pt % Restore default column separation for the matrix
\hspace{-4 mm}\Phi_3^{\prime} = 
\left[
\begin{array}{ccc}
0 & 0 & 0 \\
\frac{1}{4} & 0 & -\frac{1}{10} \\
\end{array}
\right]
\end{array}
}
\end{align*}
% $$
% \phi_1(x, c^1) = 
% \begin{bmatrix}
%     \frac{1}{3} & 1 & 1 \\
%     \frac{1}{3} & 0 & -1 \\
% \end{bmatrix}, 
% $$
% $$
% \phi_2(x, c^1) =
% \begin{bmatrix}
%     \frac{1}{3} & 1 & 1 \\
%     0 & 0 & 0 \\
% \end{bmatrix}, 
% $$
% $$
% \phi_3(x, c^1) =
% \begin{bmatrix}
%     0 & 0 & 0 \\
%     \frac{1}{3} & 0 & -1 \\
% \end{bmatrix}
% $$
%
For context $c$, action $x_1^{\prime}$ is the optimal action for agent~1. Consider a scenario, such as weather prediction or stock market prediction, where agents only observe a context distribution $\mu$ and the true context $c$ is unknown. Let $c^{\prime}:=\bE_{\mu}[\C]$ and let feature vectors for $c^{\prime}$ for agent~1 is
\begin{align*}
\scalefont{0.75}{
\begin{array}{c}
\arraycolsep=2pt % Set column separation to a smaller value for the indices
\begin{array}{ccc}
\hspace{1.8em} \scriptstyle{\phi_1(x_1, c^{\prime})} & \scriptstyle{\phi_1(x_1', c^{\prime})} & \scriptstyle{\phi_1(x_1'', c^{\prime})} \\
\end{array} \\
\arraycolsep=5pt % Restore default column separation for the matrix
\Phi_1^{''} = 
\left[
\begin{array}{ccc}
\frac{1}{4} & \frac{6}{5} & \frac{4}{5} \\
\frac{1}{3} & -\frac{1}{5} & \frac{1}{10} \\
\end{array}
\right].
\end{array}
}
\end{align*}
% $$
% \phi_1(x, c^2) = 
% \begin{bmatrix}
%     -1 & 1 & \frac{1}{2} \\
%     1 & -\frac{1}{2} & \frac{1}{3} \\
% \end{bmatrix}, 
% $$
%Based on the above observation, it is clear that action $x_1^{''}$ is optimal for $c^{\prime}$. However, for the exact context $c$, action $x_1^{''}$ is the worst action.  In this case, the optimal action derived from a context distribution may yield extremely poor rewards for the exact context. 
%This highlights a significant distinction between our context distribution problem and contextual linear bandits.
%In our scenario, the optimal action recommended by the context distribution can lead to very poor rewards for the actual context.
%
Let action $x_1$ be the baseline action, i.e., $x_1 = x_b$.  $x_1^{''}$ do not satisfy the performance constraint for context $c$. Similarly, $x_1^{\prime}$ does not meet the performance constraint for $c^{\prime}$.  Consequently, the feasible action set for $c$ is $\A_c = \{x_1, x_1^{\prime}\}$, whereas for context $c^{\prime}$ it is $\A_{c^{\prime}} = \{x_1, x_1^{''}\}$. Given that performance constraints are considered in the exact context $c$, the selection should have been made from $\A_c$ and not $\A_{c^{\prime}}$. As $x_1^{''}$ is not within the feasible action set $\A_c$, the agent will choose the baseline action $x_b$ to ensure it meets the performance constraint. Although the feasible action set $\A_c$ includes $x_1^{\prime}$, it is not present in $\A_{c^{\prime}}$, hence limiting the agent's exploration capability and preventing the selection of the optimal action $x_1^{\prime}$. This violates Lemma~C.1 proposed in \cite{moradipari2020stage}, which claims that the optimal action is always present in the feasible action set. This demonstrates the significance of a new approach when the agent can only observe the context distribution. %while ensuring that performance constraints are satisfied. 

\vspace{-3 mm}
\bibliographystyle{myIEEEtran}
\bibliography{Bandits}

\newpage
\begin{figure*}[t]
\centering
\subcaptionbox{\footnotesize $M=1, \alpha=0.5$ \label{fig:11}}{\includegraphics[scale=0.2]{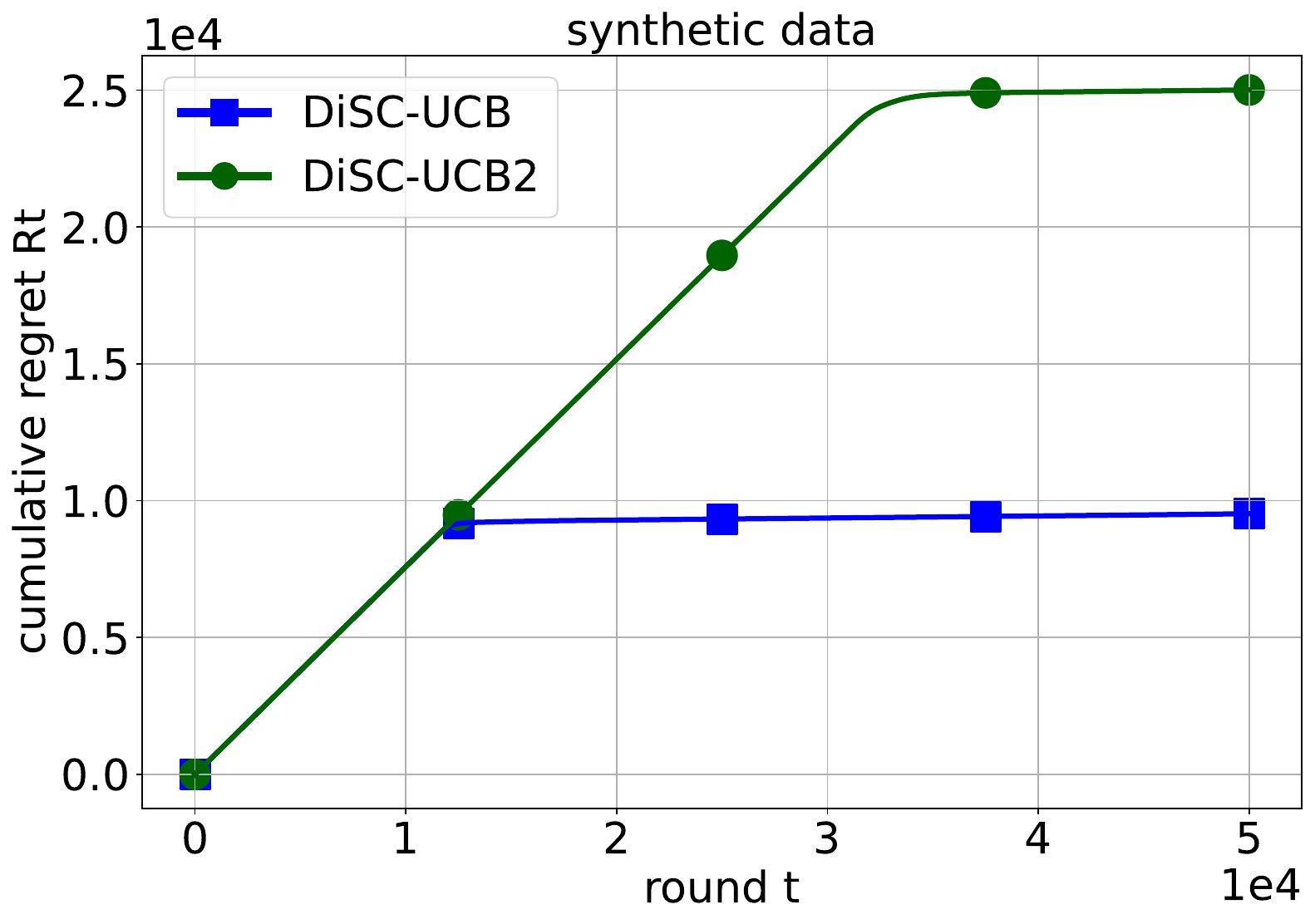}}
%\hspace{3 em}%
\subcaptionbox{\footnotesize $M=1, \alpha=0.5$ \label{fig:9}}{\includegraphics[scale=0.2]{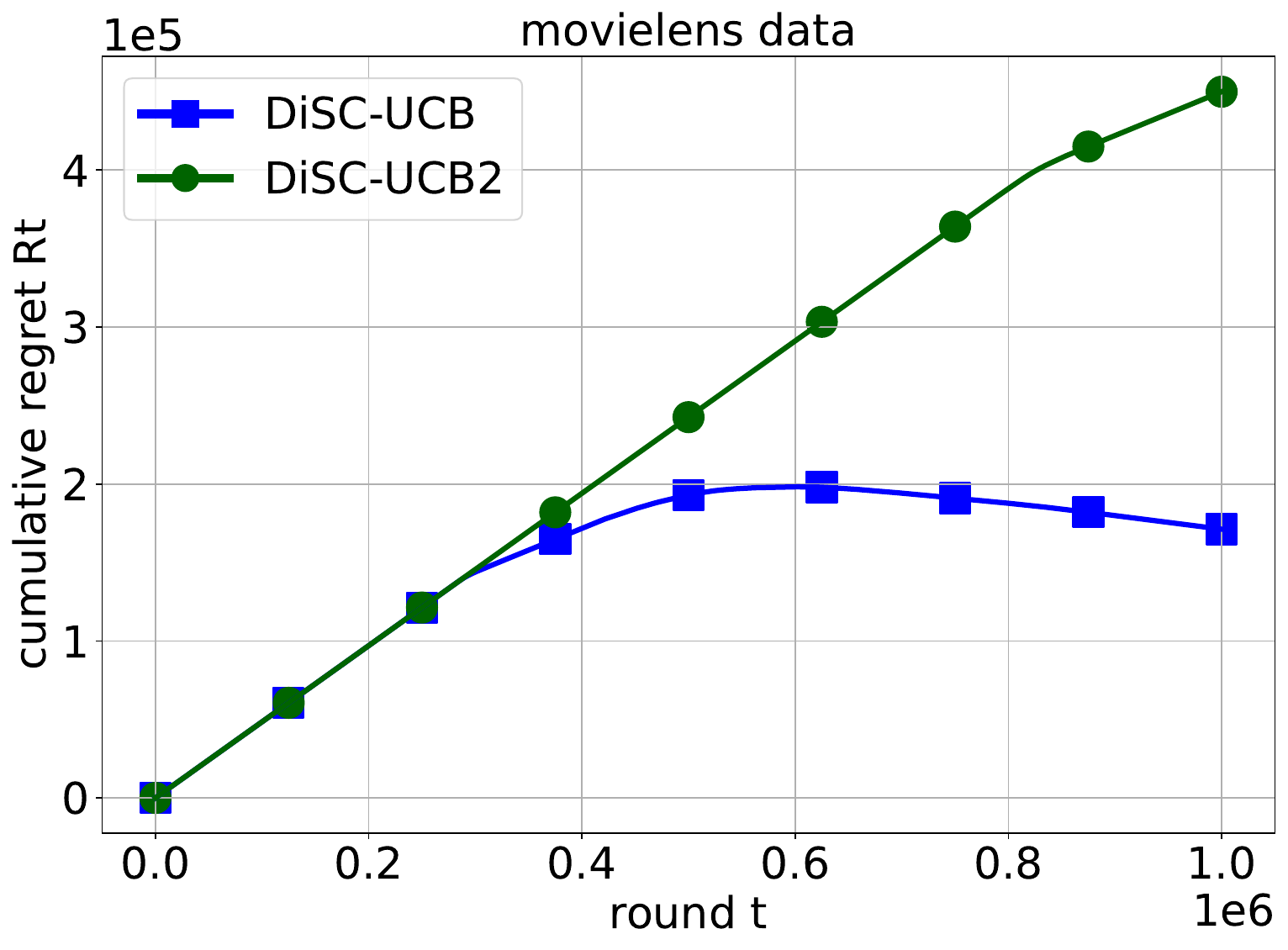}}
%\hspace{3 em}
\subcaptionbox{\footnotesize $M=3, \alpha=0.5$ \label{compare_noise}}{\includegraphics[scale=0.2]{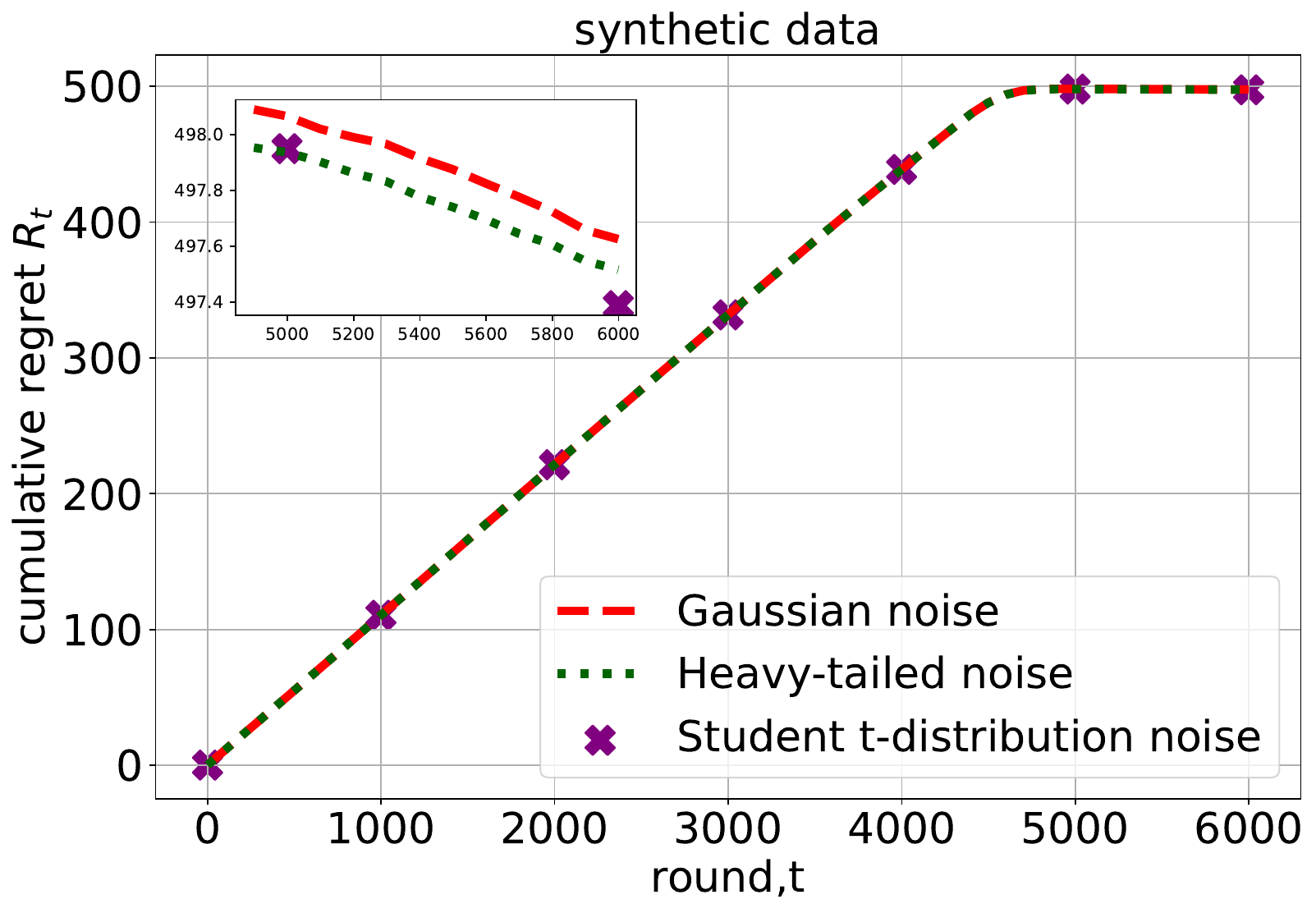}}
\vspace{-2 mm}
\caption{\small {\bf Figures~\ref{fig:11} and~\ref{fig:9}:} Comparison of cumulative regret of
DiSC-UCB with DiSC-UCB-UB for synthetic and movielens datasets. Synthetic data: In Figs.~\ref{fig:11}, we set $R=0.1$, $K=90$, $M=1$, $\ths = [1, 1]$, noise variance $=10^{-2}$, and  baseline is the 80$^{\rm th}$ best action. Movielens data: In Figs.~\ref{fig:9}, $R=0.1$, $K=50$, noise variance$=10^{-2}$, $\ths = \frac{1}{\sqrt{3}}[1, 0, 0, 0, 1, 0, 0, 0, 1]$, and the baseline action is set as the 45$^{\rm th}$ best action. 
{\bf Figures~\ref{compare_noise}:} The plot shows the cumulative regret vs. round plot for the three different noise models given in Section~\ref{sec:noise_comp}.}\label{fig:main3}
\vspace{-4 mm}
\end{figure*}
\section{Preliminaries}\label{app:pre}
\begin{prop} \label{prop:Azu}
({\bf Azuma-Hoeffdings Inequality}) Let $ M_{j} $ be a martingale on a filtration $ \F_{j} $ with almost surely bounded increments $ \left|M_{j}-M_{j-1}\right|<Q $. Then
\[
\mathbb{P}\left[M_{n}-M_{0}>b\right] \leqslant \exp \left(-\frac{b^{2}}{2 n Q^{2}}\right).
\]
\end{prop}
\begin{prop}[Lemma~11, \cite{abbasi2011improved}] \label{LT2}
Let $ \left\{X_t\right\}_{t=1}^{\infty} $ be a sequence in $ \mathbb{R}^{d}, V $ is a $ d \times d $ positive definite matrix and define $ \bar{V}_{t}=V+\sum_{s=1}^{t} X_s X_s^{\top} $. We have that
\[
\log \left(\frac{\det\left(\bar{V}_{n}\right)}{\det(V)}\right) \leqslant \sum_{t=1}^{n}\left\|X_t\right\|^2_{\bar{V}_{t-1}^{-1}}.\]
Further, if $ \left\|X_t\right\|_{2} \leqslant L $ for all $ t $, then
\[
\sum_{t=1}^{n} \min \left\{1,\left\|X_t\right\|_{V_{t-1}^{-1}}^{2}\right\}  \leqslant  2\left(\log \det\left(\bar{V}_{n}\right)-\log \det (V)\right)\]
\[\leqslant  2\left(d \log \left(\left(\trace(V)+n L^{2}\right) / d\right)-\log \det (V)\right).
\]
\end{prop}
\begin{prop}[Theorem~4.1, \cite{lin2023distributed1}]\label{ECC}
Consider the unconstrained version of the DiSC-UCB algorithm, i.e., if $\alpha=1$. Then the cumulative regret with the expected feature set $\Psi_{\ti}$ and $\beta_\ti = \beta_\ti(\sqrt{1+\sigma^2}, \delta/2)$ across all iterations $T$ is bounded with probability (w.p.) at least $1-M\delta$ by
\begin{align*}
\R_T &\leqslant 4 \beta_T \sqrt{M|N_T|d \log (MT)} + 4\beta_T \sqrt{MTd \log MT} \log(MT)\\
&+ \sqrt{2 M |N_T| \log \frac{2}{\delta}}.
\end{align*}
\end{prop}

\section{Additional Experiments}\label{app-exp}
\subsection{Comparison Between DiSC-UCB and DiSC-UCB-UB}
We evaluated the efficiency of our DiSC-UCB-UB algorithm with an unknown baseline reward using both synthetic and movielens datasets and compared it with the DiSC-UCB algorithm with a known baseline reward. 
In this experiment, we set the parameters for the synthetic data as $\lambda=0.1$, $d=2$, the number of contexts $|\C|=100$, and $M=1$. Feature vectors $\phi$ were randomly generated from a standard normal distribution such that the reward is in $[0, 1]$. We added noise from a normal distribution characterized by a mean of $0$ and a variance of $0.01$ to $\phi$ to construct $\psi$. 
Experiments are averaged over 50 independent trials. 
For the movielens data \cite{harper2015movielens}, we randomly chose 100 users and 50 movies and normalized their rating to be in the range $[0,1]$. We then performed a matrix factorization to extract the feature vectors. 

In Fig.~\ref{fig:11}, we present the results of DiSC-UCB-UB with DiSC-UCB for the synthetic data with $M=1$ and $\alpha=0.5$. In Fig.~\ref{fig:9}, we set $M=1$,  $\alpha=0.5$, and present the results for the movielens data. We provide the plots for both cumulative expected regret and the average cumulative regret (averaged over round $t$). In the case of DiSC-UCB-UB, the baseline reward is unknown, whereas in the case of DiSC-UCB, it is known. This difference leads to a tighter constraint for DiSC-UCB-UB, thereby requiring more iterations of initial exploration and consequently resulting in a larger regret. 
In Fig.~\ref{fig:9}, it can be observed that the cumulative regret of DiSC-UCB demonstrates a decreasing trend after $t=6 \times {10}^5$. This is because the regret is defined using the best action for $\psi$ since the contexts are unobserved and only a context distribution is available. However, the reward is a function of $\phi$, and the optimal action associated with $\phi$ always meets the performance constraint. Therefore, it may be present in the pruned action set. After enough learning rounds, the agent may choose optimal actions for $\phi$, which has a reward larger than that of $\xs$ (best action with respect to $\psi$) thereby causing negative per-step regret and a decrease in the cumulative regret. This is possible since the pruned action set has fewer actions, enabling the agent to learn not just the best action for $\psi$ but the best action for $\phi$ as $t$ increases. 
The code to reproduce experiments can be found at the link \nolinkurl{https://github.com/alwaysworkhard/Distributed-Multi-Task-Linear-Bandit}.

\subsection{Comparison with Additional Noise Models}\label{sec:noise_comp}
We conducted additional experiments with two new noise models. While noise affects decision-making, its impact is expected to be minimal due to its small magnitude and cancellation in regret computation.
We considered (i) a student's t-distribution for generating reward noise, as detailed in \cite{kang2023heavy}, (ii) a heavy-tailed noise model where noise takes the value $-\gamma$ with probability $1-\gamma^2$ and $\frac{1}{\gamma}$ with probability $\gamma^2$, as described in \cite{medina2016no}, and (iii)~a Gaussian noise with mean $0$ and variance $1$. We standardized each noise model with a mean of $0$ and a variance of $1$ for fair comparison. Our comparative analysis is illustrated in Figure~\ref{compare_noise}. We set the experiment with parameters as $M=3$, $\alpha=0.5$, and baseline action set by the $5$-th best action, averaging the results over $100$ independent trials. The impact of altering the reward noise model on cumulative regret was minimal and not significant. A detailed view of the top left corner of Figure~\ref{compare_noise} indicates a minor difference in cumulative regret, specifically $0.6$, implying that the selection of the reward noise model has no significant effect on our results. 
% \begin{figure}[h!]
% \centering{\includegraphics[scale=0.26]{compare_noise.pdf}} \caption{The plot shows the cumulative regret vs. round plot for different noise models.} \label{compare_noise}
% \end{figure}

\section{Unknown Baseline Reward (DiSC-UCB-UB)}\label{app_unknown}
In the following lemma, we provide the safety guarantee for the conservative feature vector $\psixc = (1 - \rho) \psixbic + \rho \zeta_{\ti}$, where $\rho \in (0, \bar{\rho})$, and $\bar{\rho} = \frac{\alpha r_l}{2}$. 

\begin{lemma} \label{L3_2}
At each round $t$, given the fraction $\alpha$, for any $\rho \in (0, \bar{\rho})$, where $\bar{\rho} = \frac{\alpha r_l}{2}$, the conservative feature vector $\psixc = (1 - \rho) \psixbic + \rho \zeta_{\ti}$ is  safe. 
\end{lemma}
%For proof, we refer to the arXiv version of the paper \cite{lin2024distributed}.
%\begin{comment}
\begin{proof}
To demonstrate the safety of the conservative feature vector $\psixc = (1 - \rho) \psixbic + \rho \zeta_{\ti}$, we need to show that $((1 - \rho) \pixbic + \rho \zeta_{\ti})^{\top} \ths \geqslant (1 - \alpha) \rbti$ always holds. This can be shown by verifying the following condition: 
$$
\rbti - \rho \rbti + \rho \zeta_\ti^{\top} \ths \geqslant (1 - \alpha) \rbti
$$
which is equivalent to
$
\rho (\rbti - \zeta_\ti^{\top} \ths) \leqslant \alpha \rbti.
$
By applying Cauchy Schwarz inequality, we deduce
\begin{equation}
\rho \leqslant \frac{\alpha \rbti}{1 + \rbti} \label{range: rho_1}
\end{equation}
Consequently, by setting a lower bound for the right-hand side of Eq.~\eqref{range: rho_1} with the assumption that $r_l \leqslant \rbti \leqslant 1$, we get
$$
\rho \leqslant \frac{\alpha r_l}{2}. 
$$
Therefore, for any $\rho \leqslant \frac{\alpha r_l}{2}$, the conservative feature vector $\psixc = (1 - \rho) \psixbic + \rho \zeta_{\ti}$ is assured to be safe. 
\end{proof}
%\end{comment}

Next, we demonstrate that the optimal action $\xs$ always exists within the pruned action set with high probability when $\lambda_{\min} (\bar{V}_{\ti}) \geqslant (\frac{2 (2 - \alpha) \beta_{\ti}}{\alpha r_l})^2$. The proof follows a similar approach as in Lemma C.1 in \cite{moradipari2020stage}. The key difference in the approach is that in \cite{moradipari2020stage}, the contexts are known. However, in our setting, the contexts are unknown.

\begin{lemma} \label{L4_2}
Let $\lambda_{\min} (\bar{V}_{\ti}) \geqslant (\frac{2 (2 - \alpha) \beta_{\ti}}{\alpha r_l})^2$. Then, with probability $1 - M \delta$, the optimal action $\xs$ lies in the pruned action set $\Z_{\ti}$ for all $M$ agent, i.e., $\xs \in \Z_{\ti}$. 
\end{lemma}
%For proof, we refer to the arXiv version of the paper \cite{lin2024distributed}.
%\begin{comment}
\begin{proof}
We start by establishing an upper bound for $\max_{v \in \B_{\ti}} \psixbic^{\top} v$. 
%The sequential procedures are described as follows. 
\begin{align}
&\scalemath{0.95}{\max_{v \in \B_{\ti}} \psixbic^{\top} v =\max_{v \in \B_{\ti}} \psixbic^{\top} (v - \hth_{\ti}) + \psixbic^{\top} \hth_{\ti}} \nonumber \\
&\leqslant \frac{\beta_{\ti}}{\sqrt{\lambda_{\min} (\bar{V}_{\ti})}} + \psixbic^{\top} \hth_{\ti} \label{4_2_1} \\
&= \frac{\beta_{\ti}}{\sqrt{\lambda_{\min} (\bar{V}_{\ti})}} + \psixbic^{\top} (\hth_{\ti} - \ths) + \psixbic^{\top} \ths \nonumber \\
&\leqslant \frac{2 \beta_{\ti}}{\sqrt{\lambda_{\min} (\bar{V}_{\ti})}} + \psixbic^{\top} \ths \label{4_2_2} \\
&= \frac{2 \beta_{\ti}}{\sqrt{\lambda_{\min} (\bar{V}_{\ti})}} + \bE [\pixbic^{\top} \ths]=\hspace{-1 mm}\frac{2 \beta_{\ti}}{\sqrt{\lambda_{\min} (\bar{V}_{\ti})}} + \rbti.\label{4_2_3} 
%&= \frac{2 \beta_{\ti}}{\sqrt{\lambda_{\min} (\bar{V}_{\ti})}} + \rbti \nonumber
\end{align}
where Eq.~\eqref{4_2_1} follows directly from Lemma~\ref{L1}. Eq.~\eqref{4_2_2} follows by applying Lemma~\ref{L1}, which requires that $\ths$ lies within the confidence set $\B_{\ti}$; this is ensured by Lemma~\ref{confidence}, which guarantees that  $\ths \in \B_{\ti}$ with probability $1 - M \delta$. Eq.~\eqref{4_2_3} is derived from $\psixbic = \bE [\pixbic]$.
To prove the optimal action $\xs$ always exists in the pruned action set with probability $1 - M \delta$ under the condition on the smallest eigenvalue of the Gram matrix, we need to show 
$$
\psixsc^{\top} \hth_{\ti} \geqslant \frac{\beta_{\ti}}{\sqrt{\lambda_{\min} (\bar{V}_{\ti})}} + (1 - \alpha) \max_{v \in \B_{\ti}} \psixbic^{\top} v.
$$
This is equivalent to demonstrating 
$
\psixsc^{\top} (\hth_{\ti} - \ths) + \psixsc^{\top} \ths \geqslant \frac{\beta_{\ti}}{\sqrt{\lambda_{\min} (\bar{V}_{\ti})}} + (1 - \alpha) \max_{v \in \B_{\ti}} \psixbic^{\top} v.
$
By using Lemma~\ref{confidence}, it can be determined that with probability $1 - M \delta$, $\ths$ exists within the confidence set $\B_{\ti}$. Subsequently, applying Lemma~\ref{L1}, we recognize that with probability $1 - M \delta$, $\psixsc^{\top} (\hth_{\ti} - \ths) \geqslant - \frac{\beta_{\ti}}{\sqrt{\lambda_{\min} (\bar{V}_{\ti})}}$. Furthermore, considering that $\psixsc^{\top} \ths - \rbti \geqslant \psixsbc^{\top} \ths - \rbti = \bE [\pixsbc^{\top} \ths - \rbti] = 0$, and our derived upper bound for $\max_{v \in \B_{\ti}} \psixbic^{\top} v \leqslant \frac{2 \beta_{\ti}}{\sqrt{\lambda_{\min} (\bar{V}_{\ti})}} + \rbti$, it suffices to demonstrate with probability $1 - M \delta$, 
$$
\alpha \rbti \geqslant \frac{2 (2 - \alpha) \beta_{\ti}}{\sqrt{\lambda_{\min} (\bar{V}_{\ti})}}. 
$$
By using the inequality $\rbti \geqslant r_l$, the sufficient condition for our result is $\lambda_{\min} (\bar{V}_{\ti}) \geqslant (\frac{2 (2 - \alpha) \beta_{\ti}}{\alpha r_l})^2$. 
Thus, the above analysis verifies that the optimal action $\xs$ is always present in the pruned action set with probability $1 - M \delta$ under the condition $\lambda_{\min} (\bar{V}_{\ti}) \geqslant (\frac{2 (2 - \alpha) \beta_{\ti}}{\alpha r_l})^2$. 
\end{proof}
%\end{comment}
\begin{lemma} \label{L2_2_}
For the unknown baseline setting, with probability $1 - M \delta$, any action chosen by the agent from the pruned action set $\Z_{\ti}$ satisfies the performance constraint if $\lambda_{\min} (\bar{V}_{\ti}) > (\frac{2 (2 - \alpha) \beta_{\ti}}{\alpha r_l})^2$. 
\end{lemma}
%For proof, we refer to the arXiv version of the paper \cite{lin2024distributed}.
%\begin{comment}
\begin{proof}
Let $\xb$ denote an action in the pruned action set $\Z_{\ti}$ that does not meet the performance constraint. Our goal is to show that when the agent's action is played, an action in $\Z_{\ti}$ that satisfies the performance constraint will be selected, while actions that do not meet this constraint will not be chosen. Based on the definition of $\xsb$, it can be observed that $\xsb$ is always contained within the pruned action set $\Z_{\ti}$ and meets the performance constraint. To this end, if we show
$$
\max_{\theta_1 \in \B_{\ti}} \psixsbc^{\top} \theta_1 > \max_{\theta_2 \in \B_{\ti}} \psixbc^{\top} \theta_2,
$$
it ensures that actions within the pruned action set $\Z_{\ti}$ that violate the baseline constraint are never selected. Since $\xb$ does not satisfy the baseline constraint, 
$
\pixbc^{\top} \ths < (1 - \alpha) \rbti
$
which gives
\begin{equation}
  \label{L2_2_3}
  \psixbc^{\top} \ths = \bE [\pixbc^{\top} \ths] < (1 - \alpha) \rbti
\end{equation}
Moreover, we have
\begin{align}
  \psixsbc^{\top} \ths - \rbti &= \bE [\pixsbc^{\top} \ths] - \rbti \nonumber\\
  &= \bE [\pixsbc^{\top} \ths - \rbti] \geqslant 0. \label{L2_2_1}
\end{align}
% and 
% \begin{equation}
%   \label{L2_2_2}
%   \lambda_{\min} (\bar{V}_{\ti}) \geqslant (\frac{2 (2 - \alpha) \beta_{\ti}}{\alpha r_l})^2
% \end{equation}
To show 
$
\max_{\theta_1 \in \B_{\ti}} \psixsbc^{\top} \theta_1 > \max_{\theta_2 \in \B_{\ti}} \psixbc^{\top} \theta_2, 
$
based on Lemma~\ref{confidence}, it is sufficient to demonstrate that with probability $1 - M \delta$, 
$
\psixsbc^{\top} \ths > \max_{\theta_2 \in \B_{\ti}} [\psixbc^{\top} \ths + \psixbc^{\top} (\theta_2 - \hth_{\ti}) + \psixbc^{\top} (\hth_{\ti} - \ths)].
$
By using Eq.~\eqref{L2_2_3}, Eq.~\eqref{L2_2_1} and Lemma~\ref{L1}, it sufficient to prove that 
$$
\rbti > (1 - \alpha) \rbti + \frac{2 \beta_{\ti}}{\sqrt{\lambda_{\min} (\bar{V}_{\ti})}}
$$
Applying $\rbti \geqslant r_l$ and $\frac{2 (2 - \alpha) \beta_{\ti}}{\sqrt{\lambda_{\min} (\bar{V}_{\ti})}} \geqslant \frac{2 \beta_{\ti}}{\sqrt{\lambda_{\min} (\bar{V}_{\ti})}}$ are always hold since $\alpha \leqslant 1$, it is sufficient to show
$$
\lambda_{\min} (\bar{V}_{\ti}) > (\frac{2 (2 - \alpha) \beta_{\ti}}{\alpha r_l})^2. 
$$
Therefore,  with probability $1 - M \delta$,  any action chosen by DiSC-UCB-UB algorithm from the pruned action set $\Z_{\ti}$ satisfies the performance constraint if $\lambda_{\min} (\bar{V}_{\ti}) > (\frac{2 (2 - \alpha) \beta_{\ti}}{\alpha r_l})^2$. 
\end{proof}
%\end{comment}
Consider any round $t$ during which the agent plays the agent's action, i.e., at round $t$, both condition $F = 1$ is met and $\lambda_{\min} (\bar{V}_{\ti}) \geqslant (\frac{2 (2 - \alpha) \beta_{\ti}}{\alpha r_l})^2$ is satisfied. 
%When $F = 1$, we are assured that there exists an action $x_{\ti}$ in the pruned action set $\X_{\ti}$ such that $\psixc^{\top} \hth_{\ti} \geqslant \frac{\beta_{\ti}}{\sqrt{\lambda_{\min} (\bar{V}_{\ti})}} + (1 - \alpha) \rbti$. 
By Lemma~\ref{L4_2}, if $\lambda_{\min} (\bar{V}_{\ti}) \geqslant (\frac{2 (2 - \alpha) \beta_{\ti}}{\alpha r_l})^2$, it is guaranteed that $\xs \in \Z_{\ti}$ with high probability. Consequently, $\lambda_{\min} (\bar{V}_{\ti}) \geqslant (\frac{2 (2 - \alpha) \beta_{\ti}}{\alpha r_l})^2$ is sufficient to guarantee that $\Z_{\ti}$ is non-empty. To this end, our analysis of the unknown baseline setting will henceforth only focus on the condition $\lambda_{\min} (\bar{V}_{\ti}) \geqslant (\frac{2 (2 - \alpha) \beta_{\ti}}{\alpha r_l})^2$. 

\begin{theorem} \label{prop1_1}
The regret of DiSC-UCB-UB (unknown baseline) can be decomposed into three terms as follows
\begin{align*}
\R_T &\leqslant {\scalefont{0.9}\underbrace{4 \beta_T \sqrt{M|N_T|d \log (MT)} + 4\beta_T \sqrt{MTd \log MT} \log(MT)}_{Term~1}}\\
&{\scalefont{0.9}+ \underbrace{\sqrt{2 M |N_T| \log \frac{2}{\delta}}}_{Term~2} + \underbrace{\sum_{i = 1}^M \sum_{t \in |N_T^c|} (2 \rho + 1 - r_l)}_{Term~3}.} 
\end{align*}
\end{theorem}
% \begin{proof}
% Proof follows similar approach as in Theorem~\ref{prop1}.
% \end{proof}
\begin{proof}
By the definition of cumulative regret, we have
\begin{align}
\R_T &= \sum_{i = 1}^M \sum_{t = 1}^T (\pixsc^{\top} \ths - \pixc^{\top} \ths) \nonumber \\
&= \sum_{i = 1}^M \sum_{t \in |N_T|} (\pixsc^{\top} \ths - \pixc^{\top} \ths)\nonumber \\
&+ \sum_{i = 1}^M \sum_{t \in |N_T^c|} (\pixsc^{\top} \ths - \pixc^{\top} \ths) \nonumber 
\end{align}
\begin{align}
&= \sum_{i = 1}^M \sum_{t \in |N_T|} (\pixsc^{\top} \ths - \pixc^{\top} \ths) \nonumber \\
&+ \sum_{i = 1}^M \sum_{t \in |N_T^c|} (\pixsc^{\top} \ths\nonumber\\
&- (1 - \rho) \pixbic^{\top} \ths - \rho \zeta_{\ti}^{\top} \ths) \nonumber \\
&\leqslant \underbrace{\sum_{i = 1}^M \sum_{t \in |N_T|} (\pixsc^{\top} \ths - \pixc^{\top} \ths)}_{Agents'~term}\nonumber \\
&+ \sum_{i = 1}^M \sum_{t \in |N_T^c|} (1 - r_l + \rho + \rho) \label{R_1_1} \\
&\leqslant 4 \beta_T \sqrt{M|N_T|d \log (MT)} + 4\beta_T \sqrt{MTd \log MT} \log(MT)\nonumber\\
&+ \sqrt{2 M |N_T| \log \frac{2}{\delta}}+ \sum_{i = 1}^M \sum_{t \in |N_T^c|} (2 \rho + 1 - r_l) \label{R_2_1},
\end{align}
where Eq.~\eqref{R_1_1} follows from $\pixsc^{\top} \ths \leqslant 1$, $r_l \leqslant \rbti \leqslant 1$, and $ - \zeta_{\ti}^{\top} \ths \leqslant |\zeta_{\ti}^{\top} \ths| \leqslant \|\zeta_{\ti}\| \|\ths\| \leqslant 1$. Eq.~\eqref{R_2_1} is derived through the analysis of the {\em Agents' term} in Eq.~\eqref{R_1_1}. Given that $|N_T| = T - |N_T^c|$, the remaining section of our proof focuses on determining the upper bound and lower bound for $|N_T^c|$. 
\end{proof}
%\end{comment}

%\subsection{Proof of Theorem~\ref{Thm1_1}}
\begin{theorem} \label{unknown_thm}
For the DiSC-UCB-UB (unknown baseline), the upper bound and lower bound of $|N_T^c|$ are given by
\begin{align*}
\scalemath{0.95}{|N_T^c|} &\scalemath{0.95}{\geqslant \frac{4}{M} \log(\frac{d}{\delta}) - \frac{(\frac{2 (2 - \alpha) \beta_{0, i}}{\alpha \rbtaui})^2 - (\lambda + M T)}{2 M \rho (1 - \rho)}}\\
&\scalemath{0.95}{- \sqrt{\frac{16}{M^2} \log^2(\frac{d}{\delta}) - \frac{4 ((\frac{2 (2 - \alpha) \beta_{0, i}}{\alpha \rbtaui})^2 - (\lambda + M T)) \log(\frac{d}{\delta})}{M^2 \rho (1 - \rho)}}} 
\end{align*}
\begin{align*}
\scalemath{0.95}{|N_T^c|} &\scalemath{0.95}{\leqslant \sqrt{(\frac{4 \sqrt{6}}{M} \log(\frac{d}{\delta}))^2- \frac{16 ((\frac{2 (2 - \alpha) \beta_{0, i}}{\alpha \rbtaui})^2 - (\lambda + M T)) \log(\frac{d}{\delta})}{M^2 \rho (1 - \rho)}}}. 
\end{align*}
\end{theorem}

\begin{proof}
The proof follows the same technique as in the proof of Theorem~\ref{L7}.
\begin{comment}
, given by
\begin{align*}
|N_T^c| &\geqslant \frac{4}{M} \log(\frac{d}{\delta}) - \frac{(\frac{2 (2 - \alpha)}{\alpha \rbtaui} (\sigma \sqrt{2 \log \frac{1}{\delta}} + \lambda^{\frac{1}{2}}))^2 - (\lambda + M T)}{2 M \rho (1 - \rho)}\\
&- \sqrt{\frac{16}{M^2} \log^2(\frac{d}{\delta}) - \frac{4 ((\frac{2 (2 - \alpha)}{\alpha \rbtaui} (\sigma \sqrt{2 \log \frac{1}{\delta}} + \lambda^{\frac{1}{2}}))^2 - (\lambda + M T)) \log(\frac{d}{\delta})}{M^2 \rho (1 - \rho)}} \\
|N_T^c| &\leqslant \sqrt{(\frac{4 \sqrt{6}}{M} \log(\frac{d}{\delta}))^2 - \frac{16 ((\frac{2 (2 - \alpha)}{\alpha \rbtaui} (\sigma \sqrt{2 \log \frac{1}{\delta}} + \lambda^{\frac{1}{2}}))^2 - (\lambda + M T)) \log(\frac{d}{\delta})}{M^2 \rho (1 - \rho)}}. 
\end{align*}
\end{comment}
\end{proof}

Now, we will combine the findings of Theorem~\ref{prop1_1} and Theorem~\ref{unknown_thm} to establish the cumulative regret for DiSC-UCB-UB algorithm. Determining the communication cost follows the same approach described in Theorem~\ref{T1}. 
\begin{comment}
\begin{theorem}[\bf Complete form of Theorem~\ref{Thm1_1}] \label{Thm1_unknown}
{\em The cumulative regret of Algorithm~\ref{alg:TV_1} with $\beta_{\ti} = \beta_{\ti}(\sqrt{1 + \sigma^2}, \delta / 2)$ is bounded at round $T$ with probability at least $1 - M \delta$ by 
\begin{align*}
\R (T) &\leqslant 4 \beta_T \sqrt{d \bar{c}^{\prime} \log (MT)} + 4\beta_T \sqrt{MTd \log MT} \log(MT) + \sqrt{2 \bar{c}^{\prime} \log (\frac{2}{\delta})} + \bar{c}^{\prime \prime} (2 \rho + 1 - r_l),
\end{align*}
where $\bar{c}', \bar{c}'' >0$ is given by
\begin{align*}
&\bar{c}^{\prime} = M T - 4 \log(\frac{d}{\delta}) + \frac{(\frac{2 (2 - \alpha)}{\alpha \rbtaui} (\sigma \sqrt{2 \log \frac{1}{\delta}} + \lambda^{\frac{1}{2}}))^2 - (\lambda + M T)}{2 \rho (1 - \rho)}\\
&+ \sqrt{16 \log^2(\frac{d}{\delta}) - \frac{4 ((\frac{2 (2 - \alpha)}{\alpha \rbtaui} (\sigma \sqrt{2 \log \frac{1}{\delta}} + \lambda^{\frac{1}{2}}))^2 - (\lambda + M T)) \log(\frac{d}{\delta})}{\rho (1 - \rho)}} \\
&\bar{c}^{\prime \prime} = \sqrt{(4 \sqrt{6} \log(\frac{d}{\delta}))^2 - \frac{16 ((\frac{2 (2 - \alpha)}{\alpha \rbtaui} (\sigma \sqrt{2 \log \frac{1}{\delta}} + \lambda^{\frac{1}{2}}))^2 - (\lambda + M T)) \log(\frac{d}{\delta})}{\rho (1 - \rho)}}.
\end{align*}
Further, for $\delta = \frac{1}{M^2 T}$, Algorithm~\ref{alg:TV_1} achieves a regret of $O(d \sqrt{MT} \log^2 T)$ with $O(M^{1.5} d^3)$ communication cost. }
\end{theorem}
\end{comment}
%\begin{proof}

\noindent{\em Proof of Theorem~\ref{Thm1_1}:}
The proof approach is  the same as that applied for Theorem~\ref{T1} in terms of both cumulative regret and communication cost. 

\end{document}